\documentclass[12pt]{article}

\usepackage{amsmath}
\usepackage{times}
\usepackage[dvipsone]{graphicx}
\usepackage{color}
\usepackage{multirow}
\usepackage[authoryear]{natbib}
\usepackage{rotating}
\usepackage{bbm}
\usepackage{latexsym}
\usepackage{epstopdf}   

\makeatletter
\@addtoreset{equation}{section}
\makeatother
\renewcommand{\theequation}{\arabic{section}.\arabic{equation}}

\newcommand{\captionfonts}{\normalsize}

\makeatletter
\long\def\@makecaption#1#2{%
  \vskip\abovecaptionskip
  \sbox\@tempboxa{{\captionfonts #1: #2}}%
  \ifdim \wd\@tempboxa >\hsize
    {\captionfonts #1: #2\par}
  \else
    \hbox to\hsize{\hfil\box\@tempboxa\hfil}%
  \fi
  \vskip\belowcaptionskip}
\makeatother

\usepackage{amsmath}
\usepackage{mathrsfs}

\usepackage{algorithm}
\usepackage{algorithmic}
\makeatletter
\renewcommand{\fnum@algorithm}{\fname@algorithm}
\makeatother

\usepackage{amsthm}

\usepackage{pmat}
\usepackage{subfig}
\usepackage{caption}

\usepackage{amssymb}

\usepackage[retainorgcmds]{IEEEtrantools}

\renewcommand{\theequationdis}{{\normalfont (\theequation)}}
\renewcommand{\theIEEEsubequationdis}{{\normalfont (\theIEEEsubequation)}}

\makeatletter
\newcommand*{\rom}[1]{\expandafter\@slowromancap\romannumeral #1@}
\makeatother

\newtheorem{thm}{Theorem}
\newtheorem{lem}{Lemma}
\newtheorem{dfn}{Definition}
\newtheorem{prp}{Proposition}
\newtheorem*{rmk}{Remark}
\newtheorem{rmk-2}{Remark}
\newtheorem{rmk-3}{Remark}
\newtheorem{rmk-4}{Remark}
\newtheorem{rmk-5}{Remark}
\newtheorem{rmk-6}{Remark}
\newtheorem{rmk-7}{Remark}
\newtheorem{rmk-8}{Remark}
\newtheorem{cl}{Corollary}

\usepackage[hyperindex,breaklinks]{hyperref}
\hypersetup{
           breaklinks=true,   
           colorlinks=true,   
           pdfusetitle=true,  
        }

\begin{document}
\hspace{13.9cm}

\ \vspace{20mm}\\

{\LARGE \flushleft Theoretical Exploration of Solutions of Feedforward ReLU Networks}

\ \\
{\bf \large Changcun Huang}\\
{cchuang@mail.ustc.edu.cn}\\
%


\thispagestyle{empty}
\markboth{}{NC instructions}
\ \vspace{-0mm}\\
%
\begin{center} {\bf Abstract} \end{center}
This paper aims to interpret the mechanism of feedforward ReLU networks by exploring their solutions for piecewise linear functions, through the deduction from basic rules. The constructed solution should be universal enough to explain some network architectures of engineering; in order for that, several ways are provided to enhance the solution universality. Some of the consequences of our theories include: Under affine-geometry background, the solutions of both three-layer networks and deep-layer networks are given, particularly for those architectures applied in practice, such as multilayer feedforward neural networks and decoders; We give clear and intuitive interpretations of each component of network architectures; The parameter-sharing mechanism for multi-outputs is investigated; We provide an explanation of overparameterization solutions in terms of affine transforms; Under our framework, an advantage of deep layers compared to shallower ones is natural to be obtained. Some intermediate results are the basic knowledge for the modeling or understanding of neural networks, such as the classification of data embedded in a higher-dimensional space, the generalization of affine transforms, the probabilistic model of matrix ranks, and the concept of distinguishable data sets.

\ \\[-2mm]

\section{Introduction}
The main desire for theories of engineering is why a neural network used in practice works so well, particularly for those called \textsl{deep learning} with excellent performances in recent years \citep*{LeCun2015}. The unravelling of this ``black box'' is vital to both the instruction of parameter settings and the further development with proper guidance. Furthermore, the successful applications of deep learning in biology \citep*{Jumper2021} and mathematics \citep*{Davies2021} indicate the prospect of neural networks in scientific areas, strengthening the importance of the explainable issues.

However, there's still a great gap between the theory and the application, such that many experimental results have not yet been well understood. The main purpose of this paper is to develop some basic principles of feedforward ReLU networks, through studies analogous to theoretical physics, paving the way for the interpretation of neural networks of engineering.

\subsection{Methodology of Theoretical Physics}
We want to introduce the methodology of theoretical physics to develop the theory of neural networks in a series of papers, with the expectation that the knowledge of this area could be formulated in a more precise and systematic way, and that there may exist a unified theoretical framework underlying widespread phenomena as the realm of physics. Thus, it's necessary to first clarify what this methodology actually means and how it could be employed in neural networks.

A metaphysical spirit of theoretical physics is the deduction like Euclidean geometry of the \textsl{Elements}, which uses simple or succinct axioms to explain more complicated facts. The simplicity of axioms are manifested in two ways: first, they should be as less as possible, which is sometimes referred to as ``Occam's razor''; and second, each of them should be simple enough to embody phenomena as widely as possible.

The selection of axioms by the above two principles is due to the reason of both elegance and practical considerations, especially for the latter. The simple or radical facts tend to give general explanations, since they occur more frequently and are more likely to be the bases of other ones. For instance, Euclidean geometry has only five axioms, and there are only three laws of classical mechanics developed by Newton,
which are all intuitive and easily understood.

By this methodology, we could examine the theories of ReLU \citep*{Nair2010,Glorot2011} networks that arose after the popularity of deep learning. The key point is the solution to the approximation or interpolation, and there are mainly four categories of original ideas: hinging hyperplanes \citep*{Arora2018,Wang2005},  polynomial intermediate methods \citep*{Yarotsky2017, Liang2017, Telgarsky2015}, wavelets \citep*{Daubechies2019,Shaham2018, Huang2020}, and piecewise linear or constant constructions \citep*{Shen2021,Huang2020}.

Despite elaborately designed, the above solutions are probably not the ones that appeared in engineering. Most of them are based on the architectures that had not been applied, since there exist regular subnetworks whose parameters or architectures are fixed as the basic components of the whole network, which are rarely seen in practice however. Thus, it is difficult for them to explain the neural networks of engineering, as theoretical physics does in understanding widespread natural phenomena.

As an initial step, we will not construct an apparently rigorous deductive system including formally defined axioms. However, the methodology of theoretical physics is reflected in four ways in our paper. First, there are three basic facts or principles underlying our deduction: theorem 1 of section 3.1 is related to the final output of the last layer; theorem 4 of section 6.1 is a principle for deep layers; lemma 4 of \citet{Huang2020} is about the affine transform. All of them are simple and easier to satisfy, with no strict restrictions on the architecture or parameter setting as the above cases.

Second, the deduction starts from obvious facts that are trivial to prove, and gradually achieves the final conclusions. Each result can be traced back to the origin clearly along this deduction route.

Third, the universality of the constructed solutions is also demonstrated by the solution generalizations of sections 7 and 8, as well as the probabilistic model of the appendix, all of which are based on the fundamental properties of neural networks. This ensures that our system grasps the basic phenomena of the solution space and could explain widespread solutions.

Fourth, throughout the paper, the theories are described under geometric backgrounds, providing a unified platform of the system. We know that conics and Riemannian geometry play an important role in Newton's classical mechanics and Einstein's general relativity, respectively, where the geometry acts as both demonstration methods and intuitively understanding ways. The affine-geometry background of this paper has the same effect, and could also enhance the integrity of the deduction as well as improve the readability of proofs.

One of the main goals of this paper is to find the solutions that might be or lead us to the ones that training method reaches. We think that if the solutions found are based on very simple rules or are widespread enough, they could be probably encountered by the training process, and thus resulting in the understanding of neural networks of engineering. So unlike some research concentrating on the topic of approximation rates (such as \citet{Ali2021} and \citet*{Lu2021}), the solution construction is our main concern and will be discussed in details.

\subsection{Interpolation Framework}
There are three reasons that the interpolation framework will be used. The first is that in practice, all the applications of neural networks are in terms of discrete-point manipulations, and hence the interpolation for discrete points is more directly related to the experimental background and is beneficial to the explanation of experimental results.

The second is that the interpolation has a close relationship with the approximation, and they can be converted into each other to some extent. Under the quadratic loss function, the global optimal solution of approximations with zero error (if any) corresponds to the solution of interpolations; and the former could be a sub-optimal version of the latter, when the training is not adequate enough or the chosen interpolant doesn't fit the task. In general, to discrete points, the interpolation framework is more representative as the optimal solution of the training, and is easier to manipulate.

The third is about the approximation to continuous functions. Due to the property of region dividing, it's trivial for most of the results of this paper to be generalized to the continuous-function approximation, and the error analysis is similar to that of \citet*{Huang2020}.

Our interpolation methods, particularly for those of deep-layer networks, are not as the usual case that each single point is the basic element to be interpolated, while a batch of points are regarded as a whole instead, with different batches interpolated independently. The technical method for that is the region dividing as in \citet*{Huang2020}, which is the reason that our results can be generalized to the continuous-function approximation with no additional effort.

\subsection{Evaluation of the Theories}
A criterion for evaluating a new theory is whether or not nontrivial or useful consequences could be obtained. We will give solutions of a three-layer network that is usually the subnetwork (last three layers) of convolutional neural networks \citep*{LeCun1989, LeCun1998, Krizhevsky2012}, multilayer feedforward neural networks \citep*{Roberts2021,LeCun2015,Lye2020,lee2018}, and the decoder of autoencoders \citep*{Hinton2006}, which are all widely used.

During the solution construction, we will explain some fundamental problems of ReLU networks, such as the parameter-sharing for multi-outputs, the interpretation of each component of network architectures, the mechanism of overparameterization solutions, the advantage of deep layers.

As part of a series of researches, the results of this paper are the foundation of our future work, and their importance will be further demonstrated by more applications or consequences.

\subsection{Arrangements and Contributions}
In general, all the explanations or thoughts are embedded in the proofs of the conclusions. Some results have remarks, where the theme, application or comment of the result that we want to emphasize is usually given.

The paper is organized as follows. Section 2 is the preliminary to the whole paper, in which we'll introduce the notation of network architectures, several basic concepts, and some notes that will be used throughout this paper.

Sections 3 presents a framework of three-layer networks, including the model description and some elementary results. Theorem 1 is one of the basic principles that the deduction will be based on. The interference among hyperplanes (definition 7) is the main difficulty to be solved by this paper.

Section 4 investigates the mechanism of implementing a piecewise linear function via three-layer networks, with both the principle and construction method given. The concept of distinguishable data sets (definition 11) is fundamental to the interpretation of the solution of three-layer networks.

Section 5 is about the multi-outputs of three-layer networks. A key point is the parameter-sharing mechanism of the hidden layer (the proof of theorem 3). And we'll give a solution of multi-category classification via three-layer networks in corollary 4, which is related to the subnetwork of the last three layers of convolutional neural networks.

Section 6 is for deep-layer networks. \textsl{Interference-avoiding principle} (theorem 4) could yield independent subnetworks and is a solution for eliminating the interference among hyperplanes via deep layers. We will construct a piecewise linear function through a decoder-like network architecture (theorem 5 and lemma 6). In the remark of proposition 4, an advantage of deep layers is discussed.

Section 7 generalizes the results of section 6 to more universal network architectures. The main tool is the affine transform realized by overparameterization networks (theorems 7 and 8). We shall provide some basic knowledge, such as the geometric interpretation of some type of overparameterization (theorem 6), the processing of data embedded in a higher-dimensional space (corollary 7), the construction of an affine transform under overparameterization (proposition 6).

Section 8 deals with the mechanism of the final output of deep-layer networks (theorem 9) when the affine-transform generalization of section 7 is used, and applies the theories of sections 7 and 8 to some network architecture of engineering (proposition 8). The interpretation of overparameterization solutions associated with affine transforms is discussed in proposition 8.

Section 9 provides the main general results of this paper. The proof of lemma 11 investigates the parameter-sharing mechanism of deep-layer networks for multi-outputs. The solutions to some typical network architectures of engineering are presented in theorem 10 and corollary 11. The explanation of autoencoders is given in section 9.2. Section 10 summaries this paper by a discussion.

The appendix introduces a probabilistic model to measure the possibility of the rank of matrices, and the probability of affine transforms in terms of matrices could be derived from this model.

\section{Preliminaries}
This section introduces some notations, definitions, and notes, all of which are the basic knowledge or assumptions for further discussions.
\subsection{Notation of Architectures}
This paper will frequently refer to different types of network architectures; in order to describe them simply, some notations are introduced.

\begin{dfn}
Let $n^{(d)}$ be a depth-$d$ neural network satisfying the conditions that each layer has $n$ units, and the adjacent layers are fully connected with no skipping-layer connections. When $d = 1$, $n^{(1)}$ is a one-layer network with $n$ units.
\end{dfn}

\begin{dfn}
The product $n_1^{(d_1)}n_2^{(d_2)}$ of two neural networks $n_1^{(d_1)}$ and $n_2^{(d_2)}$ is a new one derived from fully connecting the last layer of $n_1^{(d_1)}$ with the first layer of $n_2^{(d_2)}$.
\end{dfn}

For example, a three-layer network can be expressed as $n^{(1)}m^{(1)}k^{(1)}$, which has $n$-dimensional input, $m$ units in the hidden layer, and $k$ units in the output layer. The notation $n^{(1)}m^{(d)}k^{(1)}$ represents a deep-layer network with $d$ hidden layers, each of which is composed of $m$ units.

\subsection{Affine Data Structure}
Because the main results of this paper are based on discrete data points, we present the following definition to describe the effect of affine transforms on discrete points.

\begin{dfn}
The data structure of data set $D$ of $n$-dimensional space is the abbreviated representation of its affine-geometry properties, which are preserved by affine transforms, including collinearity, parallelism, dimensionality, and so on.
\end{dfn}

Note that if a data set $D$ is contained in a region $\Omega$ of an arrangement of hyperplanes \citep*{Stanley2012}, and an affine transform of $\Omega$ maps $D$ to $D'$, then the data structure of $D$ is equivalent to that of $D'$.

\subsection{Activation of Units}
\begin{dfn}
If the output of a unit $u$ with respect to point $\boldsymbol{x}$ is nonzero, we say that $u$ is activated by $\boldsymbol{x}$. The point $\boldsymbol{x}$ may be either directly from the previous layer or indirectly from a layer that skips some intermediate ones. When the point $\boldsymbol{x}$ is replaced by a data set $D$, it means that each element of $D$ activates $u$.
\end{dfn}

The next definition describes some basic phenomena of the activation of a layer and will be frequently mentioned throughout this paper.
\begin{dfn}
To network $n^{(1)}m^{(1)}$, if we say that data set $D$ of the $n$-dimensional input space simultaneously activates the $m$ units of the next layer, it means that any $\boldsymbol{x} \in D$ activates each of the $m$ units. And if to each of the $m$ units, there exists $\boldsymbol{x} \in D$ that activates it, but not necessarily the case of $\boldsymbol{x}$ activating all of the $m$ units, we say that $D$ partially activates the $m$ units of the next layer.
\end{dfn}

\subsection{Several Notes}
\begin{itemize}
\item[1.]  To a weight matrix $\boldsymbol{W}$ of size $n \times m$ or $m \times n$ for $m \ge n$, we assume that its rank is $n$. The reasonableness of this assumption lies in two aspects. One is that this condition can be easily satisfied by construction methods, as will be shown in this paper. The other is that under a probabilistic model of the appendix, the probability of $\text{rank}(\boldsymbol{W}) = n$ is 1 (theorem 12).

\item[2.] Suppose that hyperplane $l$ of $n$-dimensional space is derived from a ReLU, in the sense that its equation $\boldsymbol{w}^T\boldsymbol{x} + b = 0$ is associated with the input sum $\boldsymbol{w}^T\boldsymbol{x} + b$ to a ReLU. Let $l^+$ and $l^0$ be the two parts of $n$-dimensional space separated by $l$, corresponding to the nonzero outputs and zero outputs of the ReLU, respectively. The intersection of $l_1^+$ and $l_2^+$ is denoted by $l_1^+l_2^+$, and the union is $l_1^+ + l_2^+$.

\item[3.] In multilayer networks, the index of the input layer is assumed to be 0; and the indices of hidden layers start from $1$. For simplicity, all the figures of neural networks ignore the biases, which do exist however.


\item[4.] A piecewise constant function is considered as a special case of a piecewise linear function and its realization is trivial if we can manage the piecewise linear case.

\item[5.] In a notation of network architectures, if the output layer is labeled by symbol $'$, it means that its each unit is a linear one (marked by ``$\Sigma$'' in figures) with no bias, which only combines the outputs of the previous layer linearly. For example, $n^{(1)}m^{(1)}1'^{(1)}$ and $n^{(1)}m^{(1)}k'^{(1)}$.

\item[6.] To data set $D$ of $n$-dimensional space, we always assume that its cardinality $|D|$ is finite, unless it is a region of an arrangement of hyperplanes.

\item[7.] Let $\sigma(x) = \max(0, x)$ be the activation function of a ReLU. Since the output of a unit with respect to input $\boldsymbol{x}$ is $\sigma(\boldsymbol{w}^T\boldsymbol{x} + b)$, where $\boldsymbol{w}^T\boldsymbol{x} + b$ can be considered as coming from the equation $\boldsymbol{w}^T\boldsymbol{x} + b = 0$ of a hyperplane $l$, we sometimes say $\sigma(\boldsymbol{w}^T\boldsymbol{x} + b)$ is the output of hyperplane $l$; and the term \textsl{activated hyperplane} of a point means that its output with respect to this point is nonzero.
\end{itemize}

\section{Three-Layer Network}
We present a general framework of three-layer networks, including an elementary investigation and a mathematical description of the problems to be solved. Part of the following contents of section 3.1 had been mentioned by \citet*{DeVore2021}; however, our description emphasizes the theme of this paper, such as the concept of activated hyperplanes.

\subsection{Number of Piecewise Linear Components}

\begin{figure}[!t]
\captionsetup{justification=centering}
\centering
\subfloat[A three-layer network.]{\includegraphics[width=2.1in, trim = {4cm 4cm 5cm 2cm}, clip]{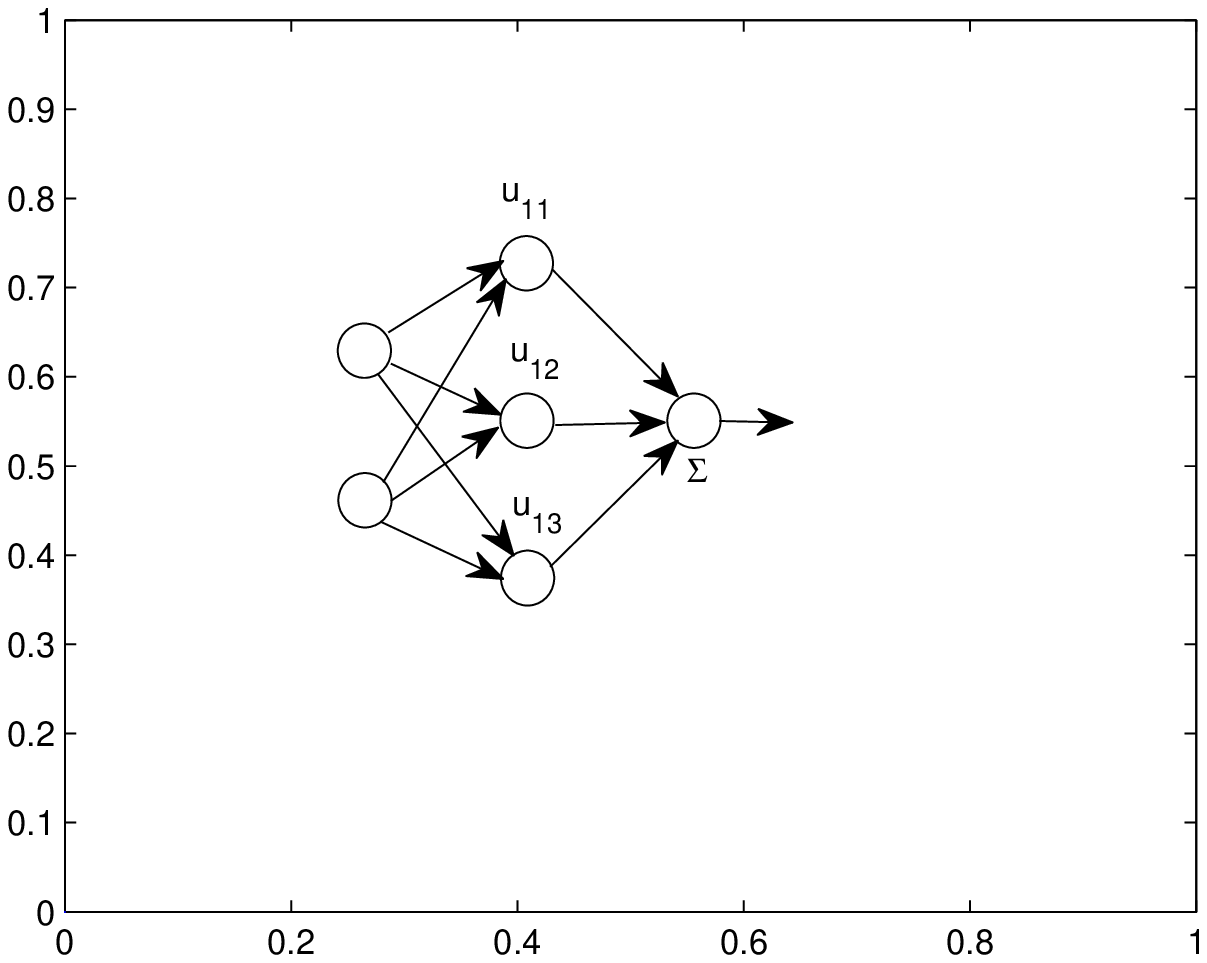}} \quad \quad \quad
\subfloat[Subdomains of a piecewise linear function.]{\includegraphics[width=2.1in, trim = {5.6cm 3.5cm 3.5cm 2cm}, clip]{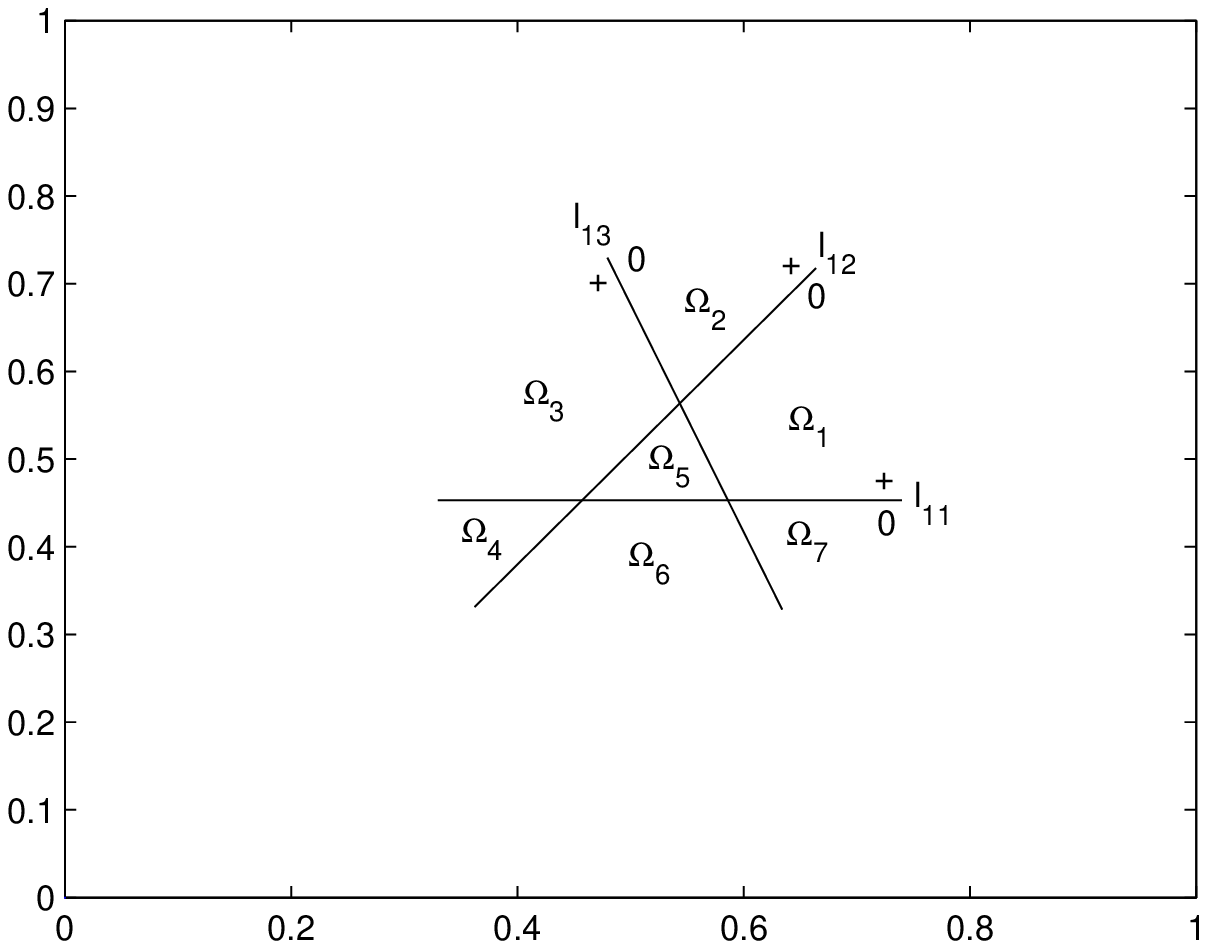}}
\caption{Number of piecewise linear components.}
\label{Fig.1}
\end{figure}

Figure \ref{Fig.1}a is a three-layer network and each line $l_{1i}$ for $i = 1, 2, 3$ of Figure \ref{Fig.1}b corresponds to the unit $u_{1i}$ of Figure \ref{Fig.1}a. Regions $\Omega_1, \Omega_2, \cdots, \Omega_7$ of Figure \ref{Fig.1}b are subdomains that comprise the domain of a piecewise linear function, on each of which a linear function is defined.

Let $\boldsymbol{w}_i^T\boldsymbol{x} + b_i = 0$ be the equation of line $l_{1i}$. Note that subdomain $\Omega_1$ is in fact the region $l_{11}^+l_{12}^0l_{13}^0$. The nonzero output of the network with respect to $\boldsymbol{x} \in \Omega_1$ is $y = \alpha_1(\boldsymbol{w}_1^T\boldsymbol{x}$ + b), where $\alpha_1$ is the output weight of $u_{11}$. This means that only line $l_{11}$ is relevant to the nonzero output with respect to $\Omega_1$, or we say that $\Omega_1$ is only influenced by $l_{11}$. Similarly, subdomain $\Omega_2 = l_{11}^+l_{12}^+l_{13}^0$ is influenced by both $l_{11}$ and $l_{12}$, and $\Omega_3$ by all the three lines. We can see that other subdomains $\Omega_4, \cdots, \Omega_7$ also have their own distinct influencing lines.

Because the parameters of the output layer are the same for all the subdomains, different linear functions on each subdomain can only be obtained by the control of the activation of the units of the hidden layer, which is related to the influencing lines of a subdomain as discussed above. To this example, since each subdomain $\Omega_j$ for $j = 1, 2, \cdots, 7$ has its own distinct influencing lines, the linear functions on them could be different from each other.

We introduce some terminologies to generalize the above example. In combinatorial geometry \citep*{Stanley2012}, a finite set $\mathcal{H}$ of hyperplanes of $n$-dimensional space $\mathbb{R}^n$ located in certain positions is called an \textsl{arrangement}; and a \textsl{region} of arrangement $\mathcal{H}$ is a connect component of $\mathbb{R}^n-\mathcal{H}$. The number of regions formed by an arrangement had been intensively studied \citep*{Dimca2017,Stanley2012,Zaslavsky1975}. Figure \ref{Fig.1}b is an example of an arrangement of three lines whose region number is seven.

\begin{dfn}
To network $n^{(1)}m^{(1)}$, let $D$ be a data set of the $n$-dimensional input space. Under an arrangement of $m$ hyperplanes formed by the first layer, the activated hyperplanes (units) of $D$ are those producing nonzero output with respect to $D$. The set of activated hyperplanes of $D$ is denoted by $l^+(D)$.
\end{dfn}

\begin{thm}
The number of linear components of a piecewise linear function $f: \mathbb{R}^n \to \mathbb{R}$ produced by network $n^{(1)}m^{(1)}1'^{(1)}$ equals the number of regions of the arrangement of $m$ hyperplanes formed by the $m$ units of the hidden layer.
\end{thm}
\begin{proof}
As the example of Figure \ref{Fig.1}, to an arrangement of $m$ hyperplanes, each region is associated with a distinct set of activated hyperplanes that produces the local linear function, which follows the conclusion.
\end{proof}

\begin{cl}
To network $2^{(1)}m^{(1)}1'^{(1)}$ with two-dimensional input, the number of linear components of output piecewise linear functions is
\begin{equation}
\mathcal{N}_2 = 1 + m + \binom{m}{2} - \sum_{i = 3}^{m}k_i\binom{i - 1}{2} - \sum_{j = 1}^{t}\binom{p_j}{2},
\end{equation}
where $k_i$ is the number of distinct intersection points of $i$ lines for $i \ge 3$, and there are $t$ classes of parallel lines with different directions, each having $p_j$ lines.
\end{cl}
\begin{proof}
Because $\mathcal{N}_2$ of equation 3.1 is the number of regions of an arrangement of $m$ lines on a plane \citep*{Dimca2017}, the conclusion holds by theorem 1. As an example, in Figure \ref{Fig.1}b, since $n = 2$, $m = 3$, $i = 2$ with $k_2 = 3$, and $t = 0$, by equation 3.1, $\mathcal{N}_2$ = 7.
\end{proof}

In $n$-dimensional space, an arrangement of $m$ hyperplanes $l_i$'s for $i = 1, 2, \cdots, m$ is called in general position, provided that \citep*{Stanley2012}
\begin{equation}
\begin{aligned}
m \le n &\Rightarrow \textnormal{dim}(l_1 \cap l_2 \cap \cdots \cap l_m) = n - m \\
m > n &\Rightarrow l_1 \cap l_2 \cap \cdots \cap l_m = \emptyset.
\end{aligned}
\end{equation}
\begin{cl}
To network $n^{(1)}M^{(1)}1'^{(1)}$ with $n$-dimensional input and $M$ hyperplanes (ReLUs) in the hidden layer, the number of linear components of its output piecewise linear function satisfies
\begin{equation}
\mathcal{N} \le \sum_{i=0}^{n}\binom{M}{i},
\end{equation}
where the equality holds when the $M$ hyperplanes are in general position. In terms of big-$O$ notation, $\mathcal{N}$ is $O(M^n)$, which is an exponential order with respect to the dimensionality $n$ of the input space.
\end{cl}
\begin{proof}
In $n$-dimensional space, the number of regions of an arrangement of $M$ hyperplanes is less than or equal to $\sum_{i=0}^{n}\binom{M}{i}$ \citep*{Dimca2017}, where the equality occurs when the $M$ hyperplanes are in general position by equation 3.2. Thus, inequality 3.3 holds by theorem 1.

To the term $\sum_{i=0}^{n}\binom{M}{i}$ of inequality 3.3, if we treat $M$ as the variable and $n$ as a constant, it is a polynomial of degree $n$. So $\mathcal{N}$ is $O(M^n)$. To $O(M^n)$, if $n$ is considered as the variable with $M$ fixed, it is an exponential order.
\end{proof}

\begin{rmk-4}
Corollary 2 tells us that as the dimensionality of the input space increases, the number of linear components of a piecewise linear function output by $n^{(1)}M^{(1)}1'^{(1)}$ will have an exponential growth in terms of the big-$O$ notation. On one hand, this may enhance the expressive capability of neural networks; on the other hand, it may become more difficult to reach a feasible solution for the training process.
\end{rmk-4}

\begin{rmk-4}
The solution to a desired piecewise linear function by a three-layer network will be constructed from the $\mathcal{N}$ different linear components of inequality 3.3.
\end{rmk-4}

\subsection{Geometric Description of Interpolation}
Given a set of data points $\{\boldsymbol{x}_i, y_i\}$'s for $i = 1, 2, \cdots, k$ coming from a piecewise linear function $f: \mathbb{R}^n \to \mathbb{R}$, namely $f(\boldsymbol{x}_i) = y_i$, we want to find a three-layer network $n^{(1)}m^{(1)}1'^{(1)}$ to interpolate them.

Denote the output of a ReLU by $\sigma(x) = max (0, x)$. Let $\boldsymbol{w}_j$ and $b_j$ for $j = 1, 2, \cdots, m$ be the weight vector and bias of the $j$th unit of the hidden layer, respectively; and the parameters of the output layer are denoted by $\alpha_j$'s. Let $l_j$ be the hyperplane corresponding to the $j$th unit of the hidden layer. Then the interpolation process can be expressed as
\begin{equation}
\boldsymbol{\Gamma}^{*}= \mathop{\arg}_{\boldsymbol{\Gamma}} \forall\boldsymbol{x}_i \sum_{j=1}^{m}\alpha_j\sigma(\boldsymbol{w}_j^T\boldsymbol{x}_i + b_j) = y_i,
\end{equation}
where $\boldsymbol{\Gamma}$ represents the set of parameters $\boldsymbol{w}_j$'s, $b_j$'s and $\alpha_j$'s of the network, and $\boldsymbol{\Gamma}^*$ is the solution.

Equation 3.4 has the disadvantage that it cannot indicate the geometric meaning relevant to theorem 1. We reformulate equation 3.4 as
\begin{equation}
\boldsymbol{\Gamma}^{*}= \mathop{\arg}_{\boldsymbol{\Gamma}} \forall\boldsymbol{x}_i\exists\Omega_j(\mathcal{D}_j(\boldsymbol{x}_i) \land \mathcal{I}_j(\boldsymbol{x}_i)),
\end{equation}
where the two predicates
\begin{equation}
\mathcal{D}_j(\boldsymbol{x}_i) := \boldsymbol{x}_i \in \Omega_j
\end{equation}
with $l^+(\boldsymbol{x}_i) = l^+(\Omega_j)$, and
\begin{equation}
\mathcal{I}_j(\boldsymbol{x}_i) := \sum_{l_k \in l^+(\Omega_j)}\alpha_k(\boldsymbol{w}_k^T\boldsymbol{x}_i + b_k) = y_i,
\end{equation}
where $l_k$ is an activated hyperplane of region $\Omega_j$ or point $\boldsymbol{x}_i$.

Equation 3.5 means that, to any $\boldsymbol{x}_i$, there exists a region $\Omega_j$ containing $\boldsymbol{x}_i$ such that $\{\boldsymbol{x}_i, y_i\}$ can be interpolated by the network. As shown in equation 3.6, predicate $\mathcal{D}_j(\boldsymbol{x}_i)$ represents that $\boldsymbol{x}_i$ is in subdomain $\Omega_j$. The interpolation process is denoted by predicate $\mathcal{I}_j(\boldsymbol{x}_i)$, where the sum is over all the activated hyperplanes of $\Omega_j$.

Under the geometric viewpoint of equation 3.5, it's easier for us to imagine how a feasible solution could be found. The following discussion will be on the basis of equation 3.5.

\subsection{Difficulty of Three-Layer Networks}
By equation 3.7, we see that each activated hyperplane $l_k$ of point $\boldsymbol{x}_i$ can influence the final output of the network. From the viewpoint of hyperplane $l_k$, half of the input space could activate it, or equivalently, $l_k$ has impact on half of the input space where the output of $l_k$ is nonzero.

To the implementation of a piecewise linear function, this means that when adjusting the parameters of $l_k$ to produce a linear function on subdomain $\Omega_j$, it could affect other linear functions that are defined on $l_k^+$. And if we try to eliminate the unwanted influence by modifying the parameters of the associative hyperplanes, the accomplished linear function on $\Omega_j$ may be disturbed again, when $\Omega_j$ happens to be in the influenced region of the modified hyperplanes. This procedure may occur recursively, resulting in the difficulty of a three-layer network in generating the piecewise linear function. We summarize the above discussion as:

\begin{dfn}
To the output of a three-layer network $n^{(1)}m^{(1)}1'^{(1)}$ with respect to a point of the $n$-dimensional input space, any new added ReLU of the hidden layer could influence half of the input space in terms of its nonzero output. We call this influence the half-space interference of hyperplanes or interference among hyperplanes.
\end{dfn}

Half-space interference is the main difficulty of constructing a piecewise linear function via three-layer networks. We will give solutions to this problem by the concept of distinguishable data sets in section 4, and by deep layers in section 6.

\section{Output Piecewise Linear Functions}
We proceed to realize a piecewise linear function via three-layer networks. First, introduce the parameter-setting method of the output layer in lemmas 1 and 2. Second, give a sufficient condition associated with the mechanism of the hidden layer in proposition 3. Finally, we prove a general result in theorem 2, demonstrating the interpolation capability of three-layer networks.

The previous works \citep*{DeVore2021,zhang2017} had presented a concise method to interpolate data with three-layer networks. Compared to that, first, due to geometric backgrounds, our result can be generalized to explain the parameter-sharing mechanism of multi-output networks (theorem 3 of section 5). Second, our theoretical framework can include their case, since theorem 1 is necessary for the interpolation via ReLU networks.

\subsection{Mechanism of Output Layer}
\begin{dfn}
To three-layer network $n^{(1)}m^{(1)}1'^{(1)}$, under an arrangement of $m$ hyperplanes derived from the hidden layer, the set of the activated hyperplanes of data set $D$ of the $n$-dimensional input space are denoted by $l^+(D) = \{l_{i_1}, l_{i_2}, \cdots, l_{i_k}\}$ with $k \le m$. The linear-output matrix of $D$ with respect to $l^+(D)$ is defined as
\begin{equation}
\boldsymbol{\mathcal{W}} = \begin{bmatrix}
\boldsymbol{w}_{i_1} & \boldsymbol{w}_{i_2} & \cdots & \boldsymbol{w}_{i_k} \\
b_{i_1} & b_{i_2} & \cdots & b_{i_k}
\end{bmatrix}
\end{equation}
whose size is $(n+1) \times k$, where each column is composed of the weight vector and bias of an activated hyperplane of $l^+(D)$.
\end{dfn}

\begin{lem}
To network $n^{(1)}m^{(1)}1'^{(1)}$ for $m > n$ with $n$-dimensional input, if a region $\Omega$ of an arrangement of $m$ hyperplanes formed by the hidden layer has at least $n + 1$ activated hyperplanes, and if the rank of the linear-output matrix $\boldsymbol{\mathcal{W}}$ of $\Omega$ is $n + 1$, then any linear function on $\Omega$ could be realized in the output layer.
\end{lem}
\begin{proof}
Let $y = \boldsymbol{w}^T\boldsymbol{x} + b$ be the linear function on $\Omega$ to be realized by $n^{(1)}m^{(1)}1'^{(1)}$. Suppose that $\Omega$ has $k$ activated hyperplanes (i.e., the cardinality $|l^+(\Omega)| = k$) with $k \ge n+1$. By equation 3.5, to $\boldsymbol{x} \in \Omega$, the output of the network is $\sum_{l_i \in l^+(\Omega)}\alpha_i(\boldsymbol{w}_i^T\boldsymbol{x} + b)$, where $\alpha_i$ is the output weight of the $i$th unit of the hidden layer. Our final goal is
\begin{equation}
\sum_{l_i \in l^+(\Omega)}\alpha_i(\boldsymbol{w}_i^T\boldsymbol{x} + b_i) = \boldsymbol{w}^T\boldsymbol{x} + b.
\end{equation}

In equation 4.2, parameters $\boldsymbol{w}_i$'s and $b_i$'s of the activated hyperplanes are fixed to be constants, and weights $\alpha_i$'s are the unknowns to be solved. The coefficient of each entry of $\boldsymbol{x}$ of the left side must be equal to the coefficient of the corresponding one of the right side, contributing to $n$ linear equations; and the biases of both the two sides are also equal, resulting in another linear equation. So we have
\begin{equation}
\boldsymbol{\mathcal{W}}\boldsymbol{\alpha} = \boldsymbol{b},
\end{equation}
where $\boldsymbol{\mathcal{W}}$ is the linear-output matrix of region $\Omega$ with respect to $l^+(\Omega)$ whose size is $(n+1) \times k$, $\boldsymbol{\alpha}$ is a $k \times 1$ vector whose entries are $\alpha_i$'s,  and $\boldsymbol{b} = [\boldsymbol{w}^ T, b]^T$.

Note that in equation 4.3, the number $k$ of the unknowns satisfies $k \ge n + 1$. Thus if $\text{rank}(\boldsymbol{\mathcal{W}}) = n + 1$, we can always find a solution of $\boldsymbol{\alpha}$ to realize equation 4.2.
\end{proof}

\begin{rmk-2}
The purpose of using a linear unit in the output layer is that it can produce negative values. In practice, if the output of a unit is constrained to be positive, ReLU can still be used. To the latter case, the left side of equation 4.2 becomes $\sum_{l_i \in l^+(\Omega)}\alpha_i(\boldsymbol{w}_i^T\boldsymbol{x} + b_i) + \beta$, where $\beta$ is the bias input of the ReLU, which also has solutions by the similar method to this lemma, and so is the case of lemma 2.
\end{rmk-2}

\begin{rmk-2}
In this lemma, if region $\Omega$ is replaced by a data set $D \subset \Omega$, the conclusion still holds, and similarly for lemma 2 below.
\end{rmk-2}

\begin{lem}
Use the notations of lemma 1 and let $n+1 < k \le m$. We select $k'$ hyeperplanes from $l^+(\Omega)$ with $n+1 \le k' < k$ to form a set $H_1 \subset l^+(\Omega)$, and let $H_2 = l^+(\Omega) - H_1$. Suppose that the output weights of hyperplanes of $H_2$ are fixed. Then we can use $H_1$ to produce a desired linear function on $\Omega$ regardless of the influence of $H_2$, provided that the rank of the linear-output matrix $\boldsymbol{\mathcal{W}}'$ of $\Omega$ with respect to $H_1$ is $n+1$.
\end{lem}
\begin{proof}
We check that under the existence of $H_2$, if only $H_1$ is used, what the change of equation 4.3 is, and whether this change could influence the generation of the linear function $y = \boldsymbol{w}^T\boldsymbol{x} + b$. Equation 4.3 then becomes
\begin{equation}
\boldsymbol{\mathcal{W}}'\boldsymbol{\alpha}' = \boldsymbol{b}',
\end{equation}
where $\boldsymbol{\mathcal{W}}'$ is an $(n+1) \times k'$ matrix whose each column is the set of parameters of a hyperplane of $H_1$, and the entries of vector $\boldsymbol{\alpha}'$ with size $k' \times 1$ are the output weights of the hyperplanes of $H_1$. The information of $H_2$ is contained in the right side of equation 4.4, in the form of constant terms, viz.,
\begin{equation}
\boldsymbol{b}' = \begin{bmatrix}
\boldsymbol{w}^T - \sum_{l_{\nu} \in H_2}\alpha_{\nu}\boldsymbol{w}_{\nu}^T & b - \sum_{l_{\nu} \in H_2} \alpha_{\nu}b_{\nu}
\end{bmatrix}^T,
\end{equation}
where $\boldsymbol{w}_{\nu}$ and $b_{\nu}$ are parameters of a hyperplane of $H_2$, and $\alpha_{\nu}$ is the corresponding output weight, all of which are considered as constant values. Therefore, despite the additional $H_2$, the right side of equation 4.4 is still a constant-entry vector, having no influences on the existence of solutions of $\boldsymbol{\alpha}'$.

Because $k' \ge n+1$, by lemma 1, if the rank of $\boldsymbol{\mathcal{W}}'$ of equation 4.4 is $n+1$, we can find a solution of $\boldsymbol{\alpha}'$ to realize the desired linear function.
\end{proof}

\begin{dfn}
A discrete piecewise linear function is defined as
\begin{equation}
f: D \to \mathbb{R}
\end{equation}
where domain $D = \bigcup_{i=1}^{k}D_i \subset \mathbb{R}^n$ with $D_i \cap D_j = \emptyset$ for $i \ne j$ and $j=1, 2, \cdots, k$. Each subdomain $D_i$ is a data set composed of discrete points with a certain linear function defined on it, and the cardinality $|D_i|$ is finite.
\end{dfn}

\begin{dfn}
To data set $D$ of $n$-dimensional space, if we say that two hyperplanes $l_1$ and $l_2$, which are derived from two units of a layer of neural networks, have the same classification effect (or result) on $D$, or that $l_1$ classifies $D$ as $l_2$, it means that $D_1 \subset l_1^+l_2^+$ and $D_2 \subset l_1^0l_2^0$, where $D_1 \cup D_2 = D$ and $D_1 \cap D_2 = \emptyset$.
\end{dfn}

\begin{figure}[!t]
\captionsetup{justification=centering}
\centering
\includegraphics[width=2.3in, trim = {5.0cm 3.3cm 4cm 2.5cm}, clip]{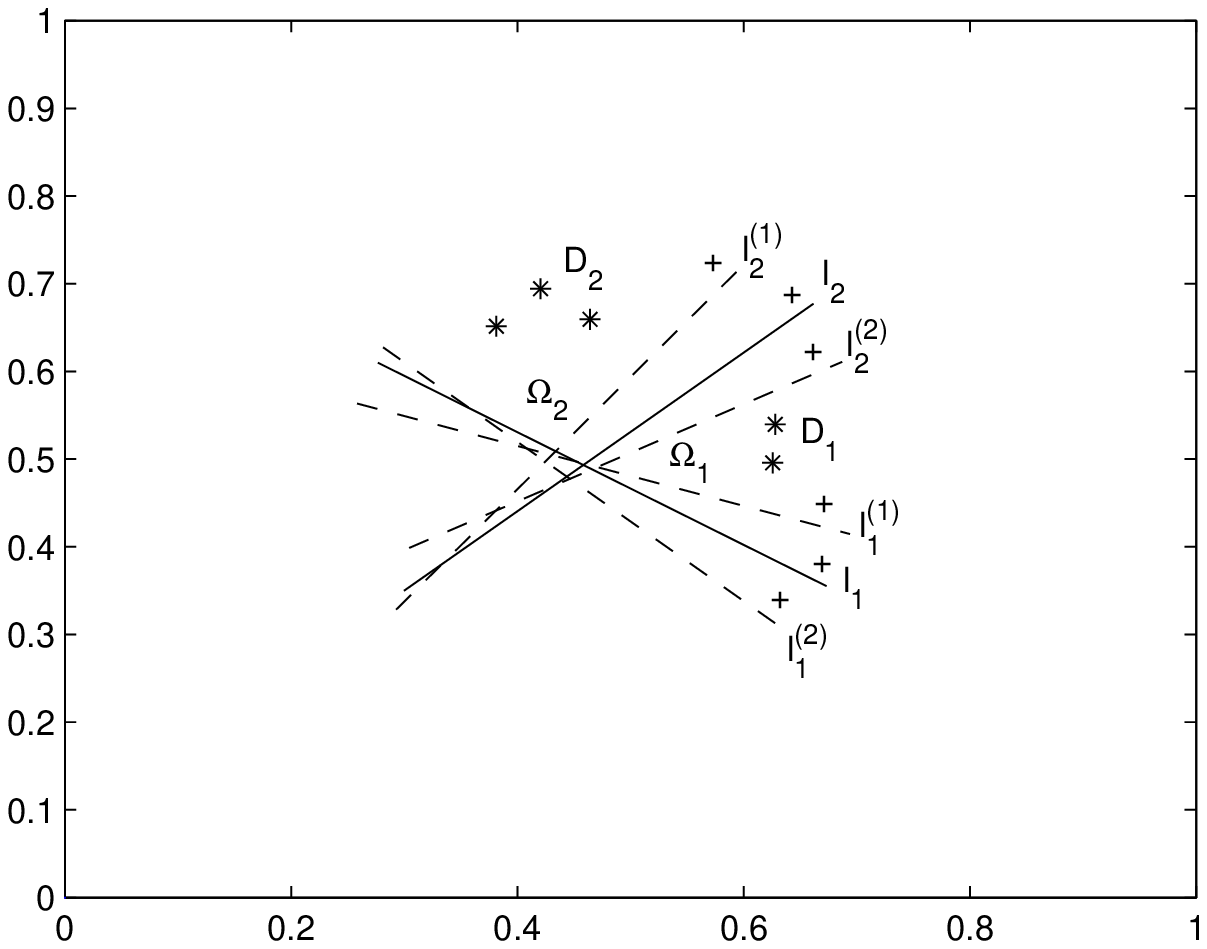}
\caption{Realization of a piecewise linear function.}
\label{Fig.2}
\end{figure}

The following proposition is an application of lemmas 1 and 2, as well as the basis of more complicated solutions given later.
\begin{prp}
To a discrete piecewise linear function of equation 4.6, suppose that its domain $D = D_1 \cup D_2$, and that $D_1 \subset l_1^+l_2^0$ and $D_2 \subset l_1^+l_2^+$, where $l_1$ and $l_2$ are two $n-1$-dimensional hyperplanes. Then a three-layer network $n^{(1)}m^{(1)}1'^{(1)}$ can realize it, provided that the number of the units of the hidden layer satisfies $m \ge2(n + 1)$.
\end{prp}
\begin{proof}
\textbf{Case $m = 2(n+1)$}: The proof is constructive and begins with an example. In Figure \ref{Fig.2}, $D_1$ and $D_2$ are two linearly separable data sets with $D_1 \subset l_1^+l_2^0$ and $D_2 \subset l_1^+l_2^+$; for simplicity, only the nonzero-output mark ``+'' is labeled on each line.

To the linear function on $D_1$, according to line $l_1$, construct other two lines $l_1^{(1)}$ and $l_1^{(2)}$ to form a region $\Omega_1 = l_1^+l_1^{(1)+}l_1^{(2)+}$ such that $D_1 \subset \Omega_1$. So $l^+(D_1) = \{l_1, l_1^{(1)}, l_1^{(2)}\}$ and $|l^+(D_1)| = 3 = n+1 $ when $n = 2$. By lemma 1, if we can adjust the parameters of $l^+(D_1)$ such that the rank of the linear output matrix of $D_1$ with respect to $l^+(D_1)$ is $n+1 = 3$, then any linear function on $D_1$ could be realized by network $2^{(1)}3^{(1)}1'^{(1)}$.

We address the parameter setting of $l^+(D_1)$ in a general form. In $n$-dimensional space, let $\boldsymbol{w}_1^T\boldsymbol{x} + b_1 = 0$ be the equation of $l_1$, where
\begin{equation}
\boldsymbol{w}_1 = \begin{bmatrix}
w_{11}, w_{12}, \cdots, w_{1n}
\end{bmatrix}^T
\end{equation}
with $w_{11} \ne 0$ (this condition can be easily satisfied by construction). To data set $D \subset l_1^+$, we should construct $n$ hyperplanes having the same classification effect as $l_1$, as well as making the rank of the linear output matrix $\boldsymbol{\mathcal{W}}_{n+1}$ of $D$ with respect to $l^+(D)$ to be $n+1$.

An $(n+1) \times (n+1)$ linear output matrix $\boldsymbol{\mathcal{W}}_{n+1}$ is constructed as
\begin{equation}
\boldsymbol{\mathcal{W}}_{n+1} = \begin{bmatrix}
w_{11} & w_{11} & w_{11} & \cdots & w_{11} \\
w_{12} & w_{12} + \varepsilon_1  & w_{12} + \varepsilon_1^2 & \cdots & w_{12} + \varepsilon_1^n \\
\vdots & \vdots & \vdots &\ddots &  \vdots\\
w_{1n} & w_{1n} + \varepsilon_{n-1}  & w_{1n} + \varepsilon_{n-1}^2 & \cdots & w_{1n} + \varepsilon_{n-1}^n \\
b_{1} & b_{1} + \varepsilon_n  & b_{1} + \varepsilon_n^2 & \cdots & b_{1} + \varepsilon_{n}^n
\end{bmatrix},
\end{equation}
where $0 < \varepsilon_i < 1$ and $\varepsilon_i \ne \varepsilon_j$ if $i \ne j$ for $i,j = 1, 2, \cdots, n$. In equation 4.8, the $n$ columns of $\boldsymbol{\mathcal{W}}_{n+1}$ except for the first one represent the $n$ constructed hyperplanes according to $l_1$, denoted by $l_{\nu}$'s for $\nu = 2, 3, \cdots, n+1$.

By theorem 4 of \citet{Huang2020}, $\det \boldsymbol{\mathcal{W}}_{n+1} \ne 0$ if $w_{11} \ne 0$. Thus, the column vectors of $\boldsymbol{\mathcal{W}}_{n+1}$ are linearly independent and $\text{rank}(\boldsymbol{\mathcal{W}}_{n+1}) = n + 1$.  Because $\varepsilon_i$'s can be arbitrarily small, we can always find $n$ hyperplanes via matrix $\boldsymbol{\mathcal{W}}_{n+1}$ such that $D \subset \prod_{\nu=2}^{n+1}l_{\nu}^+$ \citep{Huang2020}. 

Now return to the example of Figure \ref{Fig.2}. By equation 4.8, lines $l_1^{(1)}$ and $l_1^{(2)}$ that satisfy the condition of lemma 1 could be constructed. Thus, we can realize any linear function on $D_1$.

To the case of $D_2$, as shown in Figure \ref{Fig.2}, besides line $l_2$, two more lines $l_2^{(1)}$ and $l_2^{(2)}$ are constructed by the method of equation 4.8, which are activated by $D_2$ but not by $D_1$. After that, the network $2^{(1)}3^{(1)}1'^{(1)}$ becomes $2^{(1)}6^{(1)}1'^{(1)}$. Note that $ D_2 \subset l_1^+l_2^+\prod_{i,j}l_i^{(j)+}$ for $i, j =1, 2$ has six activated lines, that is, $|l^+(D_2)| = 6$. Among the six lines, the parameters of $l^+(D_1)$ of the original network $2^{(1)}3^{(1)}1'^{(1)}$ as well as their output wights should be preserved for the linear function on $D_1$. By lemma 2, we can only use the subset $H = \{l_2, l_2^{(1)}, l_2^{(2)}\}$ of $l^+(D_2)$ to output the desired linear function on $D_2$, without considering the influence of $l^+(D_1)$.

When dealing with $D_2$, since $D_1 \subset l_2^0l_2^{(1)0}l_2^{(2)0}$, to any $\boldsymbol{x} \in D_1$, the outputs of the hyperplanes of $H$ are all zero. So the constructed linear function on $D_2$ has no influence on $D_1$.

Finally, we realize the desired piecewise linear function on $D = D_1 \cup D_2$ by network $2^{(1)}6^{(1)}1'^{(1)}$, whose hidden layer has $m = 2(n+1) = 6$ units. The general case of $n$-dimensional space is similar.

\textbf{Case $m > 2(n+1)$}: Note that in Figure \ref{Fig.2}, for example, to $l_1$, when there exist more than two lines classifying the data points as $l_1$ besides $l_1^{(1)}$ and $l_1^{(2)}$, if we ensure that two of them are constructed by equation 4.8, then the rank of the linear-output weight matrix is still $n + 1 = 3$, and the redundant lines would not influence the production of the linear function on $D_1$ according to lemma 1 or 2.

In general, to data set $D$, when $|l^+(D)| > n+1$, if there exists an $(n+1) \times (n+1)$ nonsingular submatrix in the linear-output weight matrix $\boldsymbol{\mathcal{W}}$ of $D$ with respect to $l^+(D)$, then $\text{rank}(\boldsymbol{\mathcal{W}}) = n+1$; and by lemma 1 or 2, the redundant activated hyperplanes cannot influence the implementation of the linear function needed.

Let $l_1$ be an activated hyperplane of $D$, whose equation is $\boldsymbol{w}_1^T\boldsymbol{x} + b_1 = 0$, where $\boldsymbol{w}_1$ is defined as equation 4.7. Analogous to equation 4.8, we construct an $(n+1) \times m$ for $m > n+1$ linear-output weight matrix of $D$ as
\begin{equation}
\boldsymbol{\boldsymbol{\mathcal{W}}} = \begin{bmatrix}
\boldsymbol{\mathcal{W}}_{n+1}, \boldsymbol{\mathcal{W}}_c
\end{bmatrix},
\end{equation}
where $\boldsymbol{\mathcal{W}}_{n+1}$ is of equation 4.8 and
\begin{equation}
\boldsymbol{\mathcal{W}}_c = \begin{bmatrix}
w_{11} & w_{11} & \cdots & w_{11} \\
w_{12} + \varepsilon_1^{n+1} & w_{12} + \varepsilon_1^{n+2} & \cdots & w_{12} + \varepsilon_1^{m-1} \\
\vdots & \vdots & \ddots & \vdots \\
w_{1n} + \varepsilon_{n-1}^{n+1} & w_{1n} + \varepsilon_{n-1}^{n+2} & \cdots & w_{1n} + \varepsilon_{n-1}^{m-1} \\
b_{1} + \varepsilon_n^{n+1} & b_{1} + \varepsilon_n^{n+2} & \cdots & b_{1} + \varepsilon_n^{m-1}
\end{bmatrix},
\end{equation}
where $0 < \varepsilon_i < 1$ and $\varepsilon_i \ne \varepsilon_j$ for $i, j = 1, 2, \cdots, n$. Then $\text{rank}(\boldsymbol{\mathcal{W}}) = n+1$ and when $\varepsilon_i$'s are small enough, the $m-1$ constructed hyperplanes have the same classification effect as $l_1$. This is the case of $m > 2(n+1)$.
\end{proof}

\subsection{Distinguishable Data Sets}
We generalize proposition 1 with only two subdomains to the case of any finite number of subdomains.

\begin{dfn}
Let $D_i$'s for $i =1, 2, \cdots,k$ be $k$ data sets of $n$-dimensional space. We say that $D_i$'s are distinguishable, provided that we can find $k$ hyperplanes $l_i$'s satisfying the following conditions: $D_{j_1} \subseteq l_1^+$, $D_{j_{\nu}} \subseteq l_{\nu}^+ \land \bigcup_{\mu = 1}^{\nu-1}{D}_{j_{\mu}} \subseteq l_{\nu}^0$ for $\nu \ge 2$, where $j_{\nu} \in \mathbb{N}$ with $1 \le j_{\nu} \le k$ for $\nu =1, 2, \cdots,k$. The hyperplane $l_{\nu}$ with respect to $D_{j_{\nu}}$ is called a distinguishable hyperplane of $D_{j_{\nu}}$, and the order of sequence $D_{j_{1}}, D_{j_{2}}, \cdots, D_{j_k}$ is called the distinguishable order of $D_i$'s.
\end{dfn}

\begin{figure}[!t]
\captionsetup{justification=centering}
\centering
\includegraphics[width=2.4in, trim = {4.3cm 2.5cm 4cm 3.0cm}, clip]{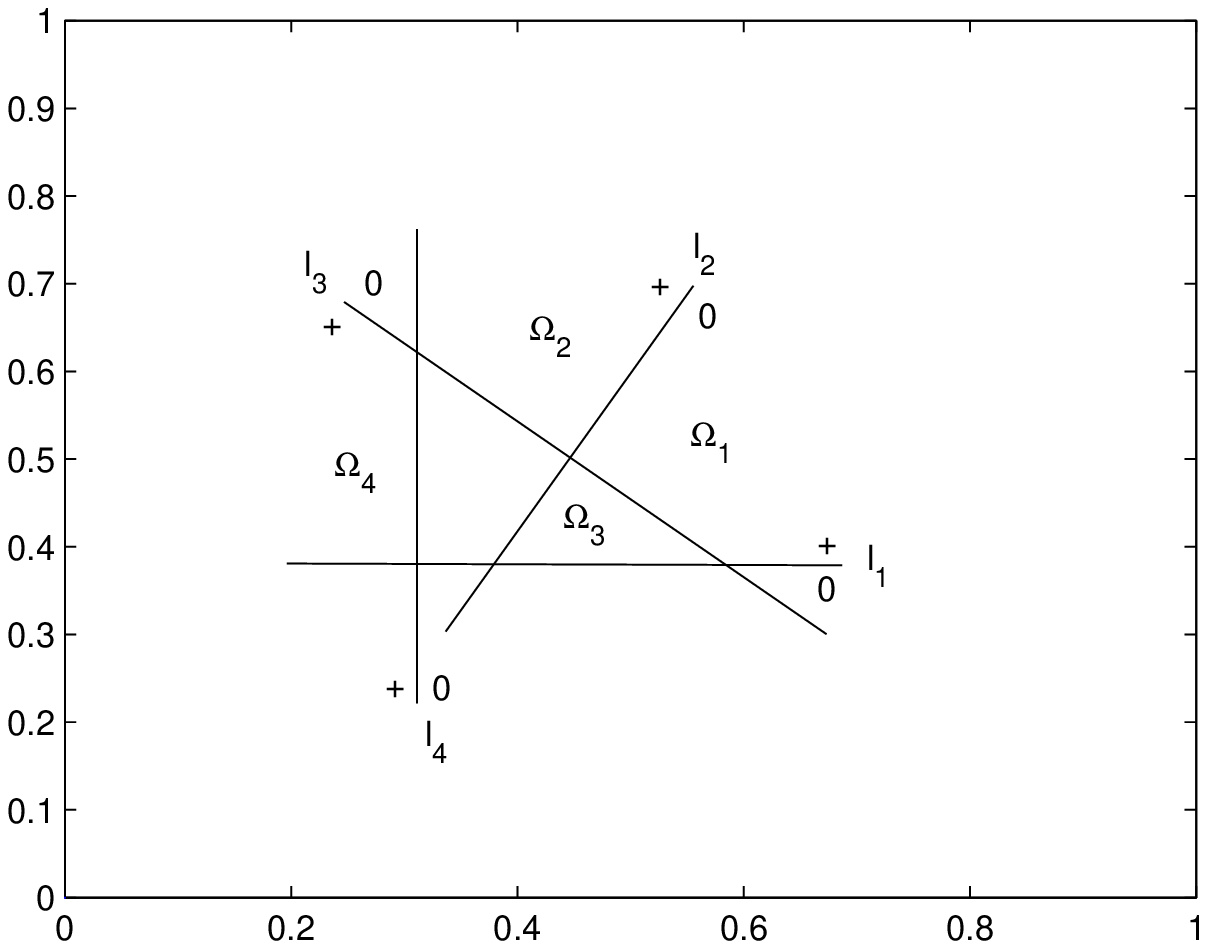}
\caption{Distinguishable regions formed by lines.}
\label{Fig.3}
\end{figure}

The concept of distinguishable data sets will be used for the solution to piecewise linear functions via three-layer networks . What follows is an example of this concept.
\begin{prp}
In $n$-dimensional space, $m$ hyperplanes can be adjusted to form $m$ distinguishable regions $\Omega_i$'s for $i =1, 2, \cdots,m$. If there exist $m$ data sets $D_i$'s with $D_i \subset \Omega_i$, they are also distinguishable.
\end{prp}
\begin{proof}
As shown in Figure \ref{Fig.3}, if only take lines $l_1$, $l_2$ and $l_3$ into consideration, the arrangement can produce regions $\Omega_1 = l_3^0l_2^0l_1^+$, $\Omega_2 = l_3^0l_2^+l_1^+$ and $\Omega_3 = l_3^+l_2^0l_1^+$, satisfying the distinguishable condition of definition 11: $\Omega_1 \subset l_1^+$, $(\Omega_2 \subset l_2^+) \land (\Omega_1 \subset l_2^0)$ and $(\Omega_3 \subset l_3^+) \land (\Omega_2 \cup \Omega_1 \subset l_3^0)$. Thus they are distinguishable.

When an arbitrary fourth line $l_4$ is added, we can translate it as far as possible to make sure that $\Omega_i \cap l_4^0 \ne \emptyset$ for $i=1, 2, 3$. Any region in $l_4^+$ can be chosen as the fourth one, such as $\Omega_4 = l_4^+l_3^+l_2^+l_1^+$; other three regions are changed into \begin{equation}
\Omega_i' = l_4^0 \cap \Omega_i
\end{equation}
for all $i$. It's easy to verify that regions $\Omega_4$ and $\Omega_i'$'s are distinguishable. For simplicity of notations, we again use the notations $\Omega_i$ to represents $\Omega_i'$ for $i = 1, 2, 3$, but the corresponding region may be changed due to equation 4.11. So $\Omega_i$'s for $i = 1, \cdots, 4$ are distinguishable.

The above procedure can be done inductively from $k = 2$. Suppose that $k-1$ distinguishable regions have been constructed by $k-1$ lines, denoted by $\Omega_i$'s for $i=1, 2, \cdots, k-1$ for $k \ge 2$. When adding the $k$th line $l_k$, we translate it to a place where $l_k^0 \cap \Omega_i \ne \emptyset$ for all $i$. Any region in $l_k^+$ could be selected as the $k$th region $\Omega_k$. Other $k-1$ regions are the intersections of $l_k^0$ and the original $\Omega_i$'s, that is, $\Omega_i' = l_k^0 \cap \Omega_i$, where $\Omega_i'$ is the new $i$th region after $l_k$ being added. Denote all of the $k$ new regions by $\Omega_i$'s for $i=1, 2, \cdots, k$ and they are distinguishable regions. Repeat it until $k = m$.

To $m$ data sets $D_i$'s for $i = 1, 2, \cdots, m$ with $D_i \subset \Omega_i$, the distinguishable condition still holds by the construction process of $\Omega_i$'s. So $D_i$'s are distinguishable data sets. The case of arbitrary $n$-dimensional space is similar. This completes the proof.

\end{proof}

\begin{prp}
To any discrete piecewise linear function of equation 4.6, if its subdomains $D_i$'s for $i = 1, 2, \cdots, k$ are distinguishable, a three-layer network $n^{(1)}m^{(1)}1'^{(1)}$ with $m \ge k(n+1)$ can realize it.
\end{prp}
\begin{proof}
As definition 11, if $D_i$'s are distinguishable, they can be arranged in a new order of $D_{j_\nu}$'s where $j_{\nu} \in \mathbb{N}$ and $1 \le j_{\nu} \le k$ for $\nu =1, 2, \cdots,k$, each of which corresponds to a distinguishable hyperplane $l_{\nu}$ such that $D_{j_{\nu}} \subset l_{\nu}^+ \land \bigcup_{\mu = 1}^{\nu-1}{D}_{j_{\mu}} \subset l_{\nu}^0$.

\textbf{Case $m = k(n+1)$}: The construction of linear functions on subdomains must be in accordance with the distinguishable order $D_{j_{\nu}}$'s for $\nu = 1, 2, \cdots, k$. The case of subdomain $D_{j_1}$ is trivial. According to its distinguishable hyperplane $l_{j_1}$, we construct $n$ hyperplanes having the same classification effect on domain $D = \bigcup_{i}D_i$ as $l_{j_1}$, by the method of equation 4.8. Then use lemma 1 to realize the linear function on $D_{j_1}$, after which all the relevant parameters of the network for $D_{j_1}$ are fixed. To subdomain $D_{j_2}$, also construct $n$ hyperplanes according to $l_{j_2}$. If $D_{j_2} \subset l_{j_2}^+l_{j_1}^0$, use lemma 1 to implement its linear function; if $D_{j_2} \subset l_{j_2}^+l_{j_1}^+$, lemma 2 is the choice. Since $D_{j_1} \subset l_{j_2}^0$ by the property of distinguishable data sets, the linear function on $D_{j_2}$ could not influence the one on $D_{j_1}$.

The above process can be done inductively. Suppose that the linear functions on $D_{j_{\nu}}$'s for $\nu \le \mu-1$ with $\mu \ge 3$ have been realized, and the next is for $D_{j_{\mu}}$. The method is the same as that of $D_{j_2}$. Since $\bigcup_{\nu = 1}^{\mu-1}{D}_{j_{\nu}} \subset l_{j_u}^0$, the new added linear function on $D_{j_{\mu}}$ cannot disturb those on $D_{j_{\nu}}$'s constructed before. Repeat the induction until $\nu = k$.

Because each subdomain need $n+1$ activated hyperplanes to produce the linear function on it, the number of units of the hidden layer is $k(n+1)$.

\textbf{Case $m > k(n+1)$}: The reason is the same as case $m > 2(n+1)$ of proposition 1 of section 4.1.
\end{proof}

\subsection{Representative Hyperplane}

\begin{figure}[!t]
\captionsetup{justification=centering}
\centering
\includegraphics[width=2.3in, trim = {4.2cm 4.3cm 4.2cm 1.9cm}, clip]{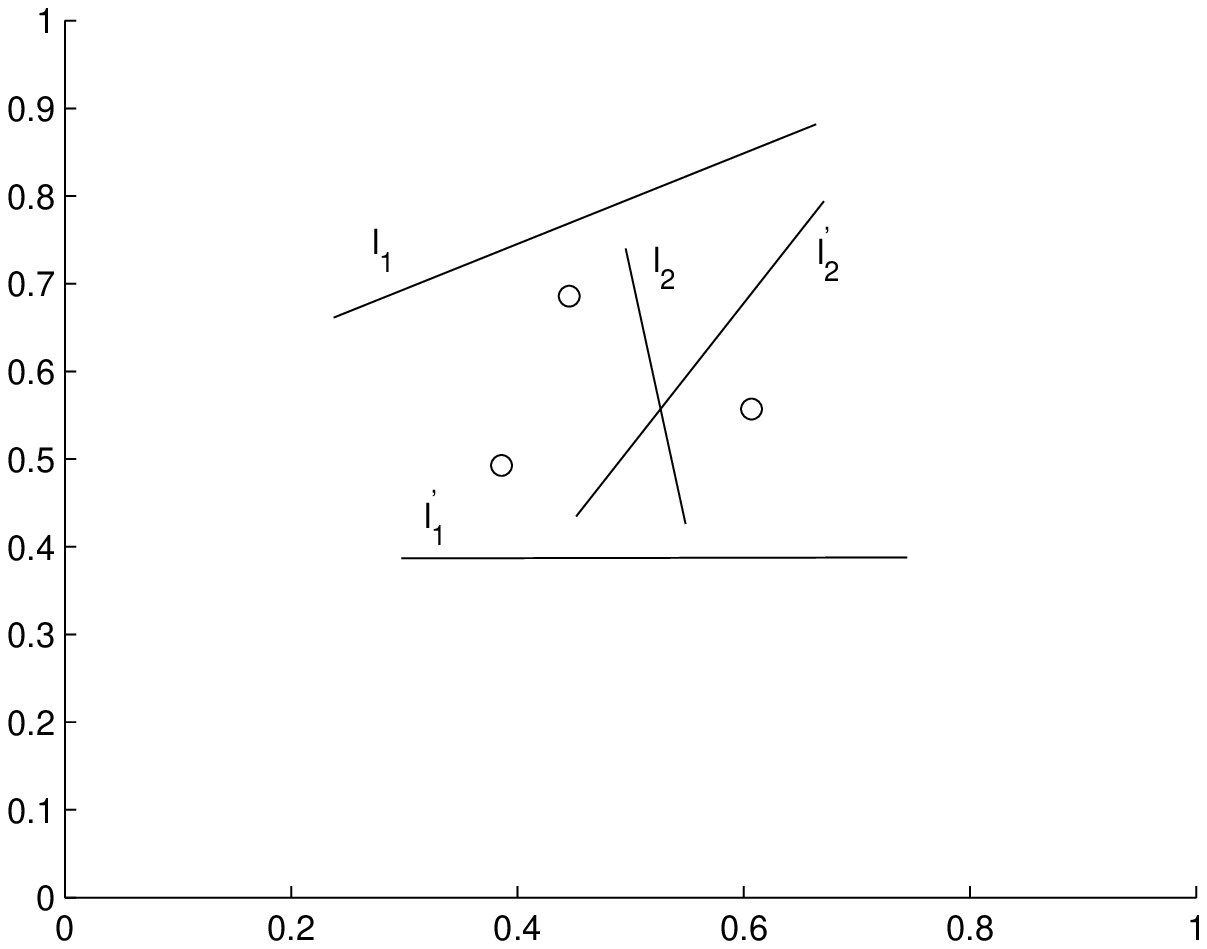}
\caption{Representative hyperplanes.}
\label{Fig.4}
\end{figure}

\begin{dfn}
Let $\mathcal{H}$ be the set of hyperplanes of $n$-dimensional space and $D$ be a set of data points with finite cardinality. A relation $R_D \subseteq \mathcal{H} \times \mathcal{H}$ is defined as: $(l_1, l_2) \in R_D$ means that $l_1$ divides $D$ into the same two subsets as $l_2$, where $l_1, l_2 \in \mathcal{H}$.
\end{dfn}
For example, in Figure \ref{Fig.4}, $(l_1, l_1') \in R_D$, since both $l_1$ and $l_1'$ divide $D$ into the subsets $D$ and $\phi$, and similarly for $(l_2, l_2') \in R_D$. It's easy to verify that $R_D$ is an equivalence relation; so $(l_m, l_n) \in R_D$ can use the notation $l_m \sim l_n$, which is read as ``$l_m$ is equivalent to $l_n$''.

\begin{dfn}
Under the notations of definition 12, denote the quotient set of $\mathcal{H}$ with respect to the relation $R_D$ by
\begin{equation}
\mathcal{H} / R_D = \{\lbrack a \rbrack_{R_D}, a \in \mathcal{H} \},
\end{equation}
where $\lbrack a \rbrack_{R_D} = \{ l \  | \  l \in \mathcal{H} \land a \sim l \}$ is an equivalence class of $\mathcal{H} / R_D$. To each element $\lbrack a \rbrack_{R_D}$ of $\mathcal{H} / R_D$, all the hyperplanes belonging to $\lbrack a \rbrack_{R_D}$ do the same classification to data set $D$; and a representative hyperplane of $\lbrack a \rbrack_{R_D}$ is defined to be any $l \in \lbrack a \rbrack_{R_D}$ that is chosen as the representation of $\lbrack a \rbrack_{R_D}$. Each equivalence class accounts for one representative hyperplane with its distinct classification result.
\end{dfn}

Each representative hyperplane of $D$ corresponds to one way of classifying $D$ into a distinct partition with two subsets. Because the cardinality $|D|$ is finite, the number of distinct two-subset partitions derived from linear classification to $D$ is also finite, and so is the number of representative hyperplanes. This property will be used in the following section.

\subsection{General Result}
Lemma 4 of this section will provide a method of constructing distinguishable data sets. The interpolation capability of three-layer networks for a single output will be given in theorem 2.

\begin{lem}
Let $D$ be a two-category data set of $n$-dimensional space whose each element is marked by either $*$ or $\Delta$; suppose that there's only one $*$-sample in $D$, denoted by $p_*$. Then among all the representative hyperplanes, we can find the one $l_{max}$ classifying $D$ into two subsets $D_*$ and $D_{\Delta}$, such that $p_* \in D_*$ and $D_{\Delta}$ has the maximum number of $\Delta$-samples. And when we translate $l_{max}$ towards $p_*$, it will not meet any $\Delta$-sample before passing through $p_*$.
\end{lem}
\begin{proof}
Because the number of representative hyperplanes is finite, we can always find the one $l_{max}$ contributing to maximum $|D_{\Delta}|$. If during the translation of $l_{max}$, it meets a $\Delta$-sample before reaching $p_*$, then there exists a representative hyperplane resulting in greater $|D_{\Delta}|$; this is a contradiction. The conclusion follows.

Figure \ref{Fig.5}a is an example. Line $l_i$ and pint $p_i$ correspond to $l_{max}$ and the $*$-sample of this lemma, respectively. When $l_i$ approaches $p_i$, it will not pass through any $\Delta$-sample.

\end{proof}

\begin{dfn}
In lemma 3, the representative hyperplane $l_{max}$ with respect to the two-category data set $D$ is called the maximum hyperplane of $D$.
\end{dfn}

\begin{figure}[!t]
\captionsetup{justification=centering}
\centering
\subfloat[Translation of $l_i$.]{\includegraphics[width=2.1in, trim = {4.5cm 2.5cm 4.4cm 3cm}, clip]{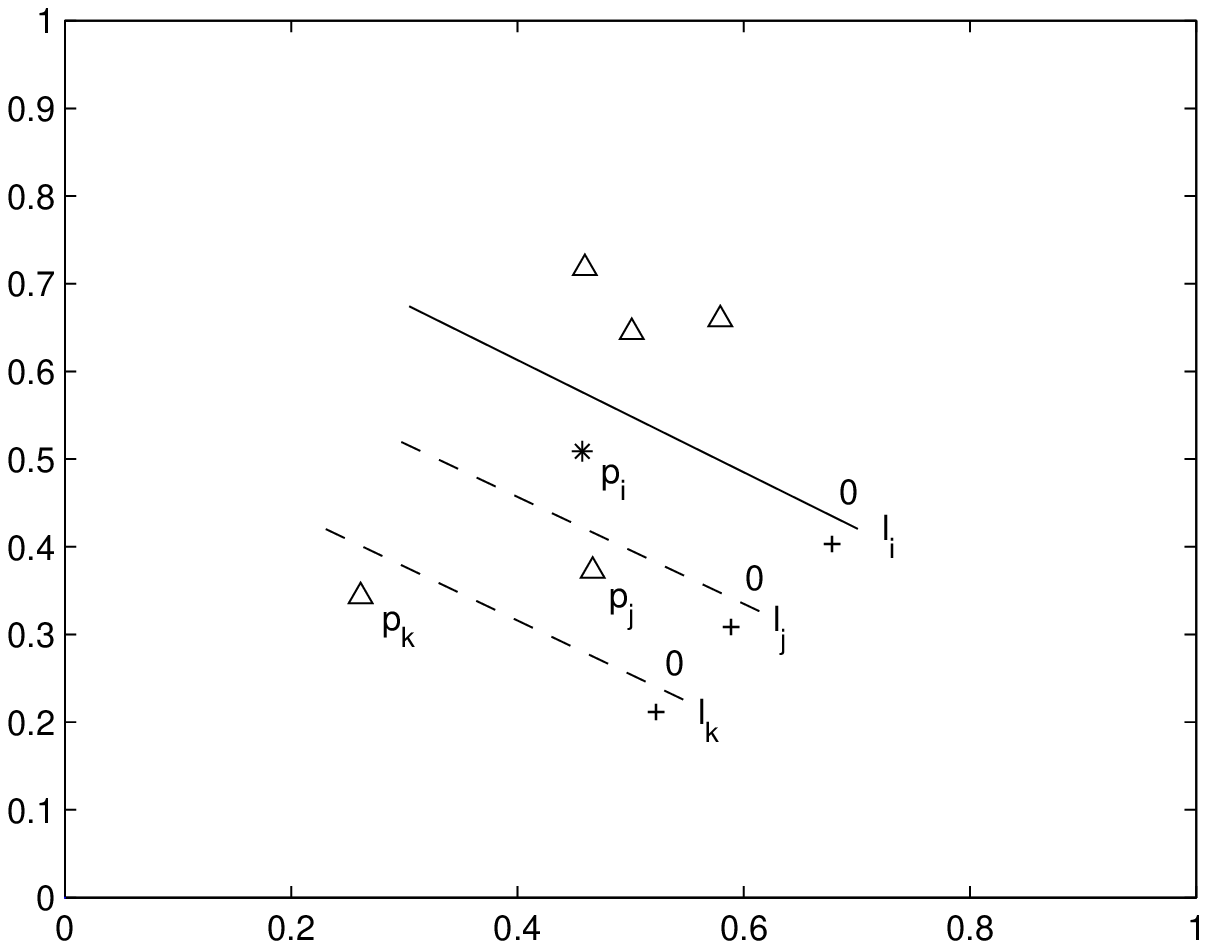}} \quad \quad \quad \quad
\subfloat[Translation of sightly perturbed $l_i$.]{\includegraphics[width=2.1in, trim = {4.3cm 2cm 4.4cm 3cm}, clip]{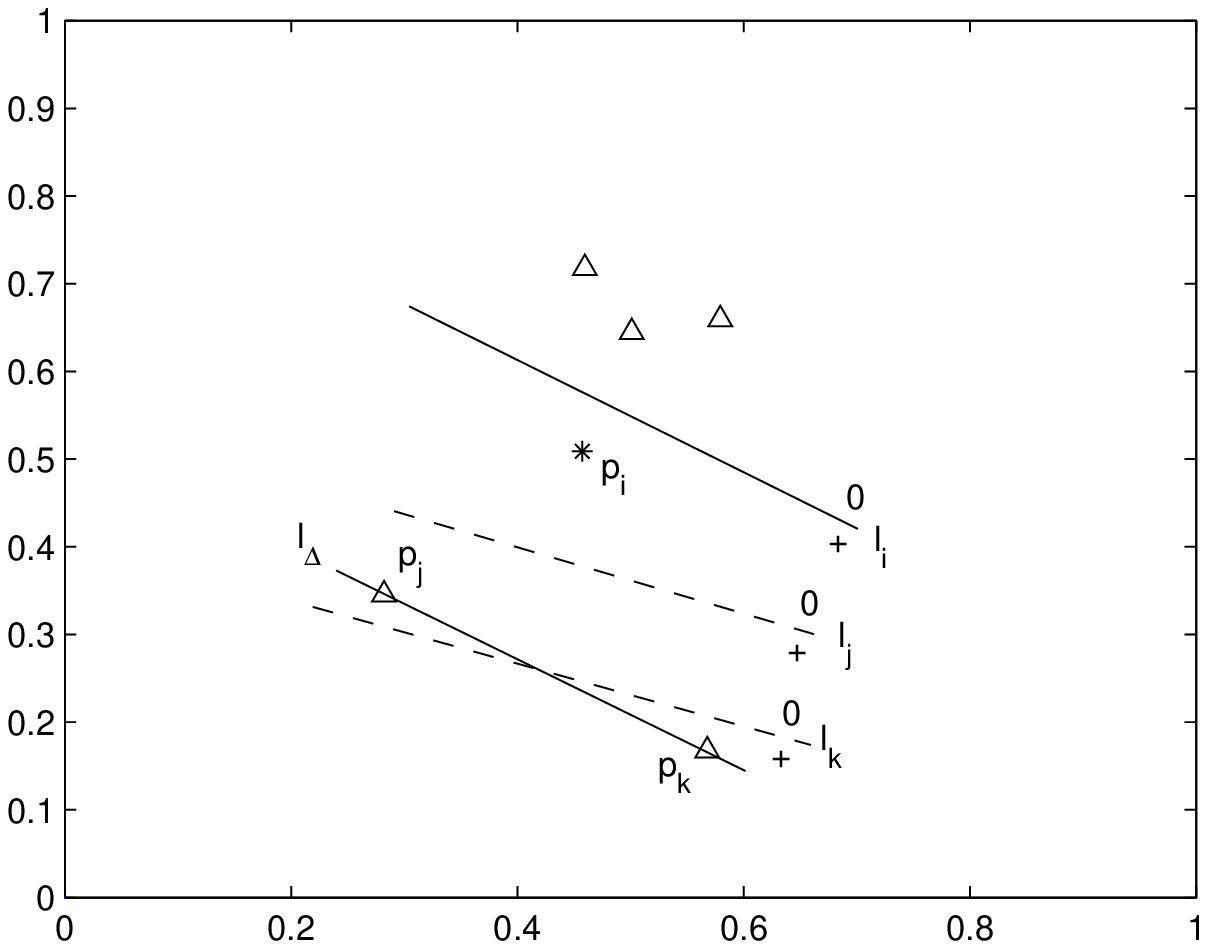}}
\caption{Construction of distinguishable data sets.}
\label{Fig.5}
\end{figure}

\begin{lem}
Given $k$ data sets $D_i$'s of $n$-dimensional space for $i = 1, 2, \cdots, k$, if each of them has only one element $p_i$, then $D_i$'s can be distinguishable.
\end{lem}
\begin{proof}
The proof is constructive. We construct $k$ hyperplanes such that their dividing $D_i$'s satisfies the condition of distinguishable data sets. $D_i$'s will be dealt with one by one according to the ascending order of the subscripts. If $D_i$ fulfils the distinguishable condition, it is called a distinguishable data set and will be put aside for the next one $D_{i+1}$. During the process of the rest of the data sets, $D_i$ should not be influenced; otherwise, it may become undistinguishable again and needs further treatment. This is the general procedure.

Since each data set has only one element, we will use point $p_i$ instead of set $D_i$ to describe the proof. All the terms related to distinguishable data set $D_i$ will be applied to point $p_i$. During the construction, points $p_i$'s for $i = 1, 2, \cdots, k$ are classified into one or two categories. If a point is distinguishable, we call it a $\Delta$-sample as in lemma 3; otherwise, it is a $*$-sample. At the beginning, all of the data pints are $*$-samples; and finally all the $*$-samples will become $\Delta$-samples. The ultimate goal is to rearrange the order of $D_i$'s into $D_{j_{\nu}}$'s for $\nu = 1, 2, \cdots, k$, with a distinguishable hyperplane $l_{j_{\nu}}$ attached to each $D_{j_{\nu}}$.

The proof is by induction and the main inductive procedure is illustrated by an example of Figure \ref{Fig.5}. The five triangles of Figure \ref{Fig.5}a or Figure \ref{Fig.5}b represent five distinguishable points that have been processed and the distinguishable order is $p_{j_1}, \cdots, p_{j_5}$. Next step is to deal with $p_i$. The solution is not evident, since $p_i$ cannot be linearly separated from the five $\Delta$-samples. Let $D_p$ be the set composed of $p_i$ and the existing five $\Delta$-samples. Find the maximum line $l_i$ of $D_p$; then $l_i$ is the distinguishable line of $p_i$, after which $p_i$ becomes a $\Delta$-sample. Add $p_i$ to the rear of the distinguishable order. If there exist no other $\Delta$-samples in $l_i^+$, the operation on $p_i$ stops.

Otherwise, for example, if there's only one $\Delta$-sample $p_{j} \in l_i^+$ for $j < i$ as shown in Figure \ref{Fig.5}a (regardless of the $\Delta$-sample $p_k$), then $p_j$ becomes undistinguishable again since the operation on $p_i$ has influenced it, which needs further treatment. There are two steps required. The first is the update of the distinguishable order. Suppose that $p_j$ is located in the position $j_2$, i.e., $p_{j_2} = p_j$. Delete $p_j$ from $p_{j_1}, \cdots, p_{j_5}$, insert $p_i$ to the rear of the sequence and put $p_j$ after $p_i$; that is, the distinguishable order becomes $p_{j_1}, p_{j_3}, p_{j_4}, p_{j_5}, p_i, p_j$; we can relabel the subscripts such that the altered sequence appears as $p_{j_1}, p_{j_2}, \cdots, p_{j_6}$. The second is to update the distinguishable line of $p_j$. We translate $l_i$ to produce a new distinguishable line of $p_j$, and delete the former one. In Figure \ref{Fig.5}a, $l_i$ is translated to the position of dotted $l_j$; and $l_j$ could be the new distinguishable line of $p_j$, since all the $\Delta$-samples (including the new established $p_i$) are in $l_j^0$ and $p_j \in l_j^+$. We call the above process the \textsl{single $\Delta$-sample procedure}.

The reason that we use the method of translating $l_i$ to produce the new distinguishable line of $p_j$ is that when $l_i$ approaches $p_j$, it will never have the opportunity to influence the $\Delta$-samples in original $l_i^0$, which greatly facilitate the construction of distinguishable data sets.

If there are more than one former $\Delta$-samples in $l_i^+$, we classify the solution into two categories due to different processing methods.

\textbf{Case 1}: If there exist two $\Delta$-samples $p_j, p_k \in l_i^+$ and $p_j, p_k $ are not on a line parallel to $l_i$, the method is as follows. Without loss of generality, suppose that $p_j$ is the first point met by the translated $l_i$. Then the operation on $p_j$ is the same as the \textsl{single $\Delta$-sample procedure} above. After that, to the other $\Delta$-sample $p_k$, two steps are also required. First, change the position of $p_k$ in the distinguishable order into the rear of the sequence. Second, translate $l_i$ until $p_k \in l_i^+$ and the previous $p_j \in l_i^0$ to update the distinguishable line of $p_k$, such as the dotted $l_k$ of Figure \ref{Fig.5}a. Now the operation on $p_i$ stops, including its own treatment as well as that of its influencing $\Delta$-samples $p_j$ and $p_k$.

In general, denote the set of $\Delta$-samples in $l_i^+$ by $D_{\Delta}$ with $|D_{\Delta}| \ge 2$, and suppose that no two of $D_{\Delta}$ are on a line parallel to $l_i$. When translating $l_i$, the first encountered $\Delta$-sample is processed as $p_j$; and the remaining ones are by the method of $p_k$.

\textbf{Case 2}: The more difficult case is when $p_j, p_k \in l_i^+$ but on a line parallel to $l_i$. In Figure \ref{Fig.5}b, the line $l_{\Delta}$ connecting $p_j$ and $p_k$ is parallel to $l_i$. In this case, we cannot distinguish $p_j$ and $p_k$ by the translation of $l_i$. The solution is to generate a slight random disturbance to the weight parameters of $l_i$. Let $\boldsymbol{w_i}^T\boldsymbol{x} + b_i =0$ be the equation of $l_i$, and let $\boldsymbol{\epsilon}$ be an $n \times 1$ vector whose entries are randomly selected from the uniform distribution on interval $(0, 1)$. The perturbed line $l_{\boldsymbol{\epsilon}}$ is expressed as
\begin{equation}
(\boldsymbol{w_i} + \alpha\boldsymbol{\epsilon})^T\boldsymbol{x} + b_i =0,
\end{equation}
where $\alpha$ is a real number whose absolute value can be arbitrarily small. Through the perturbation, $l_{\boldsymbol{\epsilon}}$ is not parallel to $l_i$; and if $|\alpha|$ is small enough, $l_{\boldsymbol{\epsilon}}$ could have the same classification result as $l_i$, which ensures that the $\Delta$-samples in $l_i^0$ are also in $l_{\boldsymbol{\epsilon}}^0$. Then translate $l_{\boldsymbol{\epsilon}}$ instead of $l_i$ to process $p_j$ and $p_k$ as case 1. Figure \ref{Fig.5}b shows the result of two dotted lines as the updated distinguishable lines.

If the collinear case occurs for several times when cardinality $|D_{\Delta}| > 2$, recursively apply the method of equation 4.13. For example, when the translated $l_i$ encounters some $\Delta$-samples lying on a line parallel to $l_i$, randomly perturb it to be $l_{\boldsymbol{\epsilon}_1}$ whose equation is $(\boldsymbol{w_i} + \alpha_1\boldsymbol{\epsilon}_1)^T\boldsymbol{x} + b_i =0$, and then translate $l_{\boldsymbol{\epsilon}_1}$. If again meet $\Delta$-samples on a line parallel to $l_{\boldsymbol{\epsilon}_1}$, change $l_{\boldsymbol{\epsilon}_1}$ to be $(\boldsymbol{w_i} + \alpha_1\boldsymbol{\epsilon}_1 + \alpha_2\boldsymbol{\epsilon}_2)^T\boldsymbol{x} + b_i =0$ of $l_{\boldsymbol{\epsilon}_2}$, and translate $l_{\boldsymbol{\epsilon}_2}$. Note that if $|\alpha_2|$ is small enough, the $\Delta$-samples in $l_{\boldsymbol{\epsilon}_1}^0$ are also in $l_{\boldsymbol{\epsilon}_2}^0$; and if both $|\alpha_1|$ and $|\alpha_2|$ are small enough, those $\Delta$-samples in $l_{i}^0$ also belong to $l_{\boldsymbol{\epsilon}_2}^0$.

Because cardinality $|D_{\Delta}|$ is finite, this disturbance-adding operation would not be done for infinite times. The final line for translation can be written as
\begin{equation}
(\boldsymbol{w_i} + \sum_{\nu = 1}^{\mu}\alpha_{\nu}\boldsymbol{\epsilon}_{\nu})^T\boldsymbol{x} + b_i =0,
\end{equation}
where $\mu$ is the number of the disturbance-adding operations. In equation 4.14, if each $|\alpha_{\nu}|$ is small enough, we can always preserve the classification result before disturbance and avoid the difficulty caused by collinear points.

During the process, the way of updating the distinguishable order is the same as that of case 1. This completes the proof of case 2.

The inductive procedure above is applicable to the general $n$-dimensional case. If we change the term \textsl{line} into \textsl{hyperplane}, the proof still holds.

Now we use mathematical induction to prove this lemma. The induction begins with $p_1$. Select a hyperplane $l_1$ such that $p_1 \in l_1^+$. The case of $p_2$ is also trivial; divide $p_2$ and $p_1$ via hyperplane $l_2$ such that $p_2 \in l_2^+$ and $p_1 \in l_2^0$. The difficulty starts from $p_3$, since it may not be easy to find $l_3$ satisfying the distinguishable condition. Use the inductive method as the example of Figure \ref{Fig.5} to deal with $p_i$ for $i \ge 3$ until $i = k$.

\end{proof}

\begin{thm}
A three-layer network $n^{(1)}m^{(1)}1'^{(1)}$ with $m \ge \nu(n+1)$ can realize any discrete piecewise linear function of equation 4.6, where $\nu = |D|$ is the number of the points of domain $D$.
\end{thm}
\begin{proof}
We decompose the domain $D = \bigcup_{i=1}^{k}D_i$ of a discrete piecewise linear function of equation 4.6 into $D = \bigcup_{j=1}^{\nu}D'_j$ for $j = 1, 2, \cdots, \nu$, where each $D'_j$ contains only one distinct element of $D$. Then by lemma 4 and proposition 3, the conclusion follows.
\end{proof}

\begin{cl}
Any two-category data set $D$ of the $n$-dimensional input space can be classified by a three-layer network $n^{(1)}m^{(1)}1^{(1)}$, provided that there are sufficiently many units of the hidden layer.
\end{cl}
\begin{proof}
By theorem 2, make the outputs of one category positive by linear functions, and make the outputs of the other category zero by constant functions, which is a two-category classification.
\end{proof}

\section{Multi-Output Case \rom{1}}
We'll explain the parameter-sharing mechanism of a three-layer network for multi-outputs in the proof of theorem 3. A solution of the last three layers of convolutional neural networks is given in corollary 4.

\begin{dfn}
A multi-dimensional discrete piecewise linear function
\begin{equation}
f: D \to \mathbb{R}^{\mu},
\end{equation}
where domain $D = \bigcup_{i=1}^{k}D_i \subset \mathbb{R}^n$ with $D_i \cap D_j = \emptyset$ for $i \ne j$ and $j=1, 2, \cdots, k$, is a function composed of $\mu$ discrete piecewise linear functions, each of which is denoted by $f_{\nu}: D \to \mathbb{R}$ for $\nu = 1, 2, \cdots, \mu$  as equation 4.6 and corresponds to the $\nu$th dimension of the codomain $\mathbb{R}^{\mu}$.
\end{dfn}

\begin{figure}[!t]
\captionsetup{justification=centering}
\centering
\includegraphics[width=3.8in, trim = {2.3cm 2cm 1.6cm 3cm}, clip]{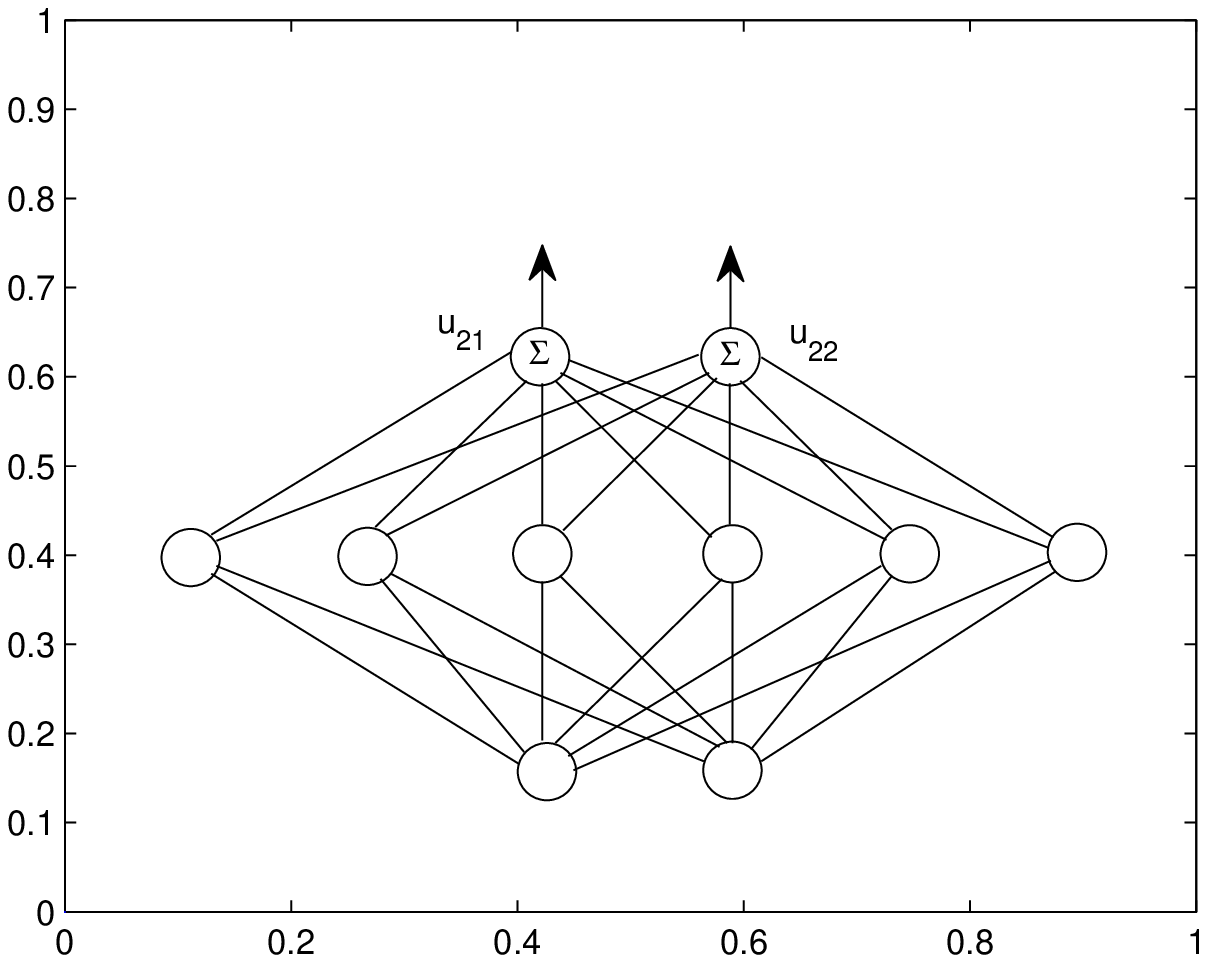}
\caption{Multi-outputs of three-layer networks.}
\label{Fig.6}
\end{figure}

\begin{thm}
Any multi-dimensional discrete piecewise linear function of equation 5.1 can be realized by a three-layer network $n^{(1)}m^{(1)}\mu'^{(1)}$ for $m \ge \nu(n+1)$ with $\nu = |D|$.
\end{thm}
\begin{proof}
Figure \ref{Fig.6} is an example of multi-output networks with $\mu' = 2$. Use the method of theorem 2 and lemma 4 to construct the distinguishable data sets for $D = \bigcup_{i=1}^{k}D_i$, through the hidden layer of $n^{(1)}m^{(1)}\mu'^{(1)}$. Suppose that the hidden layer of Figure \ref{Fig.6} has completed the construction of distinguishable data sets for domain $D$. By proposition 3, unit $u_{21}$ of the output layer can produce any discrete piecewise linear function on $D$ by adjusting its input weights.

The case of the other output unit $u_{22}$ is similar. The parameters of $u_{22}$ can be set independently of $u_{21}$, since they share no common input weights. Thus $u_{22}$ can also produce arbitrary discrete piecewise linear function on $D$.

In general, each linear unit of the output layer can do the same work independently through their independent input weights. This is the key to the proof.
\end{proof}

\begin{rmk}
Generally speaking, to the mechanism of multi-outputs of three-layer networks, the units of the hidden layer divide the domain, while the units of the output layer share the same divided subdomains and realize the linear functions on them independently.
\end{rmk}

\begin{cl}
A three-layer network $n^{(1)}m^{(1)}\mu^{(1)}$ with $\mu$ outputs can classify any $\mu$-category data set $D$ of the $n$-dimensional input space, if the number of units of the hidden layer satisfies $m \ge \nu(n+1)$ with $\nu = |D|$.
\end{cl}
\begin{proof}
To the $i$th unit $u_{2i}$ for $i = 1, 2, \cdots, \mu$ of the output layer, by corollary 3, make its output positive for the $i$th category of $D$, and zero for the remaining categories. By theorem 3, each unit of the output layer can be dealt with independently for the corresponding category.
\end{proof}
\begin{rmk}
This corollary gives one solution of the subnetwork of the last three layers of convolutional neural networks \citep*{LeCun1989, LeCun1998, Krizhevsky2012}, which is fully connected and produces the final output of multi-category classification. It may explain the general classification mechanism of three-layer networks to some extent as well.
\end{rmk}

\section{Deep-Layer Network}
The main difficulty of producing a certain piecewise linear function by three-layer networks is the half-space interference as discussed in section 3.3. The concept of distinguishable data sets of section 4 is to eliminate the disturbance among hyperplanes. To solve this problem via deep-layer networks, \citet{Huang2020} resorted to a network architecture of unconnected independent modules together with a bias constraint.

In this section, we propose another method to avoid the half-space interference via deep layers, with less constraints on network architectures, whose results could explain some networks of engineering.

The interpolation methods of \citet*{DeVore2021} and \citet*{zhang2017} can also yield a solution of deep layers. Ours is distinct in two ways. The first is that it's natural for the latter to be generalized to the approximation to continuous functions, due to its property of region dividing. The second is the clear geometric meaning of each component of network architectures, which could help us to understand the mechanism of neural networks.

In the proof of lemma 6, in order to coordinate different stages of binary classification of different subdomains, redundant subnetworks only for transmitting data via affine transforms are added in the architecture, which can be considered as a type of overparameterization solution. And the further discussion will be in section 8.2.

\subsection{Interference-avoiding Principle}
The next theorem plays a fundamental role in the solution of deep ReLU networks. It in fact solves the problem of the restricted architecture of \citet{Huang2020}, where an exclusively designed bias and independent subnetwork modules must be required to produce piecewise linear approximations. It makes the universal-solution finding of some architectures of engineering possible.

\begin{figure}[!t]
\captionsetup{justification=centering}
\centering
\subfloat[Hyperplanes and data sets.]{\includegraphics[width=2.3in, trim = {4cm 3cm 4cm 2.5cm}, clip]{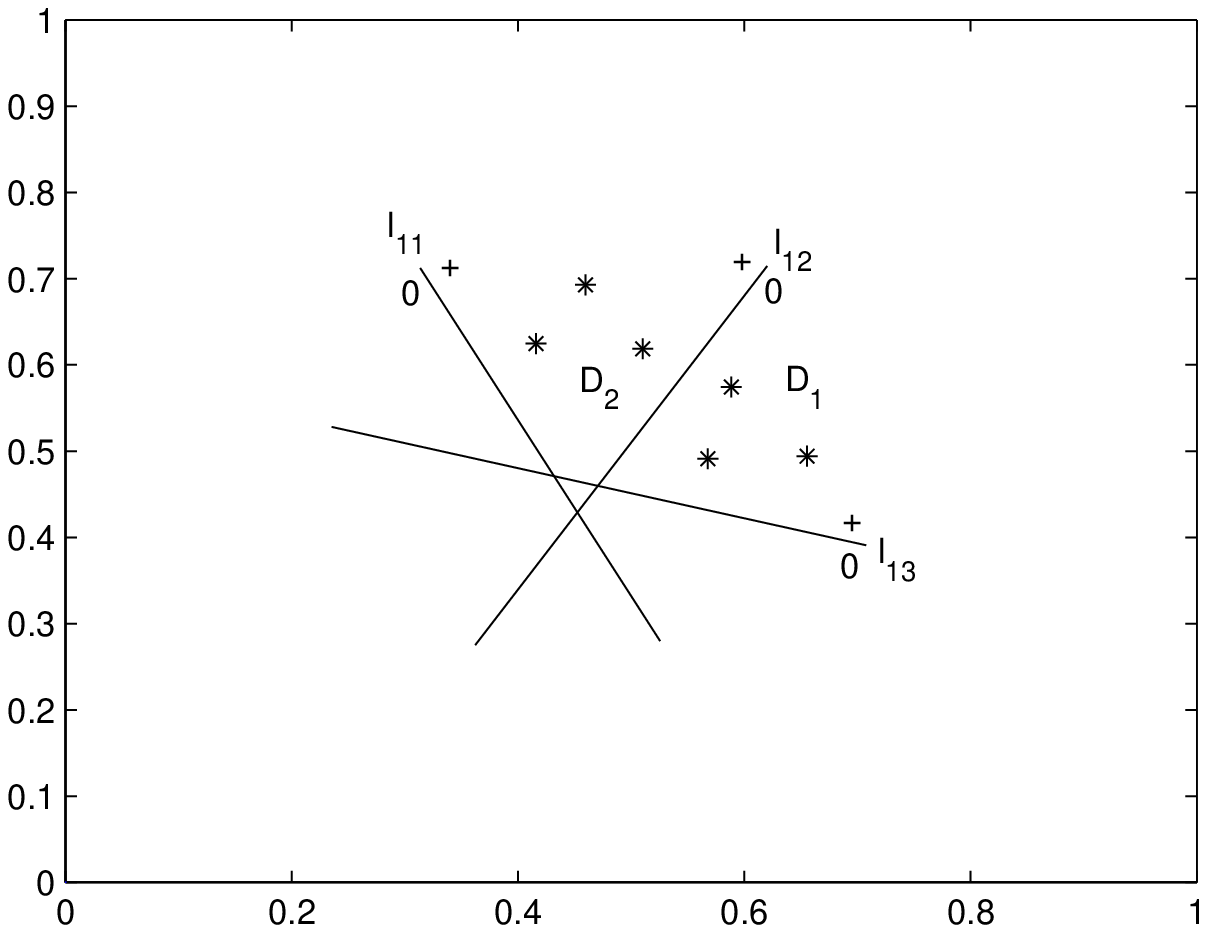}} \quad \quad \quad
\subfloat[Network architecture for (a).]{\includegraphics[width=2.3in, trim = {4cm 4cm 4.5cm 2.5cm}, clip]{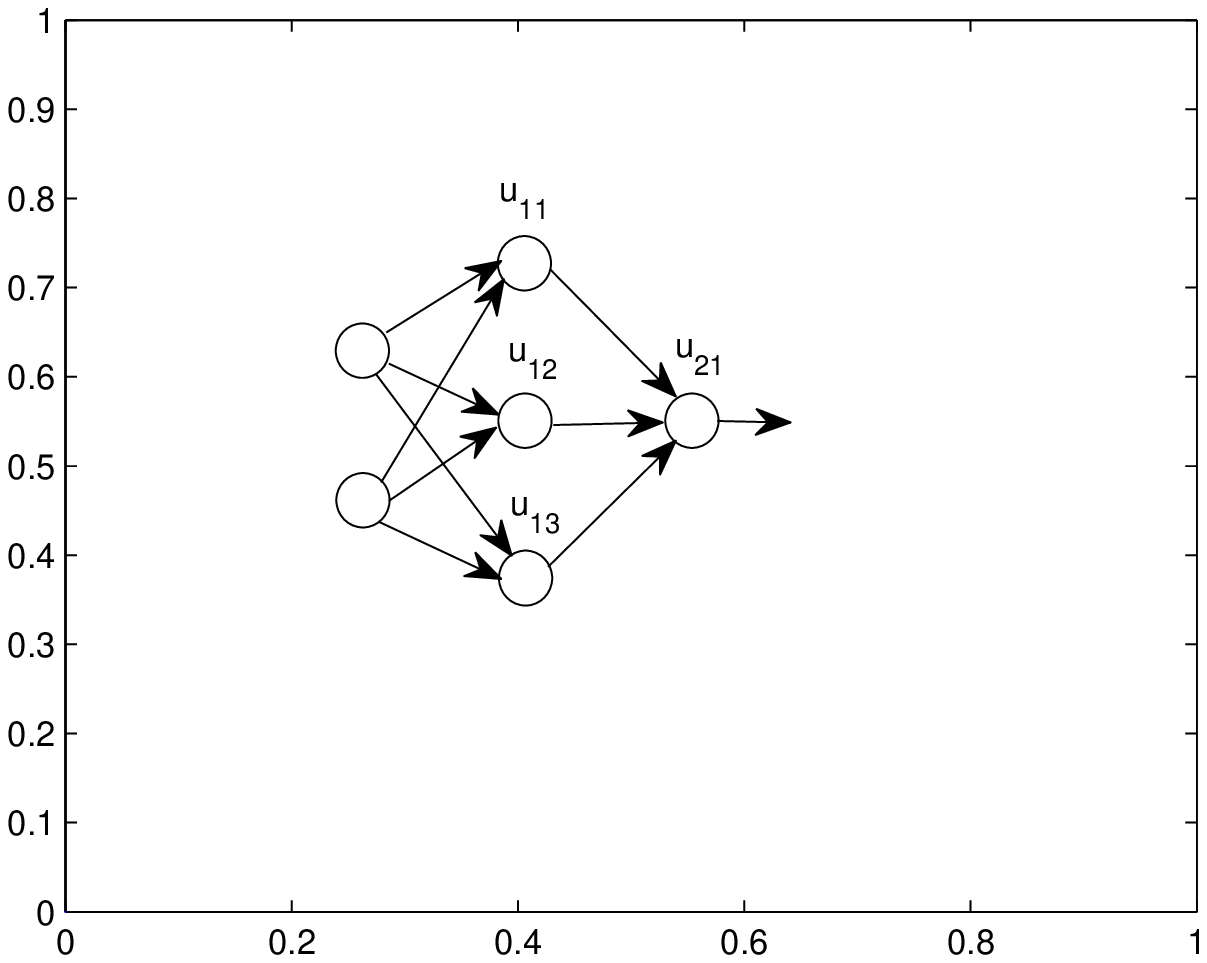}}
\caption{Interference-avoiding principle.}
\label{Fig.7}
\end{figure}

\begin{thm}[Interference-avoiding principle]
Given network $n^{(1)}m^{(1)}1^{(1)}$, let $D_1$ and $D_2$ be two data sets of the $n$-dimensional input space. Suppose that $D_1$ and $D_2$ are separated by a hyperplane $l$ such that $D_1 \subset l^0$ and $D_2 \subset l^+$, where $l$ corresponds to a unit of the first layer. Then in the second layer, if a unit is activated by $D_1$, it could be designed not to be activated by $D_2$.
\end{thm}
\begin{proof}
We first explain the idea by an example. In Figure \ref{Fig.7}a, $D_1$ and $D_2$ are two data sets, both having three elements. Each line $l_{1i}$ for $i = 1, 2, 3$ of Figure \ref{Fig.7}a corresponds to unit $u_{1i}$ of the first layer of network $2^{(1)}3^{(1)}1^{(1)}$ of Figure \ref{Fig.7}b. Let $w_i$ be the output weight of $u_{1i}$ of the first layer, as well as the $i$th input weight of $u_{21}$ of the second layer. As shown in Figure \ref{Fig.7}a, $D_1$ activates $u_{11}$ and $u_{13}$, while $D_2$ activates $u_{11}$, $u_{12}$ and $u_{13}$. There exists line $l_{12}$ such that $D_1 \subset l_{12}^0$ and $D_2 \subset l_{12}^+$.

Let $D_1'$ and $D_{2}'$ be the mapped data sets of $D_1$ and $D_2$ by the first layer, respectively. To any $\boldsymbol{x} \in D_1$, only $u_{11}$ and $u_{13}$ output nonzero values. Thus, we can set weights $w_1$, $w_3$ and bias $b$ of $u_{21}$ to make $D_1'$ activate $u_{21}$. After that, $w_1$, $w_3$ and $b$ are fixed to preserve the result.

Simultaneously, we don't want $u_{21}$ to be activated by $D_2'$. Since $D_1 \subset l_{12}^0$, the output of $u_{12}$ for $D_1$ is zero; so the output weight $w_2$ of $u_{12}$ has no influence on the activation of $u_{21}$ by $D_1$. However, when $D_2$ is the input, the output of $u_{12}$ is nonzero and the weight $w_2$ could influence the activation of $u_{21}$. That's the key point.

To any $\boldsymbol{x}^{(1)} \in D_1' \cup D_2'$, denote the input sum to $u_{21}$ by
\begin{equation}
s = \boldsymbol{w}^T\boldsymbol{x}^{(1)} + b = w_{2}x_{2} + C
\end{equation}
where $\boldsymbol{w} = [w_1, w_2, w_3]^T$, $\boldsymbol{x}^{(1)} = [x_{1}, x_{2}, x_{3}]^T$ is the output vector of the first layer, and $C = w_1x_{1} + w_3x_{3} + b$ with $w_1$, $w_3$ and $b$ fixed for $D_1$. To only $D_2'$, we write equation 6.1 as
\begin{equation}
s_{i} = \boldsymbol{w}^T\boldsymbol{x}^{(1)}_i + b = w_{2}x_{i2} + C_i,
\end{equation}
where $\boldsymbol{x}^{(1)}_i = [x_{i1}, x_{i2}, x_{i3}]^T \in D_2'$ is the $i$th element of $D_2'$ for $i = 1, 2, 3$, and
\begin{equation}
C_i = w_1x_{i1} + w_3x_{i3} + b.
\end{equation}
We should adjust $w_{2}$ of equation 6.2 to make $s_i < 0$ for all $i$, by which unit $u_{21}$ would not be activated by $D_2$.

Because the number of elements of $D_{2}'$ is finite, we have
\begin{equation}
\boldsymbol{0} < \boldsymbol{x}_{min} \le \boldsymbol{x}_{i}^{(1)} \le \boldsymbol{x}_{max},
\end{equation}
where the inequalities or equalities for vectors describe the relation of their entries, such as $\boldsymbol{x}_{min} \le \boldsymbol{x}_{1}$ meaning that each entry of $\boldsymbol{x}_{1}$ is greater than or equal to the corresponding entry of $\boldsymbol{x}_{min}$; $\boldsymbol{0}$ is a zero vector; $ \boldsymbol{x}_{min} > \boldsymbol{0}$ holds since the nonzero output of a ReLU is always greater than zero. Thus, both $x_{i2}$ and $C_i$ of equation 6.2 are bounded and $x_{i2} > 0$. With this constraint, if $w_2$ is a negative number small enough, $s_i <0$ for all $i$.

To the general case of $n^{(1)}m^{(1)}1^{(1)}$, suppose that $D_1 \subset \l_{\nu}^0$ and $D_2 \subset l_{\nu}^+$, where hyperplane $l_{\nu}$ corresponds to the $\nu$th unit of the first layer. The bias and some input weights of $u_{21}$ of the second layer have been set to activate $u_{21}$ by $D_1'$. Let $\boldsymbol{x}^{(1)}_i = [x_{i1} \ x_{i2} \  \cdots \ x_{im}]^T$ be the $i$th element of $D_2'$ for $i = 1, 2, \cdots, k$ with $k = |D_2'|$. We change equation 6.2 into
\begin{equation}
s_i = \boldsymbol{w}^T\boldsymbol{x}_i^{(1)} + b = w_{\nu}x_{i\nu} + C_i,
\end{equation}
where
\begin{equation}
C_i = \sum_{\mu \in A}w_{\mu}x_{i{\mu}} + b
\end{equation}
with
\begin{equation}
A = \{\tau | \tau \ne \nu, l_{\tau} \in l^+(\boldsymbol{x}_{i}^{(1)})\},
\end{equation}
in which $l^+(\boldsymbol{x}_{i}^{(1)})$ is the set of hyperplanes activated by $\boldsymbol{x}_{i}^{(1)}$. In equation 6.6, parameters $w_{\mu}$'s and $b$ are fixed, with some of which for $D_1'$ and others not used (if any).

Note that equation 6.4 also holds for the general case. Then equations 6.5, 6.6 and 6.4 can always yield a solution of $w_{\nu}$ of equation 6.5 such that $s_i < 0$ for all $i$; that is, $u_{21}$ of the second layer cannot be activated by $D_2$. This completes the proof.
\end{proof}

\begin{rmk}
To a linear function instead of a set of discrete points defined on a bounded domain, equation 6.4 still holds; thus, the proof of this theorem is applicable to the case of continuous piecewise linear functions.
\end{rmk}

\begin{cl}
Using the notations of theorem 4, suppose that there exist $k$ hyperplanes for $k \ge 1$ satisfying $D_1 \subset \prod_{\nu = 1}^{k}l_{\nu}^0$ and $D_2 \subset \prod_{\nu = 1}^{k}l_{\nu}^+$. If a unit of the second layer can be activated by $D_1$, we can make it not activated by $D_2$.
\end{cl}
\begin{proof}
Change the term $w_{\nu}x_{i\nu}$ of equation 6.5 into a sum form as
\begin{equation}
s_i = \boldsymbol{w}^T\boldsymbol{x}_i^{(1)} + b = \sum_{\nu = 1}^{k}w_{\nu}x_{i\nu} + C_i.
\end{equation}
The remaining proof is similar to that of theorem 4.
\end{proof}

\begin{cl}
Given network $n^{(1)}1^{(1)}$, $D_1$ and $D_2$ are two data sets of the $n$-dimensional input space. Denote a point of the input space by $\boldsymbol{x} = [x_1,x_2, \cdots, x_n]^T$. Suppose that to $D_1$, there are $k$ dimensions satisfying $x_{i_1} = 0, x_{i_2} = 0, \cdots, x_{i_{k}} = 0$ for $1 \le i_k \le n$ with $1 \le k < n$; and to $D_2$, the coordinate values of the previous $k$ dimensions are positive instead of zero. Then if a unit of the first layer is activated by $D_1$, it can be designed not to be activated by $D_2$.
\end{cl}
\begin{proof}
We consider the input layer of $n^{(1)}1^{(1)}$ as the first layer of $n^{(1)}m^{(1)}1^{(1)}$ of corollary 5, with the number of units changed. Then the condition of this corollary that $k \ge 1$ corresponds to that of corollary 5. Thus the conclusion holds.
\end{proof}

\subsection{Application of the Principle}
We propose a decoder-like architecture to implement a piecewise linear function in theorem 5, as an application of the interference-avoiding principle of theorem 4.
\begin{lem}
To network $n^{(1)}m^{(1)}$, if two data sets $D_1$ and $D_2$ of the $n$-dimensional input space are linearly separable, we can construct $m = k_1 + k_2$ hyperplanes corresponding to the units of the first layer, such that $D_1 \subset \prod_{i=1}^{k_1}l_i^+\prod_{j=k_1+1}^{m}l_j^0$ and $D_2 \subset \prod_{i=1}^{k_1}l_i^0\prod_{j=k_1+1}^{m}l_j^+$. To $D_{\nu}$ for $\nu = 1, 2$, let $D_{\nu}'$ be the mapped data set of $D_{\nu}$ by the the fist layer; and if $k_1 = k_2 = n$, $D_{\nu}'$ is an affine transform of $D_{\nu}$ for all $\nu$ by the construction method.
\end{lem}
\begin{proof}
The proof is on the basis of theorem 4 of \citet{Huang2020}, whose thought was also mentioned in proposition 1 of this paper.

Since $D_1$ and $D_2$ are linearly separable, a hyperplane $l_1$ could be found with $D_1 \subset l_1^+$ and $D_2 \subset l_1^0$. Then construct other $k_1-1$ hyperplanes $l_i$'s for $i = 2, 3, \cdots, k_1$ by the method of theorem 4 of \citet{Huang2020}, having the same classification effect as $l_1$, that is, $D_1 \subset l_i^+$ and $D_2 \subset l_i^0$ for all $i$.

Also according to $l_1$, by the same method, construct other $k_2$ hyperplanes with the output property reversed in comparison with the above case, such that $D_1 \subset \prod_{j=k_1 + 1}^{m}l_j^0$ and $D_2 \subset \prod_{j=k_1 + 1}^{m}l_j^+$.

To the mapped $D_1'$ of $D_1$ by the first layer, each of its elements can be expressed as
\begin{equation}
\boldsymbol{x}^{(1)} = \begin{bmatrix} {{\boldsymbol{x}}'}^T, \boldsymbol{0}^T \end{bmatrix}^T,
\end{equation}
where the nonzero subvector $\boldsymbol{x}'$ of size $k_1 \times 1$ comes from $\prod_{i=1}^{k_1}l_i^+$, and zero subvector $\boldsymbol{0}$ of size $(m-k_1) \times 1$ from $\prod_{j=k_1+1}^{m}l_j^0$. The zero part $\boldsymbol{0}$ of equation 6.9 has no influence on the data structure of $D_1'$. That is, we can consider $D_1'$ equivalent to a corresponding set whose each element is in the form of ${\boldsymbol{x}}'$ of equation 6.9.

By the construction method of \citet{Huang2020}, when $k_1 = n$, $\boldsymbol{x}'$ is an affine transform of $\boldsymbol{x}$ of the input space. Thus, $D_1'$ is an affine transform of $D_1$, and similarly for the case of $D_2'$.
\end{proof}

The proposition below is an example of how the interference-avoiding principle could be used to avoid the disturbances among hyperplanes through deep layers, which can help to understand the more general case of lemma 6.

\begin{figure}[!t]
\captionsetup{justification=centering}
\centering
\includegraphics[width=3.5in, trim = {2.1cm 2.3cm 2cm 1cm}, clip]{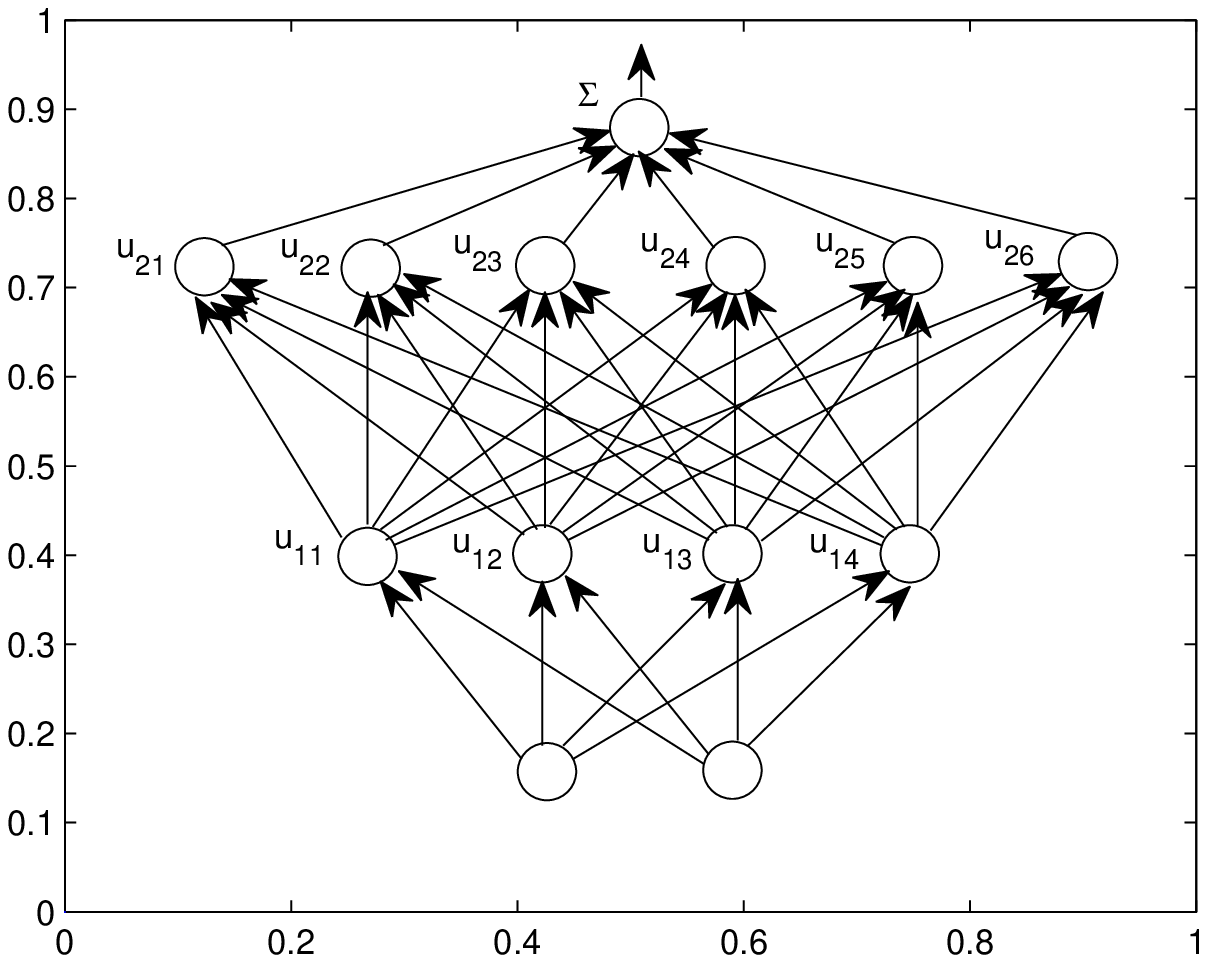}
\caption{Application of the interference-avoiding principle.}
\label{Fig.8}
\end{figure}

\begin{prp}
Let $D_1$ and $D_2$ be two data sets of two-dimensional input space, which are linearly separable. Any piecewise linear function of equation 4.6 on domain $D = D_1 \cup D_2$ could be realized by a fourth-layer network $2^{(1)}4^{(1)}6^{(1)}1'$ as shown in Figure \ref{Fig.8}. The construction of a linear function on one of the two subdomains (such as $D_1$) has no impact on that of the other one (such as $D_2$), through the parameter setting of more than one hidden layers.
\end{prp}
\begin{proof}
Because $D_1$ and $D_2$ are linearly separable, by lemma 5, we can find four lines such that $D_1 \subset l_{11}^+l_{12}^+l_{13}^0l_{14}^0$ and $D_2 \subset l_{11}^0l_{12}^0l_{13}^+l_{14}^+$, where line $l_{1j}$ for $j=1, 2, 3, 4$ corresponds to unit $u_{1j}$ of the first layer; and the mapped data sets $D_1'$ and $D_2'$ by the first layer are the affine transforms of $D_1$ and $D_2$, respectively.

In the first layer, only $u_{11}$ and $u_{12}$ are activated by $D_1$, having nonzero output when the input is $D_1$. We want units $u_{21}$, $u_{22}$ and $u_{23}$ of the second layer activated by $D_1'$, but not activated by $D_2'$. For example, to $u_{21}$, adjust its input weights associated with $u_{11}$ and $u_{12}$, as well as its bias, to make it activated by $D_1'$, regardless of the weights related to $u_{13}$ and $u_{14}$ due to their zero outputs for $D_1$. After that, fix the adjusted parameters of $u_{21}$, preserving the activation by $D_1$.

For the deactivation of $u_{21}$ by $D_2$, because $D_1 \subset l_{13}^0l_{14}^0$ and $D_2 \subset l_{13}^+l_{14}^+$, by corollary 5, $u_{21}$ could be designed not to be activated by $D_2$ through the parameter setting of output weights of $u_{13}$ and $u_{14}$. The cases of $u_{22}$ and $u_{23}$ are similar.

In the same way, units $u_{24}$, $u_{25}$ and $u_{26}$ of the second layer could be activated by $D_2$ but not by $D_1$. The next step is to realize the linear function on subdomains $D_1$ and $D_2$.

Because $D_1 \subset l_{11}^+l_{12}^+$, the mapped data set $D_1'$ by the first layer is an affine transform of $D_1$. By lemma 1, the network can produce any linear function on $D_1'$ via its activated units $u_{21}$, $u_{22}$ and $u_{23}$ of the second layer, through which the one on $D_1$ could be realized by the property of affine transforms (lemma 10 of \citet{Huang2020}). Units $u_{24}$, $u_{25}$ and $u_{26}$ cannot influence this process, since they are not activated by $D_1$. Similarly, we can produce the linear function on $D_2$ and the construction has no impact on $D_1$, due to the deactivation of $u_{21}$, $u_{22}$ and $u_{23}$ by $D_2$.
\end{proof}

\begin{rmk}
If the six units of the second layer of $2^{(1)}4^{(1)}6^{(1)}1'$ of Figure \ref{Fig.8} are the units of the hidden layer of a three-layer network, we cannot arbitrarily arrange the corresponding six hyperplanes to realize a piecewise linear function, due to the interference among hyperplanes; the concept of distinguishable data sets was used then in section 4 to solve this problem. However, by adding one layer, for example, the hyperplanes of units $u_{21}$, $u_{22}$ and $u_{23}$ for subdomain $D_1$ can be arbitrarily placed, regardless of disturbing $D_2$. Although simple, this example reveals an intrinsic advantage of deep layers.
\end{rmk}

\begin{figure}[!t]
\captionsetup{justification=centering}
\centering
\subfloat[A domain to be divided.]{\includegraphics[width=2.3in, trim = {4.0cm 3.8cm 4cm 1cm}, clip]{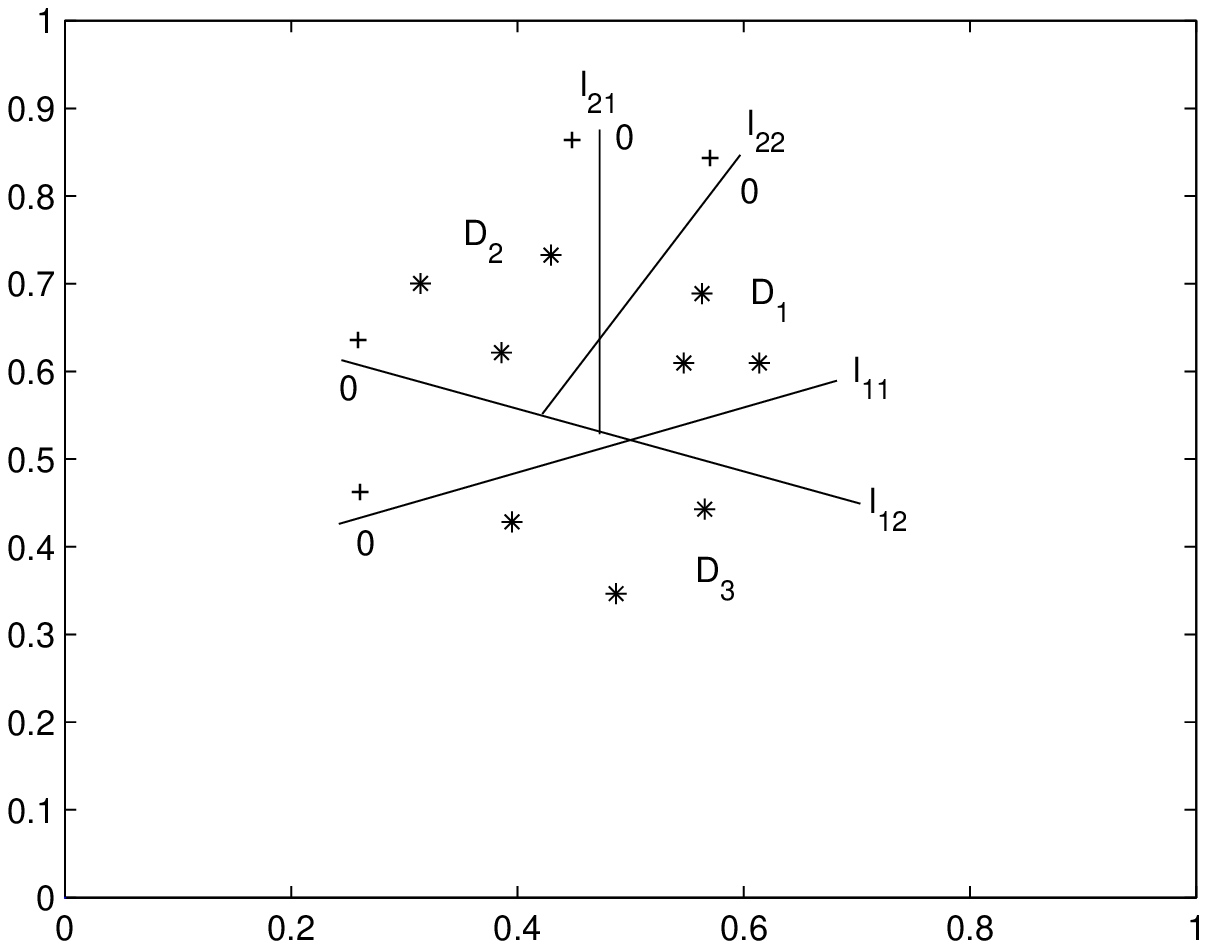}} \quad \quad \quad \quad \quad
\subfloat[Convex-polygon separation.]{\includegraphics[width=2.2in, trim = {4cm 4.1cm 5cm 1cm}, clip]{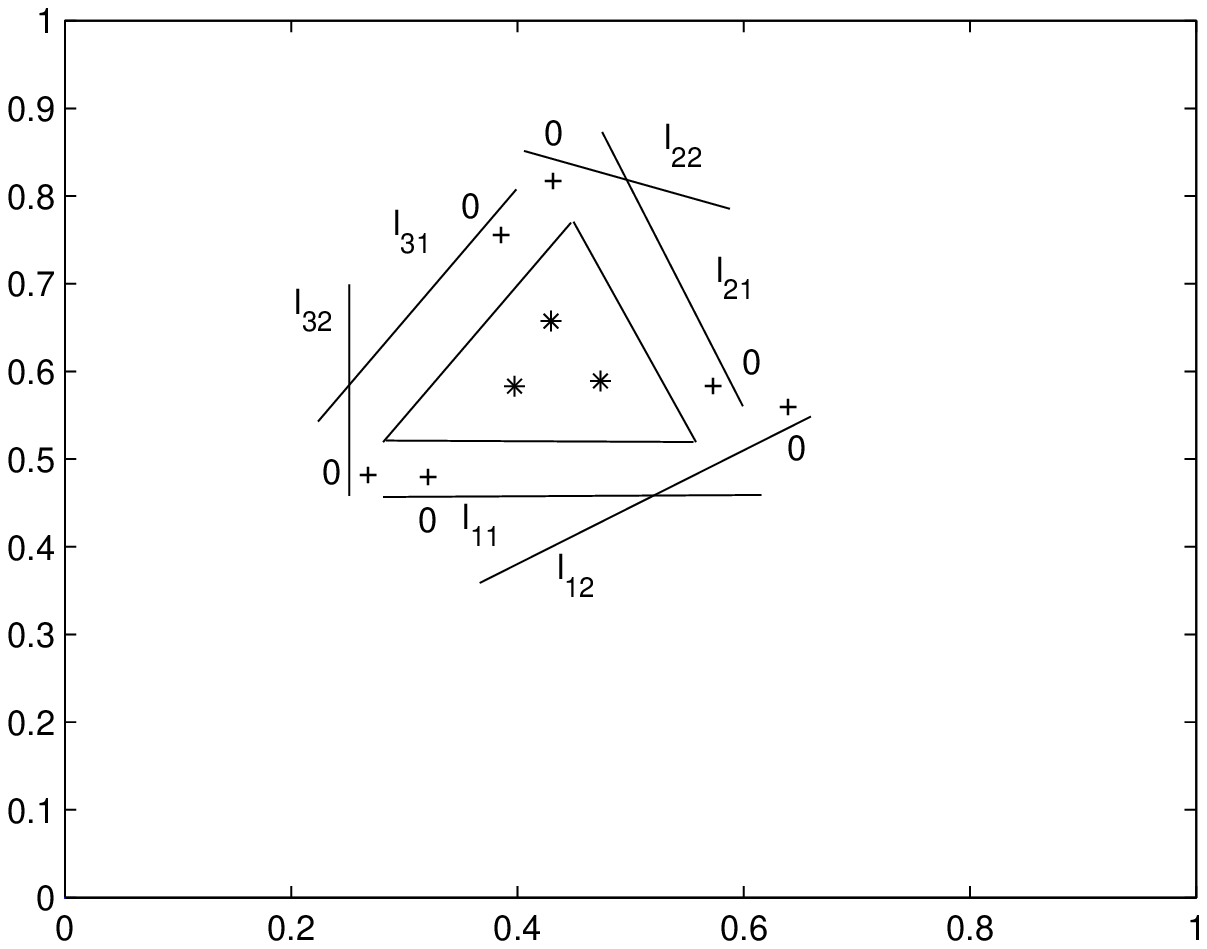}} \\
\subfloat[Network architecture for (a).]{\includegraphics[width=3.2in, trim = {2.4cm 1.4cm 3cm 1cm}, clip]{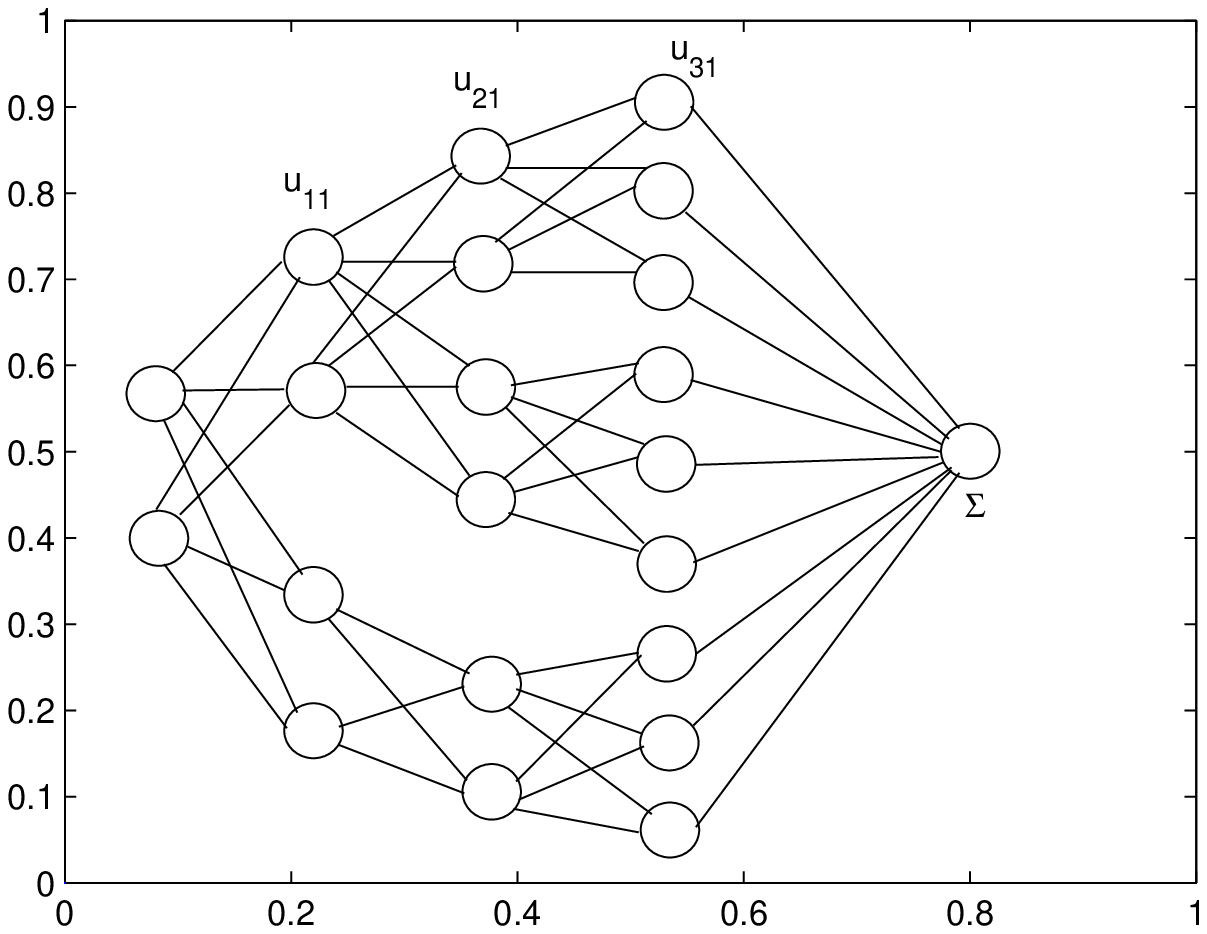}}
\caption{Implementation of piecewise linear functions.}
\label{Fig.9}
\end{figure}

To some concepts of a polytope used in this paper, we refer the reader to \citet*{Grunbaum2003}, such as \textsl{convex polytope}, \textsl{facet} and \textsl{halfspace}. The definition below is on the base of those concepts.
\begin{dfn}
An open convex polytope is a convex one that excludes its facet boundary. Intuitively speaking, if a convex polytope doesn't contain its boundary, it is an open convex polytope. It is the intersection of some open halfspaces without considering the hyperplanes that form the halfspaces. Correspondingly, we also call the usual convex polytope a closed one. Each open convex polytope corresponds to a closed one that contains the facets, and vice versa.
\end{dfn}

\begin{lem}
Let $f: D \to \mathbb{R}$ be an arbitrary discrete piecewise linear function of equation 4.6 on domain $D = \bigcup_{i=1}^{k}D_i$. Suppose that each subdomain $D_i$ is contained in an open convex polytope $\mathcal{P}$ and $(D - D_i) \cap \mathcal{P} = \emptyset$. Then network $n^{(1)}\prod_{j=1}^{d}m_j^{(1)}1'^{(1)}$ with $m_j \ge m_{j-1}$ for $j \ge 2$ could realize it, provided that depth $d$ and widths $m_j$'s are sufficiently large.
\end{lem}
\begin{proof}

We first introduce the idea by an example. As shown in Figure \ref{Fig.9}a, a certain piecewise linear function on domain $D = D_1 \cup D_2 \cup D_3$ is to be realized by a deep-layer network. Figure \ref{Fig.9}c is the network architecture for Figure \ref{Fig.9}a, which is a fully connected network but omitting some connections in the figure in order to emphasize the different modules of linear components. Each unit of the hidden layers is denoted by $u_{ij}$, which means the $j$th unit of the $i$th layer. For simplicity, we only label the first unit of each hidden layer in Figure \ref{Fig.9}c, including $u_{11}$, $u_{21}$ and $u_{31}$; the arrows of connections between units are also omitted. Line $l_{ij}$ for $i, j = 1, 2$ of Figure \ref{Fig.9}a corresponds to unit $u_{ij}$ of Figure \ref{Fig.9}c.

We show how the architecture of Figure \ref{Fig.9}c could implement a piecewise linear function on $D$ of Figure \ref{Fig.9}a. Write $D_{12} = D_1 \cup D_2$; then $D_{12}$ and $D_3$ are linearly separable, as can be seen from Figure \ref{Fig.9}a. Let $D_i'$ for $i = 1, 2, 3$ be the mapped data set of $D_i$ by the first layer, and let $D_{12}' = D_1' \cup D_2'$.

By lemma 5, in the first layer, we construct four lines such that $D_{12} \subset l_{11}^+l_{12}^+l_{13}^0l_{14}^0$ and $D_3 \subset l_{11}^0l_{12}^0l_{13}^+l_{14}^+$. So when $D_{12}$ is the input, only $u_{11}$ and $u_{12}$ have nonzero output; and $D_{12}'$ is an affine transform of $D_{12}$. To $D_{12}'$, in the second layer, construct four lines to further divide it, with $D_1' \subset l_{21}^+l_{22}^+l_{23}^0l_{24}^0$ and $D_2' \subset l_{21}^0l_{22}^0l_{23}^+l_{24}^+$. After that, because $D_{12} \subset l_{13}^0l_{14}^0$ and $D_{3} \subset l_{13}^+l_{14}^+$, by corollary 5, set the input weights of $u_{2j}$ for $j = 1, 2, 3, 4$ associated with $u_{13}$ and $u_{14}$ to make $u_{2j}$ not activated by $D_3'$.

Note that in Figure \ref{Fig.9}a, lines $l_{21}$ and $l_{22}$ should have been depicted in an affine-transform region of $l_{11}^+l_{12}^+$; however, the effect is equivalent in terms of line dividing of data points due to the property of affine transforms.

In the architecture of Figure \ref{Fig.9}c, when the input is $D_3$, the outputs of $u_{11}$ and $u_{12}$ are both zero, through which $D_3$ is excluded. We want $D_3$ still to be excluded by the succeeding layers of the upper subnetwork for dividing $D_{12}$, in terms of zero outputs as well. This is not trivial, for which in \citet*{Huang2020}, we proposed an exclusive subnetwork called a ``T-bias'' to realize that. By the proof above, we saw that the interference-avoiding principle could solve this problem by a usual network architecture of engineering, without the help of a T-bias.

Return to the construction procedure. When $D_3$ is the input, since it needs not be subdivided, just transmit it through layers via affine transforms. In the second layer, construct two units $u_{25}$ and $u_{26}$ such that $D_3' \subset l_{25}^+l_{26}^+$. And because $D_{3} \subset l_{11}^0l_{12}^0$ and $D_{12} \subset l_{11}^+l_{12}^+$, by corollary 5, set relevant parameters to make $u_{25}$ and $u_{26}$ not activated by $D_{12}$.

Denote an output vector of the second layer by
\begin{equation}
\boldsymbol{x}^{(2)} = \begin{matrix}
[x_1, x_2, x_3, x_4, x_5, x_6]^T
\end{matrix},
\end{equation}
where $x_i$ for $i = 1, 2, \cdots, 6$ corresponds to the output of $u_{2i}$. Let $D_j''$ for $j = 1, 2, 3$ be the mapped data set of $D_j'$ by the second layer. When $D_1$ is the input, to $D_1''$, only dimensions $x_1$ and $x_2$ associated with $u_{21}$ and $u_{22}$ have nonzero coordinate values; $D_2''$ is of nonzero in $x_3$ and $x_4$, and $D_3''$ in $x_5$ and $x_6$. By corollary 6, $u_{31}$, $u_{32}$ and $u_{33}$ can be designed to be activated only by $D_1''$; $D_2''$ activates $u_{34}$, $u_{35}$ and $u_{36}$, while $u_{37}$, $u_{38}$ and $u_{39}$ are left for $D_3''$.

Finally, by lemma 1 and the property of affine transforms (lemma 10 of \citet{Huang2020}), any linear function on $D_1$, $D_2$ or $D_3$ could be realized in the output layer.

We summarize the above construction process as five parts. First, subdivide the domain recursively by adding layers until only one subdomain left. Second, use lemma 5 to construct hyperplanes to satisfy the condition of the interference-avoiding principle, and to transmit the data via affine transforms. Third, the interference-avoiding principle is used to make the subdividing in a separated region without influencing other data points. Fourth, data points that need not be subdivided could be transmitted to succeeding layers via affine transforms. Fifth, the output layer produces the linear function on each subdomain by lemma 1 and the property of affine transforms.

The proof of the general case is the repeated application of the above five parts. If each subdomain $D_i$ is contained in an open convex polytope and the remaining ones satisfy $(D - D_i) \cap \mathcal{P} = \emptyset$, such as Figure \ref{Fig.9}b, we can always use the above method to separate $D_i$ and transmit it to the last hidden layer in the sense of affine transforms. Then the linear function on $D_i$ could be realized in the output layer.

To the feature of the network architecture, as the example of Figure \ref{Fig.9}c, when data subdividing is required, in the corresponding subnetwork, the number of units needed in the next layer is twice the number of input units of the current layer. And when we only need affine transforms to transmit the data, in the relevant subnetwork, the number of units of two adjacent layers is equal. The last hidden layer has the maximum number of units because of lemma 1. Thus we have $m_j \ge m_{j-1}$ for $j \ge 2$ of $n^{(1)}\prod_{j=1}^{d}m_j^{(1)}1'^{(1)}$. This completes the proof.
\end{proof}

\begin{rmk-8}
From the proof above, we see that less units in shallower layers are suitable for coarse region dividing, while more units in deeper layers are for finer region dividing.
\end{rmk-8}

\begin{rmk-8}
By the interference-avoiding principle, we can eliminate the disturbance among hyperplanes for different linear components of a piecewise linear function, without resorting to a specially designed subnetwork whose parameter setting or architecture is fixed or constrained, such as the T-bias of \citet*{Huang2020}. Thus, the solution constructed by this method is closer to the one used in engineering.
\end{rmk-8}

\begin{thm}
The network $n^{(1)}\prod_{j=1}^{d}m_j^{(1)}1'^{(1)}$ for $m_j \ge m_{j-1}$ when $j \ge 2$ can realize any discrete piecewise linear function of equation 4.6, with depth $d$ and widths $m_j$'s large enough.
\end{thm}
\begin{proof}
If the condition of lemma 6 is not satisfied, subdivide $D_i$ into some subsets, such that each of them could fulfil that condition.
\end{proof}

\begin{rmk}
The network architecture of this theorem fits the decoder of autoencoders \citep*{Hinton2006}, whose main feather is that the number of units of hidden layers increases monotonically as the depth of the layer grows. Thus, we in fact find a solution of decoders.
\end{rmk}

\section{Affine-Transform Generalization}
We generalize theorem 5 to more types of network architectures, particularly for those applied in practice. The generalizations are composed of two parts. The first is related to affine transforms (this section), and the second is for the output layer relevant to lemmas 1 and 2 (next section).

Sections 7 and 8 also serve as improving the universality of our constructed solutions, whose results are the crucial ingredients of the theories, through which some network architectures of engineering could be explained.

In this section, the affine-transform generalization aims at the problem that when a data set of the input space is embedded in a higher-dimensional space, how it can be transmitted to succeeding layers via affine transforms, by which a certain linear function on it could be realized in the output layer. This generalization could make the parameter setting of the number of the units of hidden layers flexible, with less constraints on network architectures.

Another issue is the overparameterization solution in terms of affine transforms. We know that for an $n$-dimensional input, $n$ units of the next layer are enough to produce an affine transform. When the number of units is greater than $n$, which means that the parameters are redundant, it is the case to be discussed in this section.

The main conclusions are summarized in section 7.4 by theorems 7 and 8. Yet, most of the intermediate results are also important. For instance, theorem 6 is the geometric knowledge of a basic phenomenon; corollary 7 is related to the classification of the data embedded in a higher-dimensional space, and proposition 6 provides a method of constructing affine transforms for that kind of data.

\subsection{Geometric Preliminaries}
\begin{lem}
Given network $n^{(1)}m^{(1)}$ for $m > n$, let $l_n$ be the part of the $n$-dimensional input space that simultaneously activates the $m$ units of the first layer. Then the mapped $l_n'$ of $l_n$ by the first layer lies on an $n$-dimensional subspace of $\boldsymbol{X}_m$, where $\boldsymbol{X}_m$ is the $m$-dimensional space of the first layer, or we say that $l_n'$ is on an $n$-dimensional hyperplane embedded in $\boldsymbol{X}_m$.
\end{lem}
\begin{proof}

We can imagine that an one-dimensional line could be put in a three-dimensional space, with only its location changed. That is, a line can be represented by three-dimensional vectors, or a line is embedded in a three-dimensional space in the topological language. The proof is based on this thought.

By the assumption of section 2.4, the rank of the $m \times n$ input weight matrix $\boldsymbol{W}$ of the first layer is $n$. Without loss of generality, we arrange the order of units of the first layer such that
\begin{equation}
\boldsymbol{W} =
\begin{bmatrix} \boldsymbol{W}_n^T, \boldsymbol{W}_{r}^T\end{bmatrix}^T,
\end{equation}
where $\boldsymbol{W}_n$ is an $n \times n$ nonsingular submatrix that makes the rank of $\boldsymbol{W}$ to be $n$, and $\boldsymbol{W}_{r}$ is the remaining part of $\boldsymbol{W}$ besides $\boldsymbol{W}_n$, whose size is $(m-n) \times n$.

Let $\boldsymbol{x}$ and $\boldsymbol{x}^{(1)}$ be the vectors of $l_n$ and $\boldsymbol{X}_m$, respectively. Then the nonzero output of the first layer is
\begin{equation}
\boldsymbol{x}^{(1)} = \boldsymbol{W}\boldsymbol{x} + \boldsymbol{b} =
\begin{bmatrix}
\boldsymbol{W}_n\boldsymbol{x} + \boldsymbol{b}_n \\
\boldsymbol{W}_r\boldsymbol{x} + \boldsymbol{b}_r
\end{bmatrix},
\end{equation}
where $\boldsymbol{b}$ is the bias vector of the first layer, while $\boldsymbol{b}_n$ and $\boldsymbol{b}_r$ are subvectors of $\boldsymbol{b}$ decomposed according to equation 7.1. Because $\boldsymbol{W}_n$ is nonsingular, equation 7.2 can be expressed as
\begin{equation}
\boldsymbol{x}^{(1)} =
\begin{bmatrix}
\boldsymbol{x}' \\
\boldsymbol{x}_c'
\end{bmatrix},
\end{equation}
where
\begin{equation}
\boldsymbol{x}' = \boldsymbol{W}_n\boldsymbol{x} + \boldsymbol{b}_n
\end{equation}
is an affine transform of $\boldsymbol{x}$ of the input space, and
\begin{equation}
\boldsymbol{x}_c' = \boldsymbol{W}_r\boldsymbol{x} + \boldsymbol{b}_r,
\end{equation}
whose dimensionality is $m - n$. Equations 7.3, 7.4 and 7.5 imply
\begin{equation}
\boldsymbol{x}^{(1)} =
\begin{bmatrix} \boldsymbol{x}' \\
\boldsymbol{W}_c\boldsymbol{x}' + \boldsymbol{b}_c
\end{bmatrix},
\end{equation}
where $\boldsymbol{W}_c = \boldsymbol{W}_r\boldsymbol{W}_n^{-1}$ and $\boldsymbol{b}_c = -\boldsymbol{W}_r\boldsymbol{W}_n^{-1}\boldsymbol{b}_n + \boldsymbol{b}_r$.

Based on equation 7.6, we analyze the output of the first layer from geometric viewpoints. For instance, suppose that the input is an one-dimensional line of the input space whose parametric equation is $\boldsymbol{x}_0 + t\boldsymbol{\lambda}$. Denote by $l$ the part of this line that simultaneously activates the $m$ units of the first layer. After passing through the first layer, by equation 7.6, the output of $l$ is
\begin{equation}
\boldsymbol{x}^{(1)}_l = \begin{bmatrix}
\boldsymbol{x}'_0 + t\boldsymbol{\lambda}' \\
\boldsymbol{W}_c(\boldsymbol{x}'_0 + t\boldsymbol{\lambda}') + \boldsymbol{b}_c
\end{bmatrix}
= \begin{bmatrix}
\boldsymbol{x}'_0 + t\boldsymbol{\lambda}' \\
(\boldsymbol{W}_c\boldsymbol{x}'_0 + \boldsymbol{b}_c) + t\boldsymbol{\lambda}'_c
\end{bmatrix} = \boldsymbol{x}_0{''} + t\boldsymbol{\lambda}'',
\end{equation}
which is still a line embedded in the $m$-dimensional space $\boldsymbol{X}_m$, since both $\boldsymbol{x}_0{''}$ and $\boldsymbol{\lambda}''$ are $m$-dimensional vectors.

To the input $l_n$ of this lemma, each of its elements can be represented as the parametric-equation form
\begin{equation}
\boldsymbol{x} = \boldsymbol{x}_0 + \sum_{i = 1}^{n}t_i\boldsymbol{\lambda}_i,
\end{equation}
and its output of the first layer is
\begin{equation}
\boldsymbol{x}^{(1)} = \begin{bmatrix}
\boldsymbol{x}_0' + \sum_{i = 1}^{n}t_i\boldsymbol{\lambda}_i' \\
(\boldsymbol{W}_c\boldsymbol{x}'_0 + \boldsymbol{b}_c) + \sum_{i = 1}^{n}t_i\boldsymbol{\lambda}_{c_i}'
\end{bmatrix} \\
= \boldsymbol{x}_0{''} + \sum_{i = 1}^{n}t_i\boldsymbol{\lambda}_i'',
\end{equation}
which is on an $n$-dimensional subspace embedded in $\boldsymbol{X}_m$.
\end{proof}

By lemma 7, the map by the first layer of network $n^{(1)}m^{(1)}$ is from a hyperplane to a hyperplane. However, a map between hyperplanes is not necessarily an affine transform, but may be a projective one. The following theorem will provide more information.
\begin{thm}
Given data set $D$ of the $n$-dimensional input space and network $n^{(1)}m^{(1)}$ with $m > n$, suppose that $D$ simultaneously activates the $m$ units of the first layer. Let $D'$ be the mapped data set of $D$ by the first layer. Then $D'$ is on an $n$-dimensional subspace embedded in $\boldsymbol{X}_m$, where $\boldsymbol{X}_m$ is the $m$-dimensional space of the first layer, and is equivalent to $D$ in the sense of affine transforms.
\end{thm}
\begin{proof}
By the assumption of this theorem, we have $D \subset l_n$, with $l_n$ as the part of the $n$-dimensional input space that simultaneously activates the $m$ units of the first layer. Thus, according to lemma 7, the mapped data set $D'$ is still on an $n$-dimensional hyperplane embedded in $\boldsymbol{X}_m$.

To prove the equivalence of the data structures of $D$ and $D'$, for example, we first assume that $n = 2$ and $m = 3$ of $n^{(1)}m^{(1)}$. To any $\boldsymbol{x} \in D$ of the two-dimensional input space, write equation 7.6 here as
\begin{equation}
\boldsymbol{x}^{(1)} =
\begin{bmatrix}
\boldsymbol{x}' \\
\boldsymbol{w}^T\boldsymbol{x}' + b
\end{bmatrix},
\end{equation}
where $\boldsymbol{x}^{(1)} \in D'$ is the mapped data point of $\boldsymbol{x}$, and $\boldsymbol{x}'$ is a subvector of $\boldsymbol{x}^{(1)}$ as well as an affine transform of $\boldsymbol{x}$. Since $n = 2$ and $m = 3$, write $\boldsymbol{x} = [x, y]^T$, $\boldsymbol{x}' = [x', y' ]^T$, $\boldsymbol{w} = [A, B ]^T$ and $b = C$. Then equation 7.10 can be expressed as
\begin{equation}
\boldsymbol{x}^{(1)} =
\begin{bmatrix}
x', y', Ax' + By' + C
\end{bmatrix}^T,
\end{equation}
which is on a plane of three-dimensional space.

In order to see the change of the data structure from $D$ to $D'$ in terms of $\boldsymbol{x}$ and $\boldsymbol{x}^{(1)}$, we introduce an intermediate vector
\begin{equation}
\boldsymbol{x}_0 = \begin{bmatrix} x', y', 0\end{bmatrix}^T,
\end{equation}
which has a more direct relationship with $\boldsymbol{x} = [x, y]^T$. From $\boldsymbol{x}_0$ of equation 7.12 to $\boldsymbol{x}^{(1)}$ of equation 7.11, it could be an affine transform as
\begin{equation}
\boldsymbol{x}^{(1)} = \boldsymbol{R}\boldsymbol{x}_0 + \boldsymbol{T},
\end{equation}
where
\begin{equation}
\boldsymbol{R} =
\begin{bmatrix}
1 & 0 & 0 \\
0 & 1 & 0 \\
A & B & \alpha
\end{bmatrix}
\end{equation}
and $\boldsymbol{T} = [0, 0, C]^T$, with $\alpha \ne 0$ such that $\boldsymbol{R}$ is nonsingular. Let $D_0'$ be the data set derived from the first two dimensions of $D'$, with its each element represented in terms of $\boldsymbol{x}_0$ of equation 7.12. Then by equation 7.13, $D_0'$ is an affine transform of $D'$.

Note that in equation 7.12, from $\boldsymbol{x}' = [x', y']^T$ to $\boldsymbol{x}_0$, only a new dimension whose coordinate value is zero is added, and the augmented dimension has no relationship with the subvector $\boldsymbol{x}' = [x', y' ]^T$. We just put $\boldsymbol{x}'$ into a higher-dimensional space. Because $\boldsymbol{x}'$ is an affine transform of $\boldsymbol{x}$ by equation 7.10, the data structure of $D_0'$ is equivalent to that of $D$. Combined with equation 7.13, the conclusion follows; that is, the data structure of $D'$ output by the first layer is equivalent to that of $D$ of the input space.

The general case is similar, which is mainly related to equation 7.13. The intermediate vector is constructed as
\begin{equation}
\boldsymbol{x}_0 = \begin{bmatrix} {\boldsymbol{x}'}^T, \boldsymbol{0}^T \end{bmatrix}^T,
\end{equation}
where $\boldsymbol{0}$ is a $(m - n) \times 1$ zero vector. The nonsingular matrix $\boldsymbol{R}$ is changed into
\begin{equation}
\boldsymbol{R} =
\begin{bmatrix}
\boldsymbol{I}_{n} & \boldsymbol{0}_{m-n}\\
\boldsymbol{W}_c & \boldsymbol{I}_{m-n}
\end{bmatrix},
\end{equation}
where $\boldsymbol{I}_{n}$ is the $n \times n$ identity matrix, and similarly for $\boldsymbol{I}_{m-n}$; $\boldsymbol{W}_c$ is the matrix introduced in equation 7.6; $\boldsymbol{0}_{m-n}$ is the $n \times (m-n)$ zero matrix. Note that $\boldsymbol{R}$ of equation 7.16 is nonsingular. The vector $\boldsymbol{T}$ is
\begin{equation}
\boldsymbol{T} = \begin{bmatrix} {\boldsymbol{0}'}^T, \boldsymbol{b}_c^T \end{bmatrix}^T,
\end{equation}
where $\boldsymbol{b}_c$ comes from equation 7.6 and $ \boldsymbol{0}'$ is an $n \times 1$ zero vector.

By equations 7.15, 7.16 and 7.17, $\boldsymbol{x}^{(1)}$ of equation 7.6 can be expressed as the form of equation 7.13. The remaining proof is similar to the above example.
\end{proof}

\begin{lem}
Let $\boldsymbol{x}_0$ be a point of $n$-dimensional space, and $\boldsymbol{w}^T\boldsymbol{x} + b = 0$ be a hyperplane $l$ associated with a ReLU. After an affine transform, $\boldsymbol{x}_0$ becomes $\boldsymbol{x}_0'$ and $l$ becomes ${\boldsymbol{w}{'}}^T\boldsymbol{x}' + b' = 0$. We then have $\boldsymbol{w}^T\boldsymbol{x}_0 + b = {\boldsymbol{w}{'}}^T\boldsymbol{x}'_0 + b'$, which implies that the output of a hyperplane with respect to a point is not affected by affine transforms.
\end{lem}
\begin{proof}
Denote an affine transform by $\boldsymbol{x}' = \boldsymbol{W}\boldsymbol{x} + \boldsymbol{b}$,
and then $\boldsymbol{x} = \boldsymbol{W}^{-1}(\boldsymbol{x}' - \boldsymbol{b})$. So the equation $\boldsymbol{w}^T\boldsymbol{x} + b = 0$ of $l$ can be expressed as
\begin{equation}
{\boldsymbol{w}'}^T\boldsymbol{x}' + b' = 0,
\end{equation}
where ${\boldsymbol{w}'}^T = \boldsymbol{w}^T\boldsymbol{W}^{-1}$ and $b' = -\boldsymbol{w}^T\boldsymbol{W}^{-1}\boldsymbol{b} + b$, which is the affine transform of hyperplane $l$, denoted by $l'$. The affine transform of point $\boldsymbol{x}_0$ is
\begin{equation}
\boldsymbol{x}'_0 = \boldsymbol{W}\boldsymbol{x}_0 + \boldsymbol{b}.
\end{equation}
Substituting equation 7.19 into equation 7.18, it's easy to verify $\boldsymbol{w}'^T\boldsymbol{x}'_0 + b' = \boldsymbol{w}^T\boldsymbol{x}_0 + b$, and thus $\sigma(\boldsymbol{w}'^T\boldsymbol{x}'_0 + b') = \sigma(\boldsymbol{w}^T\boldsymbol{x}_0 + b)$, where $\sigma(x)$ is the activation function of a ReLU.
\end{proof}

\subsection{Basic Principle}
\begin{prp}
Given network $n^{(1)}m^{(1)}1^{(1)}$ for $m > n$, let $l_n$ be an $n$-dimensional subspace of the $m$-dimensional space $\boldsymbol{X}_m$ of the first layer, and let $p$ be a data point of $l_n$. Then to any $m-1$-dimensional hyperplane $l_{m-1}$ formed by the unit of the second layer, if $l_{m-1} \cap l_n \ne \emptyset$ and $l_n \nsubseteq l_{m-1}$, the output of $l_{m-1}$ with respect to $p$ is equal to the output of an $n-1$-dimensional hyperplane $l_{n-1}$, with $l_{n-1} \subset l_n$ and $l_{n-1} = l_{m-1} \cap l_n$.
\end{prp}
\begin{proof}
When $n=2$ and $m = 3$ of $n^{(1)}m^{(1)}1^{(1)}$, this proposition can be interpreted intuitively as follows. To a point $p$ of a two-dimensional subspace $l_2 \subset \boldsymbol{X}_3$ of the first layer, the output of any plane $l_2'$ of $\boldsymbol{X}_3$ for $p$ is equal to the output of a line $l_1$, where $l_1 \subset l_2$ and $l_1 = l_2' \cap l_2$, provided that $l_2' \cap l_2 \ne \emptyset$ and $l_2 \nsubseteq l_2'$.

We prove the conclusion in a general form. Construct a new coordinate system of $\boldsymbol{X}_m$ of the first layer, according to the $n$-dimensional subspace $l_n$ or the $n$-dimensional hyperplane $l_n$ embedded in it, where the point $p$ is located in. On $l_n$, select an arbitrary point as the origin $O$ of the new coordinate system, and choose $n$ linearly independent vectors $\boldsymbol{v}_i$'s for $i = 1, 2, \cdots, n$ to be the $n$ bases. The remaining $m - n$ bases of the new coordinate system are other $m - n$ linearly independent vectors $\boldsymbol{v}_j$'s for $j = n+1, \dots, m$ that are not on $l_n$. To any point of $\boldsymbol{X}_m$, its coordinate vectors in the original and new coordinate systems are linked by an affine transform $\boldsymbol{A}$.

Write the equation of a $m-1$-dimensional hyperplane $l_{m-1}$ of $\boldsymbol{X}_m$ as
\begin{equation}
\boldsymbol{w}^T\boldsymbol{x}^{(1)} + b = 0,
\end{equation}
where ${\boldsymbol{x}}^{(1)}$ is the output vector of the first layer. In the new coordinate system, by the affine transform $\boldsymbol{A}$, $l_{m-1}$ becomes
\begin{equation}
{\boldsymbol{w}'}^T{\boldsymbol{x}'}^{(1)} + b' = 0,
\end{equation}
denoted by $l_{m-1}'$. By lemma 8, we have
\begin{equation}
\boldsymbol{w}^T\boldsymbol{x}_p^{(1)} + b = {\boldsymbol{w}'}^T{\boldsymbol{x}'}_p^{(1)} + b',
\end{equation}
where $\boldsymbol{x}_p^{(1)}$ and ${\boldsymbol{x}'}_p^{(1)}$ are the coordinate vectors of $p$ and $p'$ before and after the affine transform $\boldsymbol{A}$, respectively; that is, affine transforms do not change the relative position of a point with respect to a hyperplane.

Equation 7.21 can be decomposed into
\begin{equation}
{\boldsymbol{w}'}_n^T{\boldsymbol{x}'}_n^{(1)} + {\boldsymbol{w}'}_{n-m}^T{\boldsymbol{x}'}_{n-m}^{(1)} + b' = 0,
\end{equation}
where ${\boldsymbol{x}'}_n^{(1)}$ is the subvector of ${\boldsymbol{x}'}^{(1)}$ associated with the $n$ bases of the new coordinate system on hyperplane $l_n$, and subvector ${\boldsymbol{x}'}_{n-m}^{(1)}$ is composed of the remaining dimensions. Since $l_{m-1} \cap l_n \ne \emptyset$ and $l_n \nsubseteq l_{m-1}$, not all of entries of ${\boldsymbol{w}'}_n$ of equation 7.23 are zero. When
\begin{equation}
{\boldsymbol{x}'}_{n-m}^{(1)} = \boldsymbol{0}
\end{equation}
in equation 7.23, where $\boldsymbol{0}$ is a $(m-n) \times 1$ vector whose entries are all zero, equation
\begin{equation}
{\boldsymbol{w}'}_n^T{\boldsymbol{x}'}_n^{(1)} + b' = 0
\end{equation}
can be considered as an $n-1$-dimensional hyperplane of $n$-dimensional space, denoted by $l_{n-1}$. From the construction of the new coordinate system of $\boldsymbol{X}_m$, we know $l_{n-1} \subset l_n$. Because ${\boldsymbol{x}'}_{n-m}^{(1)} = \boldsymbol{0}$ of equation 7.24 could be regarded as equation of $l_n$, equation 7.25 of $l_{n-1}$ is the intersection of $l_n$ (equation 7.24) and $l_{m-1}$ (equation 7.23), that is,
\begin{equation}
l_{n-1} = l_{m-1} \cap l_n.
\end{equation}

Because $p$ is on hyperplane $l_n$ where the first $n$ bases are located in, the coordinate vector ${\boldsymbol{x}'}_p^{(1)}$ of $p'$ can be expressed as
\begin{equation}
{\boldsymbol{x}'}_p^{(1)} = \begin{bmatrix}
x_1, x_2, \cdots, x_{n}, 0, \cdots, 0
\end{bmatrix}^T =
\begin{bmatrix}
\boldsymbol{x}_{p_n}^T, \boldsymbol{0}^T
\end{bmatrix}^T,
\end{equation}
where $\boldsymbol{x}_{p_n}$ can be considered as the coordinate vector of $p'$ in $n$-dimensional space of $l_n$. Substituting equation 7.27 into equation 7.23 or 7.21, we obtain the output of $l_{m-1}'$ with respect to $p'$
\begin{equation}
y = \sigma({\boldsymbol{w}'}^T{\boldsymbol{x}'}_p^{(1)} + b') = \sigma({\boldsymbol{w}'}_n^T\boldsymbol{x}_{p_n} + b'),
\end{equation}
where $\sigma(x)$ is the activation function of a ReLU. Equations 7.28 and 7.25 indicate that the output $y$ of $m-1$-dimensional hyperplane $l_{m-1}'$ with respect to $p'$ equals the output of $n-1$-dimensional hyperplane $l_{n-1}$.

Equations 7.28 and 7.22 yield
\begin{equation}
y = \sigma(\boldsymbol{w}^T\boldsymbol{x}^{(1)}_p + b) = \sigma({\boldsymbol{w}'}_n^T\boldsymbol{x}_{p_n} + b').
\end{equation}
Combined with equation 7.26, the conclusion follows.
\end{proof}

\begin{cl}
In $m$-dimensional space $\boldsymbol{X}_m$, any linear classification via a $m-1$-dimensional hyperplane $l_{m-1}$ on data set $D$ of an $n$-dimensional subspace $\boldsymbol{X}_n \subset \boldsymbol{X}_m$ for $n < m$ could be done by an $n-1$-dimensional hyperplane $l_{n-1} \subset \boldsymbol{X}_n$, and vice versa, with $l_{m-1}$ and $l_{n-1}$ having the same output with respect to any point of $\boldsymbol{X}_n$.
\end{cl}
\begin{proof}
To classify $D \subset \boldsymbol{X}_n$ via $l_{m-1}$, we need $l_{m-1} \cap \boldsymbol{X}_n \ne \emptyset$ and $\boldsymbol{X}_n \nsubseteq l_{m-1}$, satisfying the condition of proposition 5. The classification of $D$ via $l_{m-1}$ should compute the outputs of $l_{m-1}$ with respect to the elements of $D$, which are equal to those of $l_{n-1}$ by proposition 5; and thus the classifications via $l_{m-1}$ and $l_{n-1}$ are equivalent.

The converse conclusion obviously holds, which means that in subspace $\boldsymbol{X}_n$, the classification of $D$ via any $n-1$-dimensional hyperplane $l_{n-1}$ could be done by a $m-1$-dimensional hyperplane $l_{m-1}$ whose intersection with $\boldsymbol{X}_n$ is $l_{n-1}$, which is easily constructed. The outputs of $l_{n-1}$ and $l_{m-1}$ with respect to any point of $\boldsymbol{X}_n$ are equal due to proposition 5.
\end{proof}

\subsection{Construction Method}
Proposition 6 of this section is a basic operation of constructing various solutions for piecewise linear functions, and is also an existence proof of affine transforms for the data embedded in a higher-dimensional space.
\begin{lem}
To network $n^{(1)}m^{(1)}n^{(1)}$ for $m >n$, let $l^{(1)}_n$ be the mapped region of $\prod_{j=1}^{m}l_{1j}^+$ of the $n$-dimensional input space by the first layer, where hyperplanes $l_{1j}$'s correspond to the units of the first layer. Denote by $l_{2i}$'s for $i = 1, 2, \cdots, n$ the $m-1$-dimensional hyperplanes formed by the second layer. Suppose that data set $D$ of the input space simultaneously activates all the units of each layer, and that $l_i = l_{2i} \cap l_n^{(1)} \ne \emptyset$. If $\bigcap_{i=1}^{n}l_i= O$ is a single point, the mapped data set $D''$ by the second layer is an affine transform of $D$.
\end{lem}
\begin{proof}
According to lemma 5 of \citet{Huang2020}, in $n$-dimensional input space or on an $n$-dimensional hyperplane, if there exist $n$ hyperplanes $l_{i}'$'s for $i = 1, 2, \cdots, n$ having only one common point, then the set of the $n$-tuples of the nonzero outputs of $n$ units corresponding to $l_i'$'s is an affine transform of region $\prod_{i=1}^{n}l_i'^+$ of the input space.

By lemma 7, $l^{(1)}_n$ is part of an $n$-dimensional hyperplane or an $n$-dimensional subspace embedded in the $m$-dimensional space of the first layer. Because data set $D$ simultaneously activates the $m$ units of the first layer, by theorem 6, the mapped data set $D' \subset l^{(1)}_n$, and is an affine transform of $D$. And by proposition 5, the output of a $m-1$-dimensional hyperplane $l'$ with respect to $\boldsymbol{x}^{(1)} \in l^{(1)}_n$ is equal to the output of an $n-1$-dimensional hyperplane $l \subset l^{(1)}_n$ with $l = l' \cap l^{(1)}_n$.

Therefore, if $l_i = l_{2i} \cap l_n^{(1)} \ne \emptyset$ for $i = 1, 2, \cdots, n$, $D'$ also simultaneously activates $l_i$'s, i.e., $D' \subset \prod_{i=1}^{n}l_i^+$. And because the output of $l_{2i}$ with respect to each element of $D'$ is equal to that of $l_i$ for all $i$, the mapped data set $D''$ from $D'$ by $l_{2i}$'s is the same as the mapped data set $D''_n$ by $l_i$'s. If $\bigcap_{i=1}^{n}l_i$ is a single point $O$, since $D' \subset \prod_{i=1}^{n}l_i^+$, $D''_n$ is an affine transform of $D'$; so $D''$ is an affine transform of $D'$, which follows the conclusion.
\end{proof}

\begin{lem}
To network $n^{(1)}n^{(1)}$, hyperplane $l_{1i}$ for $i = 1, 2, \cdots, n$ corresponds to unit $u_{1i}$ of the first layer. Given data set $D$ of the $n$-dimensional input space and an arbitrary point $O \in l_{11}$, if $D \subset l_{11}^+$, then we can find other $n-1$ hyperplanes, such that all of $l_{1i}$'s for $i = 1, 2, \cdots, n$ pass through the unique common point $O$ and have the same classification effect as $l_{11}$.
\end{lem}
\begin{proof}
The proof gives a construction method. Denote hyperplane $l_{11}$ by $\boldsymbol{w}_1^T\boldsymbol{x} + b_1 = 0$ where
\begin{equation}
\boldsymbol{w}_1 = \begin{bmatrix}
w_{11}, w_{12}, \cdots, w_{1n}
\end{bmatrix}^T
\end{equation}
with $w_{11} \ne 0$. Write the $n$ equations of $l_{11}$ and other $n-1$ constructed heperplanes $l_{1\nu}$'s for $\nu = 2, 3, \cdots, n$ in matrix form
\begin{equation}
\boldsymbol{W}\boldsymbol{x} + \boldsymbol{b} = \boldsymbol{0},
\end{equation}
where
\begin{equation}
\boldsymbol{W} = \begin{bmatrix}
w_{11} & w_{12} & w_{13} & \cdots & w_{1n} \\
w_{11} & w_{12} + \varepsilon_1  & w_{13} + \varepsilon_2 & \cdots & w_{1n} + \varepsilon_{n-1} \\
w_{11} & w_{12} + \varepsilon_1^2  & w_{13} + \varepsilon_2^2 & \cdots & w_{1n} + \varepsilon_{n-1}^2 \\
\vdots & \vdots & \vdots &\ddots &  \vdots\\
w_{11} & w_{12} + \varepsilon_1^{n-1}  & w_{13} + \varepsilon_2^{n-1} & \cdots & w_{1n} + \varepsilon_{n-1}^{n-1}
\end{bmatrix}
\end{equation}
is designed to be nonsingular with $0 < \varepsilon_{i} < 1$ and $\varepsilon_{i} \ne \varepsilon_{j}$ for $i \ne j$, where $i, j = 1, 2, \cdots, n$ (theorem 4 of \citet{Huang2020}). Let $\boldsymbol{x}_0$ be the coordinate vector of the arbitrarily designated point $O$ of $l_{11}$. We now choose parameters $\boldsymbol{W}$ and $\boldsymbol{b}$ such that $l_{1i}$'s for $i \ne 1$ pass through $\boldsymbol{x}_0$ and have the same classification effect as $l_{11}$. No matter what $\varepsilon_i$'s of equation 7.32 are, just let $\boldsymbol{b} = -\boldsymbol{W}\boldsymbol{x}_0$, and then all hyperplanes $l_{1i}$'s pass through point $\boldsymbol{x}_0$. The sufficiently small $\varepsilon_j$'s can make all of $l_{1i}$'s for $i \ne 1$ classify $D$ as $l_{11}$ (theorem 4 of \citet{Huang2020}). This completes the proof.
\end{proof}

\begin{prp}
Under the notation of lemma 9, suppose that $D$ simultaneously activates the $m$ units of the first layer of $n^{(1)}m^{(1)}n^{(1)}$ for $m > n$. Then we can set the parameters of the second layer, such that $D''$ output by the network is an affine transform of $D$ of the input space.
\end{prp}
\begin{proof}
By lemma 7, the output $l_n^{(1)}$ of the first layer of $n^{(1)}m^{(1)}n^{(1)}$ with respect to region $\prod_{i=1}^{m}l_{1i}^+$ of the input space lies on an $n$-dimensional subspace $\boldsymbol{X}_n$ of the $m$-dimensional space $\boldsymbol{X}_m$ of the first layer, and each element of $l_n^{(1)}$ can be written as
\begin{equation}
\boldsymbol{x}^{(1)} = \boldsymbol{x}_0 + \sum_{i = 1}^{n}t_i\boldsymbol{\lambda}_i,
\end{equation}
in which all the vectors are $m$-dimensional. In equation 7.33, point $\boldsymbol{x}_0$ and $n$ linearly independent vectors $\boldsymbol{\lambda}_i$'s make up a coordinate system $\{O; \boldsymbol{\lambda}_i\text{'s}\}$ of $n$-dimensional space $\boldsymbol{X}_n$, with origin $O = \boldsymbol{x}_0$. Equation 7.33 of $l_n^{(1)}$ can also be regarded as the parametric equation of an $n$-dimensional hyperplane of $\boldsymbol{X}_m$. Under coordinate system $\{O; \boldsymbol{\lambda}_i\text{'s}\}$ of $n$-dimensional $\boldsymbol{X}_n$, denote by $\boldsymbol{w}^T\boldsymbol{t} + b =0$ an $n-1$-dimensional hyperplane, where $\boldsymbol{t} = [t_1, t_2, \cdots, t_{n}]^T$ with $t_i$ defined in equation 7.33. Since $D$ simultaneously activates all the $m$ units of the first layer, by theorem 6, the mapped data set $D' \subset \boldsymbol{X}_n$, whose each element can be represented by a vector derived from coordinate system $\{O; \boldsymbol{\lambda}_i\text{'s}\}$.

The remaining proof is by construction according to lemma 9. Under coordinate system $\{O; \boldsymbol{\lambda}_i\text{'s}\}$, we first construct $n$ hyperplanes $l_i$'s of $\boldsymbol{X}_n$, such that $D' \subset \prod_{i=1}^{n}l_i^+$ and $\bigcap_{i=1}^{n}l_i$ is a single point $O'$. Select a hyperplane $l_1$ with $D' \subset l_1^+$ and designate any point of $l_1$ as point $O'$. Use lemma 10 to find other $n-1$ hyperplanes that pass through the unique common point $O'$ and classify $D'$ as $l_1$. Let
\begin{equation}
\boldsymbol{w}_{i}^T\boldsymbol{t} + b_i = \sum_{j=1}^{n}w_{ij}t_j + b_i =0
\end{equation}
be the equation of hyperplane $l_i$ for $i = 1, 2, \cdots, n$ that has been constructed. In equation 7.34, the coefficients $w_{ij}$'s for $j$ cannot be all zero; select one of them, say, $w_{i\nu} \ne 0$ where $1 \le \nu \le n $. Then the corresponding $t_{\nu}$ with respect to $w_{i\nu}$ can be expressed as the linear combination of other $t_j$'s for $j \ne \nu$, which is
\begin{equation}
t_{\nu} = -\sum_{j \ne \nu}t_jw_{ij}/w_{i\nu} - b_i/w_{i\nu}.
\end{equation}
Substituting equation 7.35 into equation 7.33, we get the parametric-equation form of equation 7.34 of $l_i$ in terms of embedded vectors of $\boldsymbol{X}_m$, that is,
\begin{equation}
l_i := \boldsymbol{x}_i + \sum_{j = 1}^{n-1}t_{ij}\boldsymbol{\lambda}_{ij},
\end{equation}
where $\boldsymbol{x}_i$, $t_{ij}$ and $\boldsymbol{\lambda}_{ij}$ can be easily obtained from equation 7.33 after the substitution.

Next we construct $m-1$-dimensional hyperplanes $l_{2i}$'s for $i = 1, 2, \cdots, n$ of $\boldsymbol{X}_m$ formed by the $n$ units of the second layer, subject to
\begin{equation}
l_i = l_{2i} \cap l_n^{(1)},
\end{equation}
where $l_n^{(1)}$ is the $n$-dimensional hyperplane of equation 7.33 as discussed above. Denote the equation of $l_{2i}$ by
\begin{equation}
\boldsymbol{w}_{2i}^T\boldsymbol{x}^{(1)} + b_{2i} = 0,
\end{equation}
where $\boldsymbol{x}^{(1)}$ is the output vector of the first layer. We should find the solutions for $\boldsymbol{w}_{2i}$ and $b_{2i}$ to make equation 7.37 hold. Due to $l_i \subset l_{2i}$ by equation 7.37, substituting equation 7.36 into $\boldsymbol{x}^{(1)}$ of equation 7.38, we have
\begin{equation}
\boldsymbol{w}_{2i}^T\boldsymbol{x}_i + b_{2i} +  \sum_{j = 1}^{n-1}(\boldsymbol{w}_{2i}^T\boldsymbol{\lambda}_{ij})t_{ij} = 0.
\end{equation}
Since equation 7.39 holds for all real numbers of $t_{ij}$'s for all $j$, we must have
\begin{equation}
\begin{cases}
\boldsymbol{w}_{2i}^T\boldsymbol{x}_i + b_{2i} = 0
\\
\boldsymbol{w}_{2i}^T\boldsymbol{\lambda}_{ij} = 0 \ \text{for all} \ j
\end{cases},
\end{equation}
where $j = 1, 2, \cdots, n-1$. Because $\boldsymbol{w}_{2i}$ is of size $m \times 1$, the linear system of equation 7.40 has $m+1$ unknowns and $n$ equations with $m+1 > n$. Thus, equation 7.40 has infinite number of solutions for $\boldsymbol{w}_{2i}$ and $b_{2i}$.

Use the constructed $l_{2i}$'s for $i= 1, 2, \cdots, n$ to set the input weights and biases of the second layer. Then the mapped $D''$ of $D$ by the second layer is an affine transform of $D$.

\end{proof}

The proposition below is a necessary condition that the output of $n^{(1)}m^{(1)}n^{(1)}$ for $m > n$ is an affine transform of data of the input space, from which some information about the probability of this kind of affine transform could be obtained (remark 1).
\begin{prp}
To network $n^{(1)}m^{(1)}n^{(1)}$ for $m >n$, suppose that data set $D$ of the $n$-dimensional input space simultaneously activates the $m$ units of the first layer. Let $D'$ and $D''$ be the mapped data sets of $D$ by the first layer and the second layer, respectively. Then if $D''$ is an affine transform of $D$, the rank of size $n \times m$ weight matrix $\boldsymbol{W}_{2}$ of the second layer must not be less than $n$.
\end{prp}
\begin{proof}
Let $\boldsymbol{x} \in D$ be any point of $D$; and let $\boldsymbol{W}_1$ be the $m \times n$ weight matrix of the first layer, whose rank is $n$ by the assumption of section 2.4. The bias vectors of the first layer and the second layer are denoted by $\boldsymbol{b}_1$ and $\boldsymbol{b}_2$, respectively. The nonzero output vector of the second layer with respect to $\boldsymbol{x}$ of the input space is
\begin{equation}
\begin{aligned}
\boldsymbol{x}^{(2)} &= \boldsymbol{W}_{2}(\boldsymbol{W}_{1}\boldsymbol{x} + \boldsymbol{b}_1) + \boldsymbol{b}_2 \\
&= \boldsymbol{W}\boldsymbol{x} + \boldsymbol{b}
\end{aligned}
\end{equation}
where $\boldsymbol{W} = \boldsymbol{W}_{2}\boldsymbol{W}_{1}$ with size $n \times n$ and $\boldsymbol{b} = \boldsymbol{W}_{2}\boldsymbol{b}_1 + \boldsymbol{b}_2$. Then if $\boldsymbol{W}$ is nonsingular, $\boldsymbol{x}^{(2)}$ would be an affine transform of $\boldsymbol{x}$.

Since $\text{rank}(\boldsymbol{W}_1) = n$, we can find $n$ row vectors of $\boldsymbol{W}_1$ to form an $n \times n$ nonsingular submatrix. Write $\boldsymbol{W}_1$ in the block matrix form as
\begin{equation}
\boldsymbol{W}_1 =
\begin{bmatrix}
\boldsymbol{A}'\\
\boldsymbol{B}'
\end{bmatrix},
\end{equation}
where $\boldsymbol{A}'$ is of size $(m-n) \times n$ and $\boldsymbol{B}'$ is an $n \times n$ nonsingular submatrix. The multiplication $\boldsymbol{W} = \boldsymbol{W}_{2}\boldsymbol{W}_{1}$ of equation 7.41 could be changed into block matrix form, after rearranging the columns of $\boldsymbol{W}_{2}$ in accordance with the row-vector selection for $\boldsymbol{B}'$. So we have
\begin{equation}
\boldsymbol{W} =
\begin{bmatrix}
\boldsymbol{A}, \boldsymbol{B}
\end{bmatrix}
\begin{bmatrix}
\boldsymbol{A}'\\
\boldsymbol{B}'
\end{bmatrix} = \boldsymbol{A}\boldsymbol{A}' + \boldsymbol{B}\boldsymbol{B}',
\end{equation}
where $\boldsymbol{A}$ and $\boldsymbol{B}$ comprise the columns of $\boldsymbol{W}_2$, with $\boldsymbol{A}$ of size $n \times (m-n)$ and $\boldsymbol{B}$ of size $n \times n$.
Equation 7.43 is equivalent to
\begin{equation}
\boldsymbol{W} = \boldsymbol{C}\boldsymbol{B}',
\end{equation}
where
\begin{equation}
\boldsymbol{C} = \boldsymbol{A}\boldsymbol{A}'\boldsymbol{B}'^{-1} + \boldsymbol{B}
\end{equation}
is an $n \times n$ matrix.

In equation 7.44, since $\boldsymbol{B}'$ is nonsingular, the nonsingular $\boldsymbol{C}$ implies nonsingular $\boldsymbol{W}$, and vice versa. Thus, we can deal with $\boldsymbol{C}$ of equation 7.45 instead of $\boldsymbol{W}$ of equation 7.41. The term $\boldsymbol{A}'\boldsymbol{B}'^{-1}$ of equation 7.45 is a $(m - n) \times n$ matrix that can be written as
\begin{equation}
\boldsymbol{A}'\boldsymbol{B}'^{-1} = \begin{bmatrix}
\boldsymbol{t}_1 \\
\boldsymbol{t}_2 \\
\vdots \\
\boldsymbol{t}_{m - n}
\end{bmatrix},
\end{equation}
where $\boldsymbol{t}_i$ for $i = 1, 2, \cdots, m-n$ is an $1 \times n$ row vector. We further decompose matrices $\boldsymbol{A}$ and $\boldsymbol{B}$ of equation 7.45 into block forms as
\begin{equation}
\boldsymbol{A} = \begin{bmatrix}
\boldsymbol{a}_1, \boldsymbol{a}_2, \cdots, \boldsymbol{a}_{m - n}
\end{bmatrix}
\end{equation}
and
\begin{equation}
\boldsymbol{B} = \begin{bmatrix}
\boldsymbol{b}_1, \boldsymbol{b}_2,\cdots, \boldsymbol{b}_{n}
\end{bmatrix},
\end{equation}
where $\boldsymbol{a}_i$ for $i = 1, 2, \cdots, m-n$ and $\boldsymbol{b}_j$ for $j = 1, 2, \cdots, n$ are both $n \times 1$ column vectors.

Substituting equations 7.46, 7.47 and 7.48 into equation 7.45 gives
\begin{equation}
\boldsymbol{C} = \sum_{i=1}^{m-n}\boldsymbol{a}_i\boldsymbol{t}_i +
\begin{bmatrix}
\boldsymbol{b}_1, \boldsymbol{b}_2, \cdots, \boldsymbol{b}_{n}
\end{bmatrix}.
\end{equation}
Write $\boldsymbol{t}_i = [t_{i1}, \  t_{i2}, \  \cdots, \  t_{in}]$ and equation 7.49 becomes
\begin{equation}
\begin{aligned}
\boldsymbol{C} &= \sum_{i=1}^{m-n}\boldsymbol{a}_i
\begin{bmatrix}
t_{i1}, t_{i2}, \cdots,t_{in}
\end{bmatrix}
 +
\begin{bmatrix}
\boldsymbol{b}_1, \boldsymbol{b}_2, \cdots, \boldsymbol{b}_{n}
\end{bmatrix} \\
&=
\begin{bmatrix}
\boldsymbol{b}_1 + \sum_{i=1}^{k}t_{i1}\boldsymbol{a}_i, \cdots, \boldsymbol{b}_{n} + \sum_{i=1}^{k}t_{in}\boldsymbol{a}_i
\end{bmatrix},
\end{aligned}
\end{equation}
where $k = m-n$.

From equations 7.43, 7.47 and 7.48, we know $\boldsymbol{a}_i$'s for $i = 1, 2, \cdots, m-n$ and $\boldsymbol{b}_j$'s for $j = 1, 2, \cdots, n$ comprise all the columns of $\boldsymbol{W}_2$ of equation 7.41. If we write $\boldsymbol{W}_2$ as
\begin{equation}
\boldsymbol{W}_2 = \begin{bmatrix}
\boldsymbol{w}_1, \boldsymbol{w}_2,\cdots, \boldsymbol{w}_{m}
\end{bmatrix},
\end{equation}
where $\boldsymbol{w}_\nu$ for $\nu = 1, 2, \cdots, m$ is a column vector of size $n \times 1$, equation 7.50 can be expressed as
\begin{equation}
\boldsymbol{C} =
\begin{bmatrix}
\boldsymbol{w}_{j_1} + \sum_{i=1}^{k}t_{i1}\boldsymbol{w}_{j_{n+i}}, \cdots, \boldsymbol{w}_{j_n} + \sum_{i=1}^{k}t_{in}\boldsymbol{w}_{j_{n+i}}
\end{bmatrix},
\end{equation}
where $1\le j_{\nu} \le m$ for $\nu = 1, 2, \cdots, m$, and $k = m - n$.

By the property of matrix determinants, we have
\begin{equation}
\det(\boldsymbol{C}) = \sum_{i = 1}^{N}\alpha_{i}\det(D_i)
\end{equation}
where $N = \binom{m}{n}$, $\alpha_{i}$'s are constant values not determined by $\boldsymbol{W}_2$, and
\begin{equation}
D_i =
\begin{bmatrix}
\boldsymbol{w}_{i_1}, \boldsymbol{w}_{i_2}, \cdots, \boldsymbol{w}_{i_n}
\end{bmatrix}
\end{equation}
is derived from the $n$-combination of the $m$ columns of equation 7.51.

If $\text{rank}(\boldsymbol{W}_2) < n$, in equation 7.53, $\det(D_i) = 0$, implying $\det(\boldsymbol{C}) = 0$. Then by equation 7.44, $\det(\boldsymbol{W}) = 0$. Combined with equation 7.41, the conclusion follows.
\end{proof}

\begin{rmk-3}
Under the probabilistic model of theorem 12 of the appendix, the probability of $\rm{rank}(\boldsymbol{W}) < n$ is $0$.
\end{rmk-3}

\begin{rmk-3}
Note that the necessary condition of this proposition is only related to the input weights of the second layer of $n^{(1)}m^{(1)}n^{(1)}$, which is different from lemma 9 in that the sufficient condition includes the parameters of the first layer.
\end{rmk-3}

\subsection{General Conclusions}
We summarize the preceding results by two theorems to highlight the main thought of affine-transform generalizations.
\begin{thm}
For network $n^{(1)}m^{(1)}$ with $m > n$, assume that data set $D$ of the $n$-dimensional input space simultaneously activates the $m$ units of the first layer. Then in terms of transmitting $D$ through affine transforms to the second layer to be added, the effect of $n^{(1)}m^{(1)}$ could be equivalent to that of $n^{(1)}n^{(1)}$.
\end{thm}
\begin{proof}
Proposition 6 has constructed the solution, so the conclusion holds.
\end{proof}

\begin{thm}
Given network $n^{(1)}\prod_{i=1}^{d}m_i^{(1)}$ for $m_i > n$ and data set $D$ of the $n$-dimensional input space, if $D$ simultaneously activates all the units of each layer, the effect of $n^{(1)}\prod_{i=1}^{d}m_i^{(1)}$ could be equivalent to that of $n^{(1)}n^{(d)}$ in terms of transmitting $D$ to the $d+1$th layer via affine transforms.
\end{thm}
\begin{proof}
Let $D^{(i)}$ be the mapped data set of $D$ by the $i$th layer for $i = 1, 2, \cdots, d$. In the second layer, we construct the $n$ of the $m_2$ units by the method of proposition 6, such that each element of $D^{(2)}$ can be represented as equation 7.6; and by lemma 7 and theorem 6, $D^{(2)}$ lies on an $n$-dimensional subspace of the $m_2$-dimensional space of the second layer and is an affine transform of $D$. This is the case of $i = 2$. We then take the second layer as the input layer, the same procedure could be done in the third layer. Repeat it inductively until $i = d$. Finally, $D$ could be transmitted to the last layer in the form of equation 7.6, and we can use proposition 6 to output an affine transform of $D$ in the $d+1$th layer.
\end{proof}

\section{Linear-Output Generalization}
In lemmas 1 and 2 of three-layer networks, when the input data set $D$ to the hidden layer comes from the $n$-dimensional input space, we can realize a linear function on $D$. If the input $D$ is embedded in a higher-dimensional space in the form of equation 7.6, is the network still capable of producing an arbitrary linear function on $D$? This is an inevitable problem to be faced under the affine-transform generalization of section 7.

As the summary of sections 7 and 8, we'll provide some applications of those theories in section 8.2. The mechanism of overparameterization solutions associated with affine transforms will be discussed in proposition 8.

\subsection{Main Results}
\begin{thm}
Given network $n^{(1)}m_1^{(1)}m_2^{(1)}1'^{(1)}$ for $m_1, m_2 > n$, suppose that data set $D$ of the $n$-dimensional input space simultaneously activates all the units of each hidden layer. Compared to the architecture $n^{(1)}m_2^{(1)}1'^{(1)}$ for $m_2 > n$ of lemma 1 or 2, the added first layer of $n^{(1)}m_1^{(1)}m_2^{(1)}1'^{(1)}$ could not influence the realization of a linear function on $D$.
\end{thm}
\begin{proof}
The proof is constructive. Since the weight matrix $\boldsymbol{W}_1$ of the first layer is of size $m_1 \times n$ with $m_1 > n$, by the assumption of section 2.4, we have $\text{rank}(\boldsymbol{W}_1) = n$, implying that there exists an $n \times n$ nonsingular submatrix of $\boldsymbol{W}_1$. Thus, as equation 7.6 of section 7, the nonzero output $\boldsymbol{x}^{(1)}$ of the first layer can be written as
\begin{equation}
\boldsymbol{x}^{(1)} =
\begin{bmatrix} \boldsymbol{x}' \\
\boldsymbol{x}_c' = \boldsymbol{W}_c\boldsymbol{x}' + \boldsymbol{b}_c
\end{bmatrix},
\end{equation}
where $\boldsymbol{x}'$ is an affine transform of $\boldsymbol{x}$ of the input space. Without loss of generality, we assume that the order of the units of the first layer are arranged in accordance with equation 8.1, which means that the outputs of $u_{1i}$'s for $i = 1, 2, \cdots, n$ comprise the entries of $\boldsymbol{x}'$ and the outputs of $u_{1j}$'s for $j = n+1, n+2, \cdots, m_1$ are the entries of $\boldsymbol{x}_c'$. We then fix the output weights of $u_{1j}$'s to be constant values as
\begin{equation}
w^{(2)}_{jk} = C_{jk},
\end{equation}
where subscripts $j$ and $k$ represent the $j$th unit of the first layer and the $k$th unit of the second layer for $k = 1, 2, \cdots, m_2$, respectively, and $C_{jk}$ is a constant for all $j$ and $k$.

The objective is to realize the function of the first layer of $n^{(1)}m_2^{(1)}1'^{(1)}$ by the second layer of $n^{(1)}m_1^{(1)}m_2^{(1)}1'^{(1)}$. Let $D'_n$ be the mapped data set of $D$ by the affine transform of $\boldsymbol{x}'$ of equation 8.1, and $\boldsymbol{X}_n'$ the $n$-dimensional space that $\boldsymbol{x}'$ belongs to. The first step is to use unit $u_{21}$ of the second layer to construct a hyperplane of $\boldsymbol{X}_n'$ activated by $D'_n$, under the constraint of equation 8.2. In space $\boldsymbol{X}_n'$, we select a hyperplane $l_1'$ whose equation is
\begin{equation}
{\boldsymbol{w}'}_1^T\boldsymbol{x}' + b_1' = 0,
\end{equation}
such that $D'_n \subset l_{1}'^+$.

Next, realize hyperplane $l_{1}'$ of equation 8.3 by unit $u_{21}$ of the second hidden layer. Let
\begin{equation}
\boldsymbol{w}_{21}^T\boldsymbol{x}^{(1)} + b_{21} = 0
\end{equation}
be the equation of $m_1 - 1$-dimensional hyperplane $l_{21}$ corresponding to $u_{21}$. Substituting equation 8.1 into equation 8.4 yields
\begin{equation}
{\boldsymbol{w}'}_{21}^T\boldsymbol{x}' + b_{21}' := (\boldsymbol{w}_{1_n} + \boldsymbol{W}_c^T\boldsymbol{w}_{1_c})^T\boldsymbol{x}' + \boldsymbol{w}_{1_c}^T\boldsymbol{b}_c + b_{21} = 0,
\end{equation}
with $\boldsymbol{w}_{21}' := \boldsymbol{w}_{1_n} + \boldsymbol{W}_c^T\boldsymbol{w}_{1_c}$, $b_{21}' := \boldsymbol{w}_{1_c}^T\boldsymbol{b}_c + b_{21}$ and
\begin{equation}
\begin{bmatrix}
\boldsymbol{w}_{1_n}^T, \boldsymbol{w}_{1_c}^T
\end{bmatrix}^T = \boldsymbol{w}_{21}
\end{equation}
of equation 8.4, which means that $\boldsymbol{w}_{1_n}$ is a subvector of $\boldsymbol{w}_{21}$ with respect to units $u_{1i}$'s that output $\boldsymbol{x}'$ of equation 8.1, and that $\boldsymbol{w}_{1_c}$ is composed of the fixed output weights of units $u_{1j}$'s as mentioned in equation 8.2. Denote by $l_{21}'$ the hyperplane $\boldsymbol{w}_{21}'^T\boldsymbol{x}' + b_{21}' = 0$ of equation 8.5.

We should realize equation 8.3 by equation 8.5, so let
\begin{equation}
{\boldsymbol{w}'}^T_1\boldsymbol{x}' + b'_1 = {\boldsymbol{w}'}_{21}^T\boldsymbol{x}' + b_{21}',
\end{equation}
where ${\boldsymbol{w}'}_1^T\boldsymbol{x}' + b'$ has been prescribed in equation 8.3. In combination with equations 8.5 and 8.6, by comparing the coefficients of each entry of $\boldsymbol{x}'$ of the both sides of equation 8.7, we can obtain the solution of $\boldsymbol{w}_{1_n}$; and the solution of $b_{21}$ is obtained by the equality of the biases of the two sides of equation 8.7.

Now the output of unit $u_{21}$ of the second layer with respect to $\boldsymbol{x}^{(1)}$ is equal to the output of a unit associated with $n-1$-dimensional hyperplane $l_{21}'$ of equation 8.5 with respect to $\boldsymbol{x}'$, and the latter could be the same as the case of unit $u_{11}$ of the first layer of $n^{(1)}m_2^{(1)}1'^{(1)}$ according to propositions 5 and 6. This completes the construction of $l_{21}' \subset \boldsymbol{X}_n'$ through unit $u_{21}$ of $n^{(1)}m_1^{(1)}m_2^{(1)}1'^{(1)}$, with $D'_n \subset l_{21}'^+$.

The ultimate goal is to make all of the $m_2$ outputs of the second layer in the form of $n-1$-dimensional hyperplanes like equation 8.5, with some constraints satisfied. Let
\begin{equation}
{\boldsymbol{w}'}_{2i}^T\boldsymbol{x}' + b_{2i}' = 0
\end{equation}
for $i = 2, 3, \cdots, m_2$ be the equations of the remaining $n-1$-dimensional hyperplanes of $\boldsymbol{X}_n'$ to be constructed, each denoted by $l_{2i}'$. If $D'_n \subset l_{2i}'^+$ for all $i$ and the rank of the linear-output matrix
\begin{equation}
\boldsymbol{\mathcal{W}} = \begin{bmatrix}
\boldsymbol{w}_{21}' & \boldsymbol{w}_{22}' & \cdots & \boldsymbol{w}_{2m_2}' \\
b_{21}' & b_{22}' & \cdots & b_{2m_2}'
\end{bmatrix}
\end{equation}
is $n+1$, since $l'_{21}$ and $l'_{2i}$'s are all produced by the second layer of $n^{(1)}m_1^{(1)}m_2^{(1)}1^{(1)}$, any linear function on $D'$ could be realized in the output layer. This is the key to the proof.

After $l_{21}'$ having been constructed previously, we get the parameters of $l'_{2i}$'s by the method of equation 4.9 as
\begin{equation}
\boldsymbol{\mathcal{W}} = \begin{bmatrix}
\boldsymbol{w}_{21}' & \boldsymbol{w}_{21}' + \boldsymbol{\xi}_1 & \cdots & \boldsymbol{w}_{21}' + \boldsymbol{\xi}_{m_2-1} \\
b_{21}' & b_{21}' + \varepsilon_1^n & \cdots & b_{21}' + \varepsilon_{m_2-1}^n
\end{bmatrix},
\end{equation}
where $\boldsymbol{\xi}_{\nu} = [0, \varepsilon_{\nu}, \varepsilon_{\nu}^2, \cdots, \varepsilon_{\nu}^{n-1}]^T$ for $\nu = 1, 2, \cdots, m_2-1$, and $0 < \varepsilon_{\nu} < 1$ as well as $\varepsilon_{\nu} \ne \varepsilon_{\mu}$ if $\mu \ne \nu$ for $\mu = 1, 2, \cdots, m_2-1$. Equations 8.8, 8.9 and 8.10 give
\begin{equation}
\boldsymbol{w}_{2i}' = \boldsymbol{w}_{21}' + \boldsymbol{\xi}_{i-1}, b_{2i}' = b_{21}' + \varepsilon_{i-1}^n
\end{equation}
for $i = 2, 3, \cdots, m_2$. By the proof of proposition 1, the matrix $\boldsymbol{\mathcal{W}}$ of equation 8.10 has rank $n+1$. And if $\varepsilon_{\nu}$'s are small enough, we have $D'_n \subset l_{2i}'^+$ for all $i$.

We now produce the constructed hyperplanes $l_{2i}'$'s for $i = 2, 3, \cdots, m_2$ by the units of the second layer of $n^{(1)}m_1^{(1)}m_2^{(1)}1'^{(1)}$ as the case of $l_{21}'$. Analogous to equation 8.5, the $i$th hyperplane $l_{2i}'$ of $\boldsymbol{X}'_n$ corresponds to the $i$th unit $u_{2i}$ of the second layer by
\begin{equation}
{\boldsymbol{w}'}_{2i}^T\boldsymbol{x}' + b_{2i}' = (\boldsymbol{w}_{i_n} + \boldsymbol{W}_c^T\boldsymbol{w}_{i_c})^T\boldsymbol{x}' + \boldsymbol{w}_{i_c}^T\boldsymbol{b}_c + b_{2i} = 0,
\end{equation}
where $[\boldsymbol{w}_{i_n}^T, \boldsymbol{w}_{i_c}^T]^T = \boldsymbol{w}_{2i}$ and $b_{2i}$ are the weight vector and bias of unit $u_{2i}$, respectively. Subvector $\boldsymbol{w}_{i_c}$ of $\boldsymbol{w}_{2i}$ is a fixed constant-entry vector as mentioned in equation 8.2. Since $\boldsymbol{w}_{2i}'$ and $b_{2i}'$ are known, the solution for $\boldsymbol{w}_{i_n}$ and $b_{2i}$ can be obtained by the method of equation 8.7. This completes the construction of the input parameters of $u_{2i}$'s for $i = 2, 3, \cdots, m_2$.

Finally, all of the outputs of $u_{2j}$'s for $j = 1, 2, \cdots m_2$ of the second layer are in terms of outputs of $l_{2j}'$'s of $n$-dimensional space $\boldsymbol{X}_n'$, satisfying the condition that $D'_n \subset \prod_{j=1}^{n}l_{2j}'^+$ and the corresponding linear-output matrix $\boldsymbol{\mathcal{W}}$ has rank $n+1$. By lemma 1 or 2, we can realize any linear function on $D'_n$ in the output layer of $n^{(1)}m_1^{(1)}m_2^{(1)}1'^{(1)}$. Consequently, a desired linear function on $D$ could be implemented due to the property of affine transforms (lemma 10 of \citet{Huang2020}).
\end{proof}

\begin{cl}
To network $n^{(1)}\prod_{i=1}^{d}m_i^{(1)}1'^{(1)}$ for $m_i > n$, suppose that data set $D$ of the input space simultaneously activates all the units of each hidden layer, and that in the $k = d-1$th layer, each element $\boldsymbol{x}^{(k)}$ of the mapped data set $D^{(k)}$ of $D$ could be represented in the form of equation 8.1 as $\boldsymbol{x}^{(k)} = {[{\boldsymbol{x}'_k}^T, {\boldsymbol{x}'}_{k_c}^T]}^T$, where $\boldsymbol{x}_k'$ is an affine transform of point $\boldsymbol{x}$ of $D$. Then any linear function on $D$ could be realized by this network.
\end{cl}
\begin{proof}
Consider the $d-1$th layer of $n^{(1)}\prod_{i=1}^{d}m_i^{(1)}1'^{(1)}$ of this corollary as the first layer of $n^{(1)}m_1^{(1)}m_2^{(1)}1'^{(1)}$ of theorem 9, and the conclusion follows.
\end{proof}

\begin{cl}
Any linear function on data set $D$ of the $n$-dimensional input space could be realized by network $n^{(1)}\prod_{i=1}^{d}m_i^{(1)}1'^{(1)}$ for $m_i > n$.
\end{cl}
\begin{proof}
If in the $j$th layer for $j = 1, 2, \dots, d-1$, any element $\boldsymbol{x}^{(j)}$ of the mapped data set $D^{(j)}$ of $D$ can be expressed in the form of equation 8.1 as
\begin{equation}
\boldsymbol{x}^{(j)} =
\begin{bmatrix}
{\boldsymbol{x}_j'}^T, {\boldsymbol{x}'}_{j_c}^T
\end{bmatrix}^T,
\end{equation}
then the network could transmit $D$ to the $d-1$th layer through subvectors $\boldsymbol{x}_j'$'s in the sense of affine transforms. Combined with corollary 8, this corollary would be proved.

To the first layer for $j = 1$, use the method of equation 8.1 to produce equation 8.13. To each succeeding layer for $j = 2, 3, \dots, d-1$, turn to proposition 6 to construct $n$ dimensions producing $\boldsymbol{x}_j'$ of equation 8.13.
\end{proof}

\begin{rmk}
The hidden-layer part of architecture $n^{(1)}\prod_{i=1}^{d}m_i^{(1)}1'^{(1)}$ could be regarded as a subnetwork that produces one linear component of a piecewise linear function of equation 4.6.
\end{rmk}

\subsection{Applications}

In theorem 5, we have demonstrated the capability of architecture $n^{(1)}\prod_{j=1}^{d}m_j^{(1)}1'^{(1)}$ for the production of piecewise linear functions. However, there's a constraint that $m_j \ge m_{j-1}$ for $j \ge 2$. In practice, only decoders have this typical feature. The following proposition will relax this condition, and therefore can explain more types of architectures.

Both the methods of linear-output generalizations of this section and affine-transform generalizations of section 7, as well as lemma 1 or 2 of section 4, will be used to achieve this goal. The parameter redundancy or overparameterization solution is also one of our concerns.
\begin{prp}
Suppose that network $n^{(1)}\prod_{j=1}^{d}m_j^{(1)}1'^{(1)}$ for $m_j \ge m_{j-1}$ when $j \ge 2$ has been constructed to produce a discrete piecewise linear function of equation 4.6 by theorem 5. Then it can be generalized to $n^{(1)}\prod_{j=1}^{d'}{M_j}^{(1)}1'^{(1)}$ with $M_j > m_j$ and $d' \ge d$, which can output the same discrete piecewise linear function.
\end{prp}
\begin{proof}
To the last hidden layer of $n^{(1)}\prod_{j=1}^{d}m_j^{(1)}1'^{(1)}$, use lemma 1 or 2 to add new units with respect to each subdomain, without influencing the production of the linear function on it.

For other hidden layers, we can add units in places where there are affine transforms mentioned. Theorem 8 and corollary 9 tell us that redundant units doesn't influence the transmission of data points via affine transforms, as well as the implementation of linear functions; and proposition 6 and theorem 9 provide the parameter-setting methods after new units having been added.

To the greater depth $d'$ of hidden layers, we first assume that any new added layer would not be the last hidden one. For an added layer, if it can transmit the data points of each subdomain via affine transforms or in the form of equation 8.13, by theorems 5, 8 and corollary 9, the same discrete piecewise linear function can still be realized. The case of adding more than one layers is similar.

For instance, in the last hidden layer of the network of Figure \ref{Fig.9}c, to subdomain $D_1$, besides $u_{3i}$ for $i = 1, 2, 3$, we can add any number of units only activated by $u_{21}$ and $u_{22}$; by lemma 1 or 2, the linear function on $D_1$ would not be influenced after updating relevant parameters. In other hidden layers such as the second one, for example, since $u_{21}$ and $u_{22}$ have the function of affine transforms, by theorem 8 and corollary 9, new units activated only by $u_{11}$ and $u_{12}$ could be added, without influencing the transmission of $D_1$ and the linear function on it. Thus, the original architecture $2^{(1)}4^{(1)}6^{(1)}9^{(1)}1'^{(1)}$ of Figure \ref{Fig.9}c can be generalized to any architecture $2^{(1)}\prod_{j=1}^{4}{M_j}^{(1)}1'^{(1)}$ for $M_1 > 4$, $M_2 > 6$ and $M_3 > 9$.

We can also add new layers into $2^{(1)}\prod_{j=1}^{4}{M_j}^{(1)}1'^{(1)}$. For instance, add a layer after the second one that can transmit the three subdomains via affine transforms or equation 8.13, which would not influence the final output after adjusting relevant parameters.
\end{proof}

\begin{rmk}
From another viewpoint, if a network has redundant units or layers due to theorem 8 or corollary 9, dropping some of them would not affect its interpolation capability. This is related to the topic of parameter redundancies or overparameterization solutions of a neural network.
\end{rmk}

\begin{cl}
Network architecture $n^{(1)}m^{(d)}1'^{(1)}$ for $m > n$ can realize any discrete piecewise linear function of equation 4.6, provided that $m$ is sufficiently large.
\end{cl}
\begin{proof}
After constructing a network $n^{(1)}\prod_{j=1}^{d}m_j^{(1)}1'^{(1)}$ to produce the desired piecewise linear function by theorem 5, use proposition 8 to generalize the architecture to $n^{(1)}m^{(d)}1'^{(1)}$ with $m \ge \max(m_1, m_2, \cdots, m_d)$.
\end{proof}

\section{Multi-Output Case \rom{2}}
In this section, we will investigate several multi-output network architectures applied in engineering, whose results are generalized from the single-output case.

\subsection{Main Results}
We generalize theorem 5 to lemma 11, and proposition 8 to theorem 10, from a single output to multi-outputs.
\begin{lem}
Network $n^{(1)}\prod_{j=1}^{d}m_j^{(1)}\mu'^{(1)}$ for $m_j \ge m_{j-1} > n$ when $j \ge 2$ derived from theorem 5 can realize arbitrary multi-dimensional discrete piecewise linear function of equation 5.1, if depth $d$ and widths $m_j$'s are large enough.
\end{lem}
\begin{proof}
This is an immediate consequence of theorem 5. If we can implement a discrete piecewise linear function on domain $D = \bigcup_{i=1}^{k}D_i \subset \mathbb{R}^n$ via network $n^{(1)}\prod_{j=1}^{d}m_j^{(1)}1'^{(1)}$ by theorem 5, it means that the hidden layers have successfully divided domain $D$ into subdomains. To each subdomain $D_i$, any number of linear functions could be defined on it, each corresponding to one linear unit of the output layer. Each linear unit of the output layer can produce arbitrary discrete piecewise linear function on $D$, independent of other ones by adjusting its own input weights as the three-layer case of theorem 3.

To construct a solution of this lemma, we first implement one dimension of the multi-dimensional discrete piecewise linear function by network $n^{(1)}\prod_{j=1}^{d}m_j^{(1)}1'^{(1)}$. Then add $\mu-1$ units in the output layer for other dimensions by the principle above.
\end{proof}

The next theorem is one of the main results of this paper, which is the generalization of proposition 8. It is the multi-output case with relaxed constraints on network architectures.
\begin{thm}
Any multi-dimensional discrete piecewise linear function of equation 5.1 could be implemented by network $n^{(1)}\prod_{j=1}^{d}m_j^{(1)}\mu^{(1)}$ (such as \citet*{LeCun2015} and \citet*{Deng2013}) for $m_j > n$, provided that depth $d$ and widths $m_j$'s are sufficiently large.
\end{thm}
\begin{proof}
Proposition 8 and lemma 11 imply the conclusion.
\end{proof}

Due to the convenience of the parameter setting of the number of units in each hidden layer, the architecture $n^{(1)}m^{(d)}\mu^{(1)}$ is popular in engineering (such as \citet*{Roberts2021}, \citet*{Lye2020}, and \citet*{lee2018}). Through trivial operations on the network of Figure 7c of \citet{Huang2020}, the solution can reach $n^{(1)}m^{(d)}\mu^{(1)}$. Expand the T-biases by ReLU networks and construct the corresponding feedforward network by lemma 2 of \citet{Huang2020}. And let the unconnected units between independent subnetwork modules be linked by zero-weight connections, then it would become $n^{(1)}m^{(d)}1^{(1)}$, after which the multi-output case $n^{(1)}m^{(d)}\mu^{(1)}$ follows.

However, the above solution for $n^{(1)}m^{(d)}\mu^{(1)}$ is too specially designed and not easily encountered in practice. The following corollary is to find more general solutions that the training process may reach in much more easier ways.
\begin{cl}
Network $n^{(1)}m^{(d)}\mu^{(1)}$ for $m > n$ can realize any multi-dimensional discrete piecewise linear function of equation 5.1, with $m$ large enough.
\end{cl}
\begin{proof}
The corollary is a special case of theorem 10 when $m_j$'s are all equal to $m$.
\end{proof}

\begin{rmk}
The universality of the solution depends on whether its underlying mechanism is based on fundamental properties of neural networks whose associated phenomena widely exist. The condition of this conclusion associated with theorems 5 and 10 is succinct and easier to fulfil, since it is mainly related to the number of activated units and the rank of weight matrices.
\end{rmk}

\subsection{Interpretation of Autoencoders}
The typical network architecture of autoencoders was illustrated in Figure 1 of \citet{Hinton2006}. The general form of autoencoders can be expressed as
\begin{equation}
\mathcal{A} := n^{(1)}\prod_{i=1}^{d_1}m_{i}^{(1)}n_e^{(1)}\prod_{j=1}^{d_2}M_{j}^{(1)}n^{(1)},
\end{equation}
where $m_{1} < n$ and $n_e < M_1$, in which the encoder is
\begin{equation}
\mathcal{E} := n^{(1)}\prod_{i=1}^{d_1}m_{i}^{(1)}n_e^{(1)},
\end{equation}
where $m_{i+1} < m_{i}$ for $i = 1, 2, \cdots, d_1-1$, and the decoder is
\begin{equation}
\mathcal{D} := n_e^{(1)}\prod_{j=1}^{d_2}M_{j}^{(1)}n^{(1)},
\end{equation}
where $M_{j+1} > M_{j}$ for $j = 1, 2, \cdots, d_2-1$.

Note that the architecture of a decoder of equation 9.3 is a type of network architecture of theorem 5 and lemma 6, for which we can interpret autoencoders from a new perspective.

\begin{thm}
Suppose that an encoder $\mathcal{E}$ of equation 9.2 maps a single point $\boldsymbol{x}$ of $n$-dimensional space to a single point $\boldsymbol{x}_e$ of $n_e$-dimensional space with $n_e < n$. Then a decoder $\mathcal{D}$ can map $\boldsymbol{x}_e$ back to $\boldsymbol{x}$, whose solution can be obtained by lemma 11.
\end{thm}
\begin{proof}
The map $f^{-1}(\boldsymbol{x}_e) = \boldsymbol{x}$ through the decoder can be realized in terms of multi-dimensional discrete piecewise linear functions of equation 5.1. To the $i$th dimension of the $n$-dimensional output of the decoder, for $i = 1, 2, \cdots, n$, or to the $i$th output-layer unit, construct a linear function passing through point $(\boldsymbol{x}_e, x_i)$ by theorem 5, where $x_i$ is the $i$th entry of $\boldsymbol{x}$ to be decoded. Then the map $f^{-1}(\boldsymbol{x}_e) = \boldsymbol{x}$ is realized.
\end{proof}

\noindent \textbf{Example}. As an example analogous to \citet{Hinton2006}, if the input is a $M \times N$ image $I$, it can be regarded as a point $\boldsymbol{x}$ of $n = M \times N$-dimensional space, which is obtained by the zigzag order of gray values of the image. The encoder compresses $\boldsymbol{x}$ into $\boldsymbol{x}_e$ with dimension $n_e < n$, which is also a point of $n_e$-dimensional space. The decoder maps $\boldsymbol{x}_e$ back to $\boldsymbol{x}$ by theorem 11. And through the reverse of zigzag order of $\boldsymbol{x}$, we can get the original image $I$.
\\
\\
\indent A single autoencoder can also process more than one signal elements:
\begin{cl}
Let $f:D \to D_e$ be a map through the encoder of equation 9.2, where $D$ is a data set of the $n$-dimensional input space, and $D_e$ is the mapped data set of $D$ in the $n_e$-dimensional space of the output layer of the encoder. If the map $f$ is bijective, we could reconstruct $D$ by a decoder of equation 9.3.
\end{cl}
\begin{proof}
Let $\boldsymbol{x}_i$ and $\boldsymbol{x}_j$ be arbitrary two data points of $D$ for $i, j = 1, 2, \cdots, k$ where $k = |D|$. The mapped data points of $\boldsymbol{x}_i$ and $\boldsymbol{x}_j$ by the encoder are denoted by $\boldsymbol{x}'_i$ and $\boldsymbol{x}'_j$, respectively; that is, $f(\boldsymbol{x}_1) = \boldsymbol{x}'_1$ and $f(\boldsymbol{x}_2) = \boldsymbol{x}'_2$. By the assumption, if $\boldsymbol{x}_i \ne \boldsymbol{x}_j$, then $\boldsymbol{x}_i' \ne \boldsymbol{x}_j'$.

Now we want the decoder to realize the function $f^{-1}: D_e \to D$ with $f^{-1}(\boldsymbol{x}'_i) = \boldsymbol{x}_i$. Write $D_e = \bigcup_{i=1}^{k}D_i$, where $D_i = \{\boldsymbol{x}'_i\}$; then the function $f^{-1}$ is a kind of multi-dimensional discrete piecewise linear function of equation 5.1. By lemma 11, theorems 5 and 11, a network architecture such as the decoder of equation 9.3 can implement $f^{-1}$.
\end{proof}

\begin{rmk}
We have interpreted a solution of decoders in lemma 6 and theorem 5, with single output however. The multi-output case is similar, since the difference is only in the output layer.
\end{rmk}

\section{Discussion}
The solution proposed by this paper may not be exactly the one used in practice, due to some possible constraints difficult to be fulfilled by the automatic training process. However, we expect that the underlying principles are general enough to explain the solution of engineering to some extent.

Among the results, the fundamental ones that we consider include: the affine-geometry background throughout this paper, the interference-avoiding principle (theorem 4), the affine-transform generalization of section 7, the parameter-sharing mechanism for multi-outputs (theorem 3 and lemma 11), the overparameterization-solution explanation (proposition 8), the concepts of distinguishable data sets (definition 11) and interference among hyperplanes (definition 7), the mechanism of the output layer (lemmas 1 and 2), and the probabilistic model of the appendix. They comprise the main underlying principles as mentioned above.

The concrete solutions of this paper may only serve as an evidence or application of those principles. We look forward to the ultimate goal that the solution of engineering is interpreted under our theoretical framework. We will further develop the theory on the basis of this paper in our future work.

We also want to formally introduce the methodology of theoretical physics to the study of neural networks, as discussed in section 1.1. The remaining part of this series of researches will all be based on this deductive way, and this paper is an initial step.

\section*{Appendix}

In the realm of random matrix theory \citep*{Tao2012}, the probability of the singularity of a square matrix or the rank of a matrix is one of the research interests, whose studies are characterized as follows. First, the entries of a matrix are constrained to be of special types, such as $(0,1)$ random matrix {\citep*{Komlos1967}} and $\pm1$ Bernoulli random matrix \citep*{Tao2007}. Second, the probability varies as a function of the matrix size {\citep*{Campos2021,Coja-Oghlan2019,Bourgain2010, Komlos1967,Komlos1968,Tao2007}}. Third, certain types of matrices are of interest, such as symmetric ones \citep*{Campos2021}.

Those researches were motivated and developed mainly for pure-mathematics reasons. The model below exclusively aims at the purpose of this paper, whose probability space is different from theirs. We want to measure the possibility of the rank or singularity property of arbitrary matrices, which is irrelevant to the matrix size. The consequences of the model are part of out theories and related to the training of neural networks.

Write a $m \times n$ weight matrix for $m \ge n$ of some layer of a neural network as
\begin{equation}
\boldsymbol{W} = \begin{bmatrix}
\boldsymbol{w}_1^T, \boldsymbol{w}_2^T, \cdots, \boldsymbol{w}_m^T
\end{bmatrix}^T,
\end{equation}
where $\boldsymbol{w}_i$ is the weight vector of the $i$th unit with size $1 \times n$ for $i = 1, 2 , \cdots, m$. In geometric language, each $\boldsymbol{w_i}$ of $\boldsymbol{W}$ is a normal vector of a hyperplane corresponding to the $i$th unit. When $m = n$, we use a special notation $\boldsymbol{W}_n$ to represent $\boldsymbol{W}$.

\begin{dfn}
The probability space
\begin{equation}
(\Omega, \mathcal{F}, P)
\end{equation}
is defined as: Let $\Omega = \{\boldsymbol{w} \  | \ \boldsymbol{w} \in \mathbb{R}^n \ and \ \|\boldsymbol{w}\| = 1 \}$ for $n \ge 2$, which can be regarded as an $n$-sphere centered at the origin point with radius $1$; $\mathcal{F} = 2^{\Omega}$ is the power set of $\Omega$; the random vectors in $\Omega$ are uniformly distributed and the measure $P(A) =\frac{1}{S_{\Omega}}\int dA$ with $A \in \mathcal{F}$, where $S_{\Omega}$ and $\int dA$ represent the surface areas of $\Omega$ and $A$, respectively.
\end{dfn}

\begin{lem}
Under the measurement of the probability space of equation 10.2, the weight matrix $\boldsymbol{W}_n$ of size $n \times n$ is nonsingular with probability $1$.
\end{lem}
\begin{proof}
If the length of each row vector $\boldsymbol{w}_i$ for $i = 1, 2, \cdots, n$ of matrix $\boldsymbol{W}_n$ is normalized to 1, the property of whether $\boldsymbol{W}_n$ is singular does not change. So throughout the proof, we always assume that $\boldsymbol{w}_i$ of $\boldsymbol{W}_n$ is a normalized vector whose length is 1.

We consider each row $\boldsymbol{w}_i$ of $\boldsymbol{W}_n$ as a point of the $n$-sphere $\Omega$ of equation 10.2. When $n = 2$, $\Omega$ is a circle with radius 1, and the probability of an event $\mathcal{F}_2$ is proportional to the length of the corresponding arc of $\mathcal{F}_2$ on the circle. $\boldsymbol{W}_2$ is singular if and only if its two row vectors $\boldsymbol{w_1}$ and $\boldsymbol{w_2}$ are collinear, that is, they coincide into a single vector or they have opposite directions. Suppose that we have randomly selected $\boldsymbol{w_1}$. The probability of $\boldsymbol{w_2}$ collinear with $\boldsymbol{w_1}$ is proportional to the length of the arc formed by the distribution of $\boldsymbol{w_2}$ relative to $\boldsymbol{w_1}$. If $\boldsymbol{W}_2$ is singular, the distribution of $\boldsymbol{w_2}$ can only be two points of the circle, whose length is zero; so the probability of singular $\boldsymbol{W}_2$ is 0. That is, $\boldsymbol{W}_2$ is nonsingular with probability 1.

When $n = 3$, the three row vectors $\boldsymbol{w}_1$, $\boldsymbol{w}_2$ and $\boldsymbol{w}_3$ of $\boldsymbol{W}_3$ are on sphere $\Omega$ and are chosen one by one. First $\boldsymbol{w}_1$ is selected, and the probability of $\boldsymbol{w}_2$ collinear with $\boldsymbol{w}_1$ is zero as discussed above. If $\boldsymbol{w}_1$, $\boldsymbol{w}_2$ and $\boldsymbol{w}_3$ are on a plane (denoted by event $\mathcal{F}_3$), $\boldsymbol{W}_3$ is singular. The distribution of $\boldsymbol{w}_3$ spanned by $\boldsymbol{w}_1$ and $\boldsymbol{w}_2$ forms a circle of sphere $\Omega$. To the circumference of the circle, its area is zero, so we have $P(\mathcal{F}_3) = 0$.

By induction, in $n$-dimensional space for $n > 3$, assume that the probability of $\boldsymbol{w}_i$'s for $i = 1, 2, \cdots, n - 1$ being on an $n - 2$-dimensional subspace is zero, and then we check the case of $\boldsymbol{w}_n$. If $\boldsymbol{W}_n$ is singular, $\boldsymbol{w}_n$ should be the linear combination of other $\boldsymbol{w}_i$'s for $i \ne n$; that is, they are on an $n - 1$-dimensional hyperplane. The distribution of $\boldsymbol{w}_n$ is the intersection of an $n - 1$-dimensional hyperplane and the $n$-sphere $\Omega$, whose surface area is zero since it lacks one dimension. Therefore, the probability of $\boldsymbol{W}_n$ being singular is 0 and the conclusion follows.
\end{proof}

\begin{thm}
The probability of $\rm{rank}(\boldsymbol{W}) = n$ of equation 10.1 is $1$ under the probabilistic model of equation 10.2.
\end{thm}
\begin{proof}
We construct $\boldsymbol{W}$ row by row. By lemma 12, a submatrix $\boldsymbol{W}_n$ of $\boldsymbol{W}$ is nonsingular with probability 1, so that the probability of $\text{rank}(\boldsymbol{W}) = n$ is also 1.
\end{proof}

\begin{rmk}
This conclusion is related to the training of neural networks, as well as the universality of the constructed solutions of this paper. If $\rm{rank}(\boldsymbol{W}) = n$ holds almost everywhere, it will be easier for the training process to reach the solution relevant to our results.
\end{rmk}

\bibliographystyle{APA}

\end{document}